%% file: main.tex
\documentclass[letterpaper]{article} 
\usepackage{aaai23}  
\usepackage{times}  
\usepackage{helvet}  
\usepackage{courier}  
\usepackage[hyphens]{url}  
\usepackage{graphicx} 
\usepackage{natbib}  
\usepackage{caption} 
\usepackage{booktabs}       
\usepackage{amsfonts}       
\usepackage{nicefrac}       
\usepackage{microtype}      
\usepackage{xcolor}         
\usepackage{amsmath}
\usepackage{subcaption}
\usepackage{bbold}
\usepackage{mathtools}
\usepackage{enumitem}
\usepackage{amsthm}
\usepackage{amssymb}
\usepackage{algorithm}
\usepackage{algpseudocode}

\usepackage{bm}
\usepackage{tabularx,adjustbox}
\usepackage{multirow}

\usepackage[section]{placeins}
\usepackage{cancel}

\usepackage[capitalize]{cleveref}
\Crefname{section}{Sec.~}{Sec.~}
\Crefname{table}{Tab.~}{Tabs.~}

\theoremstyle{definition}
\newtheorem{definition}{Definition}
\newtheorem{proposition}{Proposition}
\newtheorem*{proposition*}{Proposition} 

\newtheorem{lemma}{Lemma}

\newtheorem{corollary}{Corollary}

\DeclareMathOperator*{\argmax}{argmax}
\DeclareMathOperator*{\argmin}{argmin}
\newcommand{\squishlisttwo}{
 \begin{list}{$\bullet$}
  { \setlength{\itemsep}{1pt}
     \setlength{\parsep}{0pt}
    \setlength{\topsep}{0pt}
    \setlength{\partopsep}{0pt}
    \setlength{\leftmargin}{1em}
    \setlength{\labelwidth}{1.5em}
    \setlength{\labelsep}{0.5em} } 
}
\newcommand{\squishend}{
  \end{list}  }

\usepackage{newfloat}
\usepackage{listings}
\DeclareCaptionStyle{ruled}{labelfont=normalfont,labelsep=colon,strut=off} 
\title{Probably Approximate Shapley Fairness with Applications in Machine Learning}
\author{
    Zijian Zhou\textsuperscript{\rm 1}\equalcontrib
    ~Xinyi Xu\textsuperscript{\rm 12}\equalcontrib
    ~Rachael Hwee Ling Sim\textsuperscript{\rm 1}
    ~Chuan Sheng Foo\textsuperscript{\rm 2}
    ~Kian Hsiang Low\textsuperscript{\rm 1}
}
\affiliations{
    \textsuperscript{\rm 1}Department of Computer Science, National University of Singapore, Singapore \\
    \textsuperscript{\rm 2}Institute for Infocomm Research, A*STAR, Singapore\\
    \{zhou\_zijian, xinyi.xu, rachael.sim\}@u.nus.edu, foo\_chuan\_sheng@i2r.a-star.edu.sg, lowkh@comp.nus.edu.sg
}

\usepackage{bibentry}

\begin{document}

\maketitle

\input{sections/abstract}

\input{sections/introduction}

\input{sections/preliminaries}
\input{sections/fairness}
\input{sections/estimation}
\input{sections/experiments}

\input{sections/conclusion.tex}

\newpage
\subsubsection{Acknowledgements.} This research/project is supported by the National Research Foundation Singapore and DSO National
Laboratories under the AI Singapore Programme (AISG Award No: AISG$2$-RP-$2020$-$018$). Xinyi Xu is 
also supported by the Institute for Infocomm Research of Agency for Science, Technology and Research (A*STAR)

\bibliography{references}

\newpage
\onecolumn
\input{appendix}

\end{document}

%% file: sections/abstract.tex
\begin{abstract}
The \emph{Shapley value} (SV) is adopted in various scenarios in machine learning (ML), including data valuation, agent valuation, and feature attribution, as it satisfies their fairness requirements. 
However, as exact SVs are infeasible to compute in practice, SV estimates are approximated instead. This approximation step raises an important question: \emph{do the SV estimates preserve the fairness guarantees of exact SVs?}
We observe that the fairness guarantees of exact SVs are too restrictive for SV estimates. Thus, we generalise Shapley fairness to \emph{probably approximate Shapley fairness} and propose fidelity score, a metric to measure the variation of SV estimates, that determines how probable the fairness guarantees hold.
Our last theoretical contribution is a novel \emph{greedy active estimation} (GAE) algorithm that will maximise the lowest fidelity score and achieve a better fairness guarantee than the \emph{de facto} Monte-Carlo estimation. We empirically verify GAE outperforms several existing methods in guaranteeing fairness while remaining competitive in estimation accuracy in various ML scenarios using real-world datasets.
\end{abstract}

%% file: sections/introduction.tex
\section{Introduction} \label{sec:intro}
The \emph{Shapley value} (SV) is widely used in machine learning (ML), to value and price data \citep{Agarwal2019_datamaketplace,ohrimenko2019collaborative,pmlr-v97-ghorbani19c,Ghorbani2020ADF,xu2021vol,Kwon2022,sim2022-survey,zhao2022}, value data contributors in collaborative machine learning (CML) \citep{Hwee2020,Tay_Xu_Foo_Low_2022,Lucas2022,phong2022} and federated learning (FL) \citep{Song2019-profitfl,Wang2020,xu2021gradient} to decide fair rewards, and value features' effects on model predictions for interpretability \citep{Covert2021_improvingkernelSHAP,lundberg2017_kernelSHAP}.
We consider data valuation as our main example. Given a set $N$ of $n$ training examples and a utility function $v$ that maps a set $P$ of training examples to a real-valued utility (e.g., test accuracy of an ML model trained on $P$), the SV $\phi_i$ of the $i$-th example is
\begin{equation}
    \label{eq:shapley}
    \begin{aligned}
       \phi_i = \phi_i(N,v) &\coloneqq 1/(n!) \textstyle \sum_{\pi \in \Pi} \sigma_i(\pi) \\
       \sigma_i(\pi) &\coloneqq v(P_i^{\pi} \cup \{i\}) - v(P_i^{\pi}) 
    \end{aligned}
\end{equation}
where $\pi$ is a permutation of the $N$ training examples and $\Pi$ denotes the set of all possible permutations. The SV for the $i$-th example is its average marginal contribution, $\sigma_i(\pi)$, across all permutations. The marginal contribution $\sigma_i(\pi)$ measures the improvement in utility (e.g., test accuracy) when the $i$-th example is added to the predecessor set $P_i^{\pi}$ containing all training examples $j$ preceding $i$ in $\pi$.

The wide adoption of SV is often justified through its fairness axiomatic properties (recalled in Sec.~\ref{sec:preliminaries}). 
In data valuation, SV is desirable as it ensures that any two training example that improves the ML model performance equally when added to any data subset (i.e., all marginal contributions are equal) are assigned the same value (known as \emph{symmetry}).
However, a key downside of SV is that the exact calculation of $\phi_i$ in Equ.~\eqref{eq:shapley} has exponential time complexity and is intractable when valuing more than hundreds of training examples \citep{pmlr-v97-ghorbani19c,Jia2019nn}.
Existing works address this downside by viewing the SV definition in Equ.~\eqref{eq:shapley} as an expectation over the uniform distribution $U$ over $\Pi$, $\phi_i = \mathbb{E}_{\pi \sim U} [\sigma_i(\pi)]$, and applying Monte Carlo (MC) approximation \citep{Castro2009} with $m_i$ randomly sampled permutations, thus $\phi_i \approx \varphi_i \coloneqq 1/m_i \sum_{t = 1}^{m_i} \sigma_i(\pi^t), \pi^t \sim U$ \citep{pmlr-v97-ghorbani19c,jia2019towards,Song2019-profitfl}.

However, this approximation creates an important issue --- \emph{(i) do the fairness axioms that justify the use of Shaplay value still hold after approximation \cite{Sundararajan2020_axioms_do_not_hold,Rozemberczki2022}?} The answer is unfortunately no as we empirically demonstrate that the symmetry axiom does not hold after approximations in \cref{fig:k_value_dataset}. For two \emph{identical} training examples, their (approximated) SVs used in data pricing are not guaranteed to be equal.
We address this unresolved important issue by proposing the notion of \emph{probably approximate Shapley fairness} for SV estimates $\varphi_i$, for every $i \in N$. As the original fairness axioms are too restrictive, in Sec.~\ref{sec:fairness}, we relax the fairness axioms to \emph{approximate} fairness and consider how they can be satisfied with high \emph{probability}. We introduce a \emph{fidelity score} (FS) to measure the approximation quality of $\varphi_i$ w.r.t.~$\phi_i$ for each example $i$ and provide a fairness guarantee dependent on the worst/lowest fidelity score across all training examples.

In data valuation, computing the marginal contribution of an $i \in N$ in any sampled permutation is expensive as it involves training model(s).
\emph{(ii) How do we achieve probably approximate Shapley fairness with the lowest budget (number of samples) of marginal contribution evaluations?}  While it is difficult to achieve the highest approximate fairness possible, we show that we can instead achieve a high fairness \emph{guarantee} (i.e., a lower bound to probably approximate Shapley fairness) via the insight that the budget need not be equally spent on all training examples. For example, if the marginal contribution of example $i$, $\sigma_i(\pi^t)$ in many sampled permutations are constant, we should instead evaluate that of example $j$ with widely varying $\sigma_j(\pi^t)$ sampled so far.
Our method may use a different number of marginal contribution samples, $m_i$, for each example $i$ and greedily improve the current worst fidelity score across all training examples.

Lastly, to improve the fidelity score, we novelly \emph{use previous samples, i.e., evaluated marginal contribution results, to influence and guide the current sampling of permutations}.
In existing MC methods \citep{Owen1972,Okhrati2020,Mitchell2021} the sampling distribution that generates $\pi$ is pre-determined and fixed across iterations.
In our work, we use importance sampling to generate $\pi$ for $\varphi_i$ as it supports using an alternative proposal sampling distribution, $q_i(\pi)$. For any example $i$, we constrain permutations $\pi$ with predecessor set of equal size to have the same probability $q_i(\pi)$. The parameters of the sampling distribution $q_i$ are actively updated across iterations and learnt from past results (i.e., tuples of predecessor set size and marginal contribution) via maximum likelihood estimation or a Bayesian approach. By doing so, we reduce the variance of the estimator $\varphi_i$ as compared to standard MC sampling, thus improving the fidelity scores efficiently and the overall fairness (guarantee).

Our specific contributions are summarized as follows:
\squishlisttwo
    \item We propose a \emph{probably approximate Shapley fairness} for SV estimates and exploit an error-aware \emph{fidelity score} to provide a fairness guarantee via a polynomial budget complexity.
    \item We design greedy selection, which by iteratively prioritising $\varphi_i$ with lowest FS, can obtain the optimal minimum FS given a fixed total budget $m$ and improve the fairness guarantee (Proposition~\ref{proposition:chebyshev_fidelity}).
    \item We derive the optimal categorical distribution (intractable) for selecting permutations, and obtain an approximation for active permutation selection. 
    We integrate both greedy and active selection into a novel \emph{greedy active estimation} (GAE) with provably better fairness than MC.
    \item We empirically verify that GAE outperforms existing methods in guaranteeing fairness while remaining competitive in estimation accuracy in training example and dataset valuations, agent valuation (in CML/FL) and feature attribution with real-world datasets.
\squishend

%% file: sections/preliminaries.tex
\section{Preliminaries}
\label{sec:preliminaries}

\paragraph{Fairness of SV.} 
SV (Equ.~\eqref{eq:shapley}) is often adopted \citep{Agarwal2019_datamaketplace,Hwee2020,Song2019-profitfl,xu2021gradient,sim2022-survey} for guaranteeing fairness by satisfying several axioms \cite{cgt2011}:

\begin{enumerate}[label*=F\arabic*.,ref=F\arabic*,leftmargin=0.75cm, topsep=1pt]
    \item \label{original_axioms:nullity} Nullity: $(\forall \pi \in \Pi, \sigma_i(\pi) = 0) \implies \phi_i = 0$. 
    \item \label{original_axioms:symmetry} Symmetry: $(\forall C \subseteq N \setminus \{i,j\}, v(C \cup \{i\}) = v(C \cup \{j\})) \implies \phi_i = \phi_j$. 
    \item \label{original_axioms:desirability} Strict desirability~\cite{Bahir1966}: $\forall i\neq j \in N, (\exists B \subseteq N \setminus \{i,j\}, v(B \cup \{i\}) > v(B \cup \{j\})) \wedge (\forall C \subseteq N \setminus \{i, j\}, v(C \cup \{i\}) \geq v(C \cup \{j\})) \implies \phi_i > \phi_j$.
\end{enumerate}
Nullity means if a training example does not result in any performance improvement (marginal contribution is $0$ to any permutation), then its value is $0$. It ensures offering useless data does not give any reward \citep{Hwee2020}).
Symmetry ensures identical values for identical training examples.
Strict desirability implies if $i$ gives a larger performance improvement than $j$ in all possible permutations, then $i$ is strictly more valuable than $j$.
We exclude the efficiency axiom as it does not suit ML use-cases \citep{pmlr-v97-ghorbani19c,bian2021energy,Kwon2022},\footnote{Efficiency requires $\sum_i \phi_i = v(N)$ which is difficult to verify in practice for ML \citep{Kwon2022}; it does not make sense to ``distribute'' the voting power in feature attribution (interpretation of efficiency) to each $i$ \citep{bian2021energy}.} and exclude the linearity \citep{Jia2019nn} and monotonicity \citep{Hwee2020} axioms for fairness analysis as we restrict our consideration to one utility function $v$ \citep{bian2021energy}.
Note that the fairness of exact SV is \emph{binary}: either satisfying all these axioms or not.

\paragraph{Sampling-based estimations.}
These methods typically extend the MC formulation of $\phi_i \approx \varphi_i \coloneqq  \mathbb{E}_{\pi_t \sim U} \left[\sigma_i(\pi_t) \right]$ by changing the sampling distribution of $\pi_t$. 
Importantly, in all these methods, for each sampled permutation $\pi_t$, a \emph{single} marginal contribution $\sigma_i(\pi_t)$ is evaluated and used in $\varphi_i$. Thereafter, we consistently refer to this single evaluation as expending one budget (i.e., one permutation) and the corresponding marginal contribution $\sigma_i(\pi_t)$ as one sample. 

Formally, a sampling-based estimation method estimates $\phi_i$ via the expectation of the random variable $\sigma_i(\pi)$ (which depends on the permutations randomly sampled according to some distribution $q$): $\phi_i \approx \mathbb{E}_q[\sigma_i(\pi)]$. Hence, such methods differ from each other in the distribution $q$ as well as the selection of entry $i \in N$ to evaluate in each iteration. The estimates $\varphi_i$'s can be independent of each other such as MC~\citep{Castro2009}, and stratified sampling~\citep{maleki2013,Castro2017}, or dependent such as antithetic Owen method~\citep{Owen1972,Okhrati2020} and Sobol method~\citep{Mitchell2021}. We provide theoretical results for both scenarios when the estimates $\varphi_i$'s are independent and dependent.

%% file: sections/fairness.tex
\section{Generalised Fairness for Shapley Value Estimates}
\label{sec:fairness}

\subsection{Fairness Axioms for Shapley Value Estimates}
We propose the following re-axiomatisation of fairness for SV estimates (based on axioms \ref{original_axioms:nullity}-\ref{original_axioms:desirability}) using conditional events, by analysing multiplicative and absolute errors.

\begin{enumerate}[label*=A\arabic*., ref=A\arabic*,leftmargin=0.75cm, topsep=1pt]
    \item \label{axiom:nullity} Nullity: 
    let $E_{A_1}$ be the (conditional) event that for any $i \in N$, conditioned on $\phi_i = 0$, $|\varphi_i| \leq \epsilon_2$. 
    \item \label{axiom:symmetry} Symmetry: let $E_{A_2}$ be the (conditional) event that for all $i \neq j \in N$, conditioned on $\phi_i = \phi_j$, then $|\varphi_i - \varphi_j| \leq (\epsilon_1|\phi_i| + \epsilon_2) + (\epsilon_1|\phi_j| + \epsilon_2)$.
    \item \label{axiom:desirability} Approximate desirability: let $E_{A_3}$ be the (conditional) event that for all $i \neq j \in N$, conditioned on
    $(\exists B \subseteq N \setminus \{i,j\}, v(B \cup \{i\}) > v(B \cup \{j\})) \wedge (\forall C \subseteq N \setminus \{i,j\}, v(C\cup\{i\}) >  v(C\cup\{j\}))$, then $\varphi_i - \varphi_j > -(\epsilon_1|\phi_i| + \epsilon_2) - (\epsilon_1|\phi_j| + \epsilon_2)$.
\end{enumerate}
Thereafter, satisfying \ref{axiom:nullity}-\ref{axiom:desirability} refers to the events $E_{A_1}$-$E_{A_3}$ occurring, respectively. 
To see \ref{axiom:nullity}-\ref{axiom:desirability} generalise the original axioms \ref{original_axioms:nullity}-\ref{original_axioms:desirability}:\footnote{An equivalent formulation for \ref{original_axioms:nullity}-\ref{original_axioms:desirability} using conditional events are in \cref{app:additional-analysis}.}
\ref{axiom:nullity} requires the SV estimate to be small for a true SV with value $0$ (as in \ref{original_axioms:nullity});
\ref{axiom:symmetry} requires the SV estimates for equal SVs be close (generalised from being equal in \ref{original_axioms:symmetry});
\ref{axiom:desirability} requires the ordering of $\varphi_i,\varphi_i$ for some $i,j$ from \ref{original_axioms:desirability} to be preserved up to some error, specifically a multiplicative error $\epsilon_1$ (to account for different $|\phi_i|$) and an absolute error $\epsilon_2$ (to avoid degeneracy from extremely small $|\phi_i|$) where \ref{original_axioms:desirability} has no such error term.
Intuitively, smaller errors $\epsilon_1,\epsilon_2$ mean $\boldsymbol{\varphi}$ are ``closer'' (in fairness) to $\boldsymbol{\phi}$.
In addition, the ratio $\xi \equiv \epsilon_2/\epsilon_1$ denotes the tolerance (to be set by user) of relative (multiplicative) vs.~absolute errors where a larger $\xi$ implies a higher tolerance for absolute error (i.e., favours \ref{axiom:nullity} over \ref{axiom:symmetry} \& \ref{axiom:desirability}) and \textit{vice versa}. 
In contrast to existing works that only consider the concentration results of $\varphi_i$ w.r.t.~$\phi_i$ for each $i$ \citep{Castro2009,Castro2017,maleki2013}, we additionally consider the interaction between $i,j$ to define the following:
\begin{definition}[\bf Probably Approximate Shapley Fairness] \label{definition:fairness-guarantee}
For fixed $\epsilon_1, \epsilon_2$, and some $\delta \in (0,1)$ s.t. $\boldsymbol{\varphi}$ satisfy \ref{axiom:nullity}-\ref{axiom:desirability} \emph{jointly} w.p.~$\geq 1-\delta$, then we call $\boldsymbol{\varphi}$ satisfy $(\epsilon_1,\epsilon_2,\delta)$-Shapley fairness.
\end{definition}

In \cref{definition:fairness-guarantee}, $\boldsymbol{\varphi}$ are probably (w.p.~$\geq 1-\delta$) approximately (w.r.t.~errors $\epsilon_1,\epsilon_2$) Shapley fair.
Hence, given the error requirements $\epsilon_1, \epsilon_2$, a smaller $\delta$ means a better fairness guarantee.
In particular, $\boldsymbol{\phi}$ satisfy the optimal $(0,0,0)$-Shapley fairness.
Despite the appeal, analysing existing estimators w.r.t.~\cref{definition:fairness-guarantee} is difficult because most do not come with a direct variance result (elaborated in \cref{app:additional-analysis}). The only expect is the MC method \citep{Castro2009,maleki2013} which we analyse in \cref{sec:active_valuation}.

\subsection{Fairness Guarantee via the Fidelity Score of Shapley Value Estimates}
Inspired by the concept of \emph{signal-to-noise ratio} (SNR) widely adopted in optics~\cite{snr_optic} and imaging~\cite{rose2012vision}, we design a metric for the variation of $\varphi_i$, named fidelity score, expressed in \cref{def:afs}.
\begin{definition}[\bf Fidelity Score] 
The fidelity score (FS) of an (unbiased) SV estimate $\varphi_i$ for $\phi_i$ is defined as
$ f_i \equiv \text{FS}(\varphi_i, \epsilon_1, \epsilon_2) \coloneqq (|\phi_i| + \epsilon_2/\epsilon_1 )^2 /\mathbb{V}[\varphi_i] $
where $\mathbb{V}[\varphi_i]$ is the variance of $\varphi_i$.\footnote{Our implementation estimates $|\phi_i|$ and $\mathbb{V}[\varphi_i]$ to obtain $f_i$.}
\label{def:afs}
\end{definition}
The FS exactly matches the SNR in $\varphi_i$ when $\epsilon_2=0$.\footnote{For a fixed FS, an example $i$ with a larger SV (signal) can contain more noise.} A higher $f_i$ implies a more accurate $\varphi_i$. $f_i$ is higher when the variance $\mathbb{V}[\varphi_i]$ is small. This occurs when the marginal contributions, $\sigma_i(\pi)$, are close for all permutations $\pi \in \Pi$ or when the number of samples $m_i$ used to compute $\varphi_i$ is large. As $m_i \to \infty$, $\mathbb{V}[\varphi_i] \to 0$ and $f_i\to \infty$. 
Additionally, we introduce an error-aware term $\xi \coloneqq \epsilon_2/\epsilon_1$ in the FS numerator to better capture estimation errors and allow examples with SV of $0$ to have different FSs.

Moreover, we empirically verify that $f_i$ is a good reflection of the approximation quality and analyze the impact of $\xi$.
For the former, we compared $f_i$ vs.~absolute percentage error (APE) of $\varphi_i \coloneqq |(\varphi_i - \phi_i) / \phi_i|$.\footnote{
We fit a learner on $50$ randomly selected training examples from breast cancer~\citep{breast_cancer} and MNIST~\citep{cnn_mnist} (diabetes~\citep{Efron_2004}) datasets and set test accuracy (negative mean squared error) as $v$ using data Shapley \citep{pmlr-v97-ghorbani19c} with $m_i =50$.
$f_i$ is approximated using sample evaluations.
Additional details and results are in \cref{app:experiments}. 
}
\cref{fig:ri_mape} shows the negative correlation between FS ($f_i$) and estimation error ($\text{APE}$) (note that we plot $\text{APE}^{-0.5}$).
This will justify our proposal to prioritize improving the estimate $\varphi_i$ with the lowest $f_i$ in Sec~\ref{sec:greedy_selection}.
For the latter,
we compare the correlation between FS, $f_i$ and the inverse estimation error, $\text{APE}^{-0.5}$, for different values of $\xi$ in Tab.~\ref{tab:xi_ri_ape} and find that $\xi=1\text{e-}3$ leads to a strong positive correlation and is a sweet spot (second best in both settings).
\begin{figure}[!ht]
    \centering
    \includegraphics[width=0.49\linewidth]{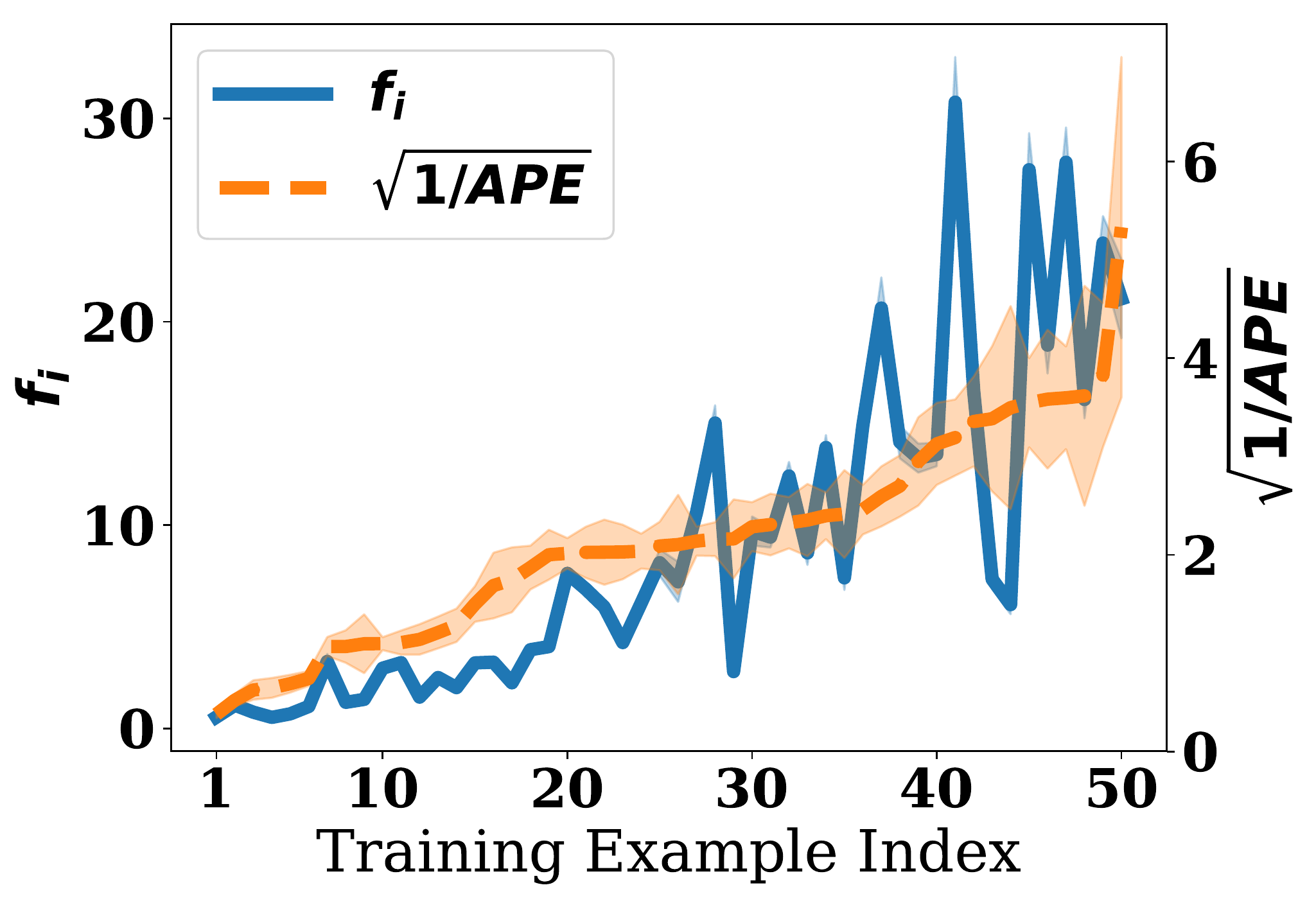}
    \includegraphics[width=0.49\linewidth]{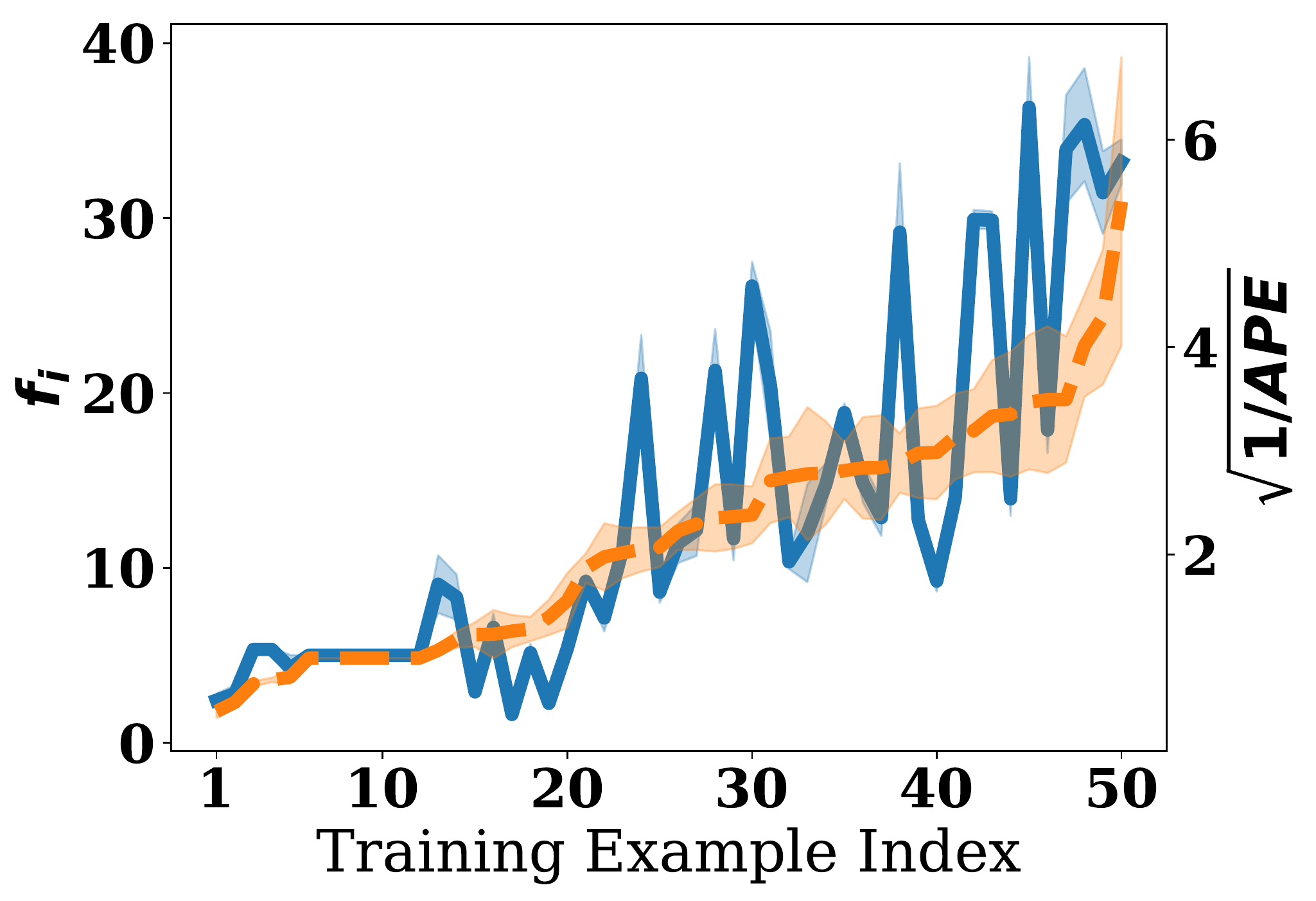}
    \caption{Average (standard error) of $f_i,\text{APE}^{-0.5}$ over $20$ trials (sorted in increasing order of $\text{APE}^{-0.5}$ of $50$ training examples) with $\xi=1\text{e-}3$.
    Left (right) is logistic regression ($k$-nearest neighbors) using breast cancer (MNIST) dataset.}
    \label{fig:ri_mape}
\end{figure}
\vspace{-0.5em}
\begin{table}[!ht]
	\centering
    \setlength{\tabcolsep}{2pt}
    \caption{ Spearman coefficient between $f_i$ and $\text{APE}^{-0.5}$. Average (standard error) over $20$ independent trials. Higher is better.
    }
    \resizebox{0.8\linewidth}{!}{
    \footnotesize
    \begin{tabular}{lllrrrr}
    \toprule
    $\xi$ & breast cancer (logistic) &  diabetes (ridge) \\
    \midrule
    $1$e-$5$ &  \textbf{6.89e-01} (1.41e-02) &  5.49e-01 (2.24e-02) \\
    $1$e-$4$ &  \textbf{6.89e-01} (1.41e-02) &  5.48e-01 (2.24e-02) \\
    $1$e-$3$ &  6.78e-01 (1.50e-02) & 5.53e-01 (2.20e-02) \\
    $1$e-$2$ &  -1.94e-01 (2.59e-02) & \textbf{5.67e-01} (2.32e-02)  \\
    $1$e-$1$ &  -1.47e-01 (2.60e-03) &   5.00e-01 (2.55e-02) \\
    \bottomrule
    \end{tabular}
    }
    \label{tab:xi_ri_ape}
\end{table}

\cref{def:afs} allows us to leverage the Chebyshev's inequality to derive a fairness guarantee, through the minimum FS $\underline{f} \coloneqq \min_{i\in N} f_i$.
\begin{proposition}
\label{proposition:chebyshev_fidelity}
$\boldsymbol{\varphi}$ satisfy $(\epsilon_1, \epsilon_1\xi, \delta)$-Shapley fairness where $\delta = 1- (1-1/(\epsilon_1^2\underline{f}))^n$ if all $\varphi_i$'s are independent; otherwise ~$\delta =  n/(\epsilon_1^2\underline{f})$.
\end{proposition}
\cref{proposition:chebyshev_fidelity} (its proof is in \cref{app:proofs}) formalises the effects of the variations in $\boldsymbol{\varphi}$ in satisfying probably approximate Shapley fairness where a larger minimum variation (i.e., $\underline{f}$) results in larger $\delta$, hence lower probability of $\boldsymbol{\varphi}$ satisfying the fairness axioms. To see whether $\varphi_i$'s are independent, it is equivalent to checking whether any permutation sampled is used for estimating multiple $\varphi_i$'s (proof in \cref{app:proofs}).

The fidelity score, $f_i$, is sensitive to the number of sampled permutations and marginal contributions, $m_i$, as the variance of the SV estimator uses $m_i$ independent samples to produce $\mathbb{V}[\varphi_i] =  \mathbb{V}[\sigma_i(\pi)] / m_i$.
Therefore, we define an insensitive quantity, the \emph{invariability} of $i$, $r_i$, as the FS with only \emph{one} sample of $\pi$. Hence, $r_i \propto 1/ \mathbb{V}[\sigma_i(\pi)]$ and here $\pi \sim U$. A higher $r_i$ implies that $i$-th marginal contribution is more invariable across different permutations. The fidelity score is product of the invariability and number of samples, i.e., $f_i = m_i r_i$, used in proving the following corollary.

\begin{corollary} \label{corollary:k_axioms}
Using the notations in \cref{proposition:chebyshev_fidelity}, the minimum total budget $m = \sum_{i \in N} m_i$ needed to satisfy $(\epsilon_1, \epsilon_1\xi,\delta)$-Shapley fairness is at most $\mathcal{O}(n \epsilon_1^{-2} (1-(1-\delta)^{1/n})^{-1})$ if $\varphi_i$'s are independent; and $\mathcal{O}(n^2 \epsilon_1^{-2}\delta^{-1})$ o/w.
\end{corollary}
The budget complexity is an upper bound in a best/ideal case in the sense that our derivation (in \cref{app:proofs}) uses $r_i$ which cannot be observed in practice. 
While it is not shown to be tight, 
our $\mathcal{O}(n)$ budget complexity for the independent scenario is linear in terms of the number of training examples, and the $\mathcal{O}(n^2)$ budget complexity for the dependent scenario seems necessary for the $\mathcal{O}(n^2)$ pairwise interactions between $\varphi_i$'s.
\cref{sec:active_valuation} describes an estimation method that runs in the budget complexity upper bound given in \cref{corollary:k_axioms}.

%% file: sections/estimation.tex
\section{Fairness via Greedy Active Estimation}
\label{sec:active_valuation}

To achieve probably approximate Shapley fairness with the lowest budget (number of samples) of marginal contribution evaluations, we propose a novel greedy active estimation (GAE) consisting of two core components - \emph{greedy selection} and \emph{importance sampling}. The first component efficiently split the training budget across examples while the second component influence and guide the sampling of permutations and reduce the variance $\mathbb{V}[\varphi_i]$ for each example $i$. In this section, we outline these components and show that they improve the minimum fidelity scores, $\underline{f}$.
The details and full pseudo-code of the algorithm is given in \cref{app:algorithm}. 

\subsection{Greedy Selection using Pigou-Dalton} \label{sec:greedy_selection}
From \cref{proposition:chebyshev_fidelity}, we can observe that a larger $\underline{f}$ decreases the probability of unfairness.
Hence, to efficiently achieve probably approximate fairness, we should maximize $\underline{f}$ by improving the FS of the training example with the current lowest $f_i$. Our greedy method is outlined in \cref{proposition:optimality_of_greedy} and formally proven in \cref{app:proofs}.

\begin{proposition}[\textit{Informal}]  \label{proposition:optimality_of_greedy}
Given the constraint of evaluating a total of $m$ marginal contributions for all $j \in N$ to its predecessor set when the permutations are sampled from a fixed distribution $q$ (e.g., the uniform distribution).
Then, the minimum FS $\underline{f}$ is maximised by (iteratively) greedily selecting and evaluating a marginal contribution of $i = \argmin_{j \in N} f_j$, until $m$ is exhausted.
\end{proposition}
One direct implication is that greedy selection outperforms equally allocating the budget $m$ among all $N$ (i.e., $m_i = m/n$), which is used by existing methods (\cref{sec:preliminaries}).
Greedy selection will use a lower budget $m_i$ on a training example $i$ with higher invariability $r_i$ (lower variance in marginal contribution across permutations) to meet the same threshold $\underline{f}$. The budget will be mainly devoted to training examples with lower invariability and higher variance instead.

Moreover, improving $\underline{f}$ across all $i \in N$ is in line with the Pigou-Dalton principle (PDP) \citep{dalton1920}:
Suppose we have two divisions of the budget that result in two sets of FSs denoted by $\boldsymbol{f},\boldsymbol{f}'\in \mathbb{R}^{n}$ respectively. 
PDP prefers $\boldsymbol{f}$ to $\boldsymbol{f}'$ if $\exists i,j\in N$ s.t., (a) $\forall k \in N \setminus \{i,j\}, f_k = f_k'$ and (b) $f_i+f_j=f_i'+f_j'$ and (c) $|f_i - f_j| < |f_i' - f_j'|$.
PDP favors a division of the budget that leads to more equitable distribution of FSs.
For SV estimation, it means we are approximately equally sure about all the estimates of the training examples, which can improve the effectiveness of identifying valuable training examples for active learning \citep{Ghorbani2021_batch_active} (\cref{fig:add-remove}) or the potentially noisy/adversarial ones \citep{pmlr-v97-ghorbani19c} (\cref{fig:f1_score} in \cref{app:experiments}).
Theoretically, an inequitable distribution of FSs with some training examples with significantly lower $f_i$ will have a worse fairness guarantee (\cref{proposition:chebyshev_fidelity}).
We show that greedy selection satisfies PDP (\cref{proposition:pe_pdp} in \cref{app:proofs}).

\emph{Remark.}
Although greedy selection maximizes the minimums FS, $\underline{f}$, it is not guaranteed to achieve approximate Shapley fairness with the highest probability as the probability bound in \cref{proposition:chebyshev_fidelity} is not tight (e.g., due to the application of union bound in derivation). Nevertheless, in \cref{sec:experiments}, we empirically demonstrate that greedy selection indeed outperforms other existing methods in achieving probably approximate fairness. It is an appealing future direction to further improve the analysis and provide a tight bound for \cref{proposition:chebyshev_fidelity}.

\subsection{Active Permutation Selection via Importance Sampling} \label{sec:importance_sampling}
To improve the fidelity scores in Definition~\ref{def:afs}, our GAE method uses importance sampling to reduce $\mathbb{V}[\varphi_i]$  for every training example $i$ by setting $\varphi_i \coloneqq 1/m_i \sum_{t = 1}^{m_i} \sigma_i(\pi^t)/(q_i(\pi^t)n!) , \pi^t \sim q_i$. Here, $q_i$ is our proposal distribution over set of all permutations $\Pi$ that assigns probability $q_i(\pi)$ to permutation $\pi$.
Following existing works \cite{Castro2017,Ghorbani2020ADF}, we encode the assumption that any permutation, $\pi$, with the same cardinality for the predecessor set of $i$, $P_i^{\pi}$, should be assigned equal sampling probability. Hence $q_i(\pi) \propto q_i'(|P_i^{\pi}|)$ where the function $q_i'$ maps the predecessor's cardinality to the sampling probability.
We derive the optimal (but intractable) distribution $q_i^*(\pi) \propto q_i'^*(|P_i^{\pi}|) \propto (\mathbb{E}_{\pi \sim U(|P_i^{\pi}|)}[\sigma(\pi)^2])^{0.5}$ (proof in \cref{app:proofs}) and approximate it with a ``learnable'' $q_i'$.
Specifically, we use a categorical distribution over the support $\{0,\ldots,n-1\}$ as $q_i'$ and learn its parameters through maximum likelihood estimation (MLE) on tuples of predecessor set size and marginal contribution, with bootstrapping (i.e., sampling a small amount of permutations using MC, detailed in \cref{app:algorithm}). 
This active selection leads to both theoretical (\cref{prop:active_greedy_mc}) and empirical improvements (see Sec.~\ref{sec:experiments}), whilst ensuring $\mathbb{E}[\varphi_i]=\phi_i$ (\cref{app:additional-analysis}).

\begin{proposition}\label{prop:active_greedy_mc} 
For a fixed budget $m$, denote the minimum FS achieved by greedy selection and active importance sampling, greedy selection only (with uniform sampling), and uniform MC sampling as $\underline{f}_{\text{GAE}},\underline{f}_{\text{greedy}}$ and $\underline{f}_{\text{MC}}$, respectively.
Then, GAE outperforms the other methods as
\[\underline{f}_{\text{GAE}} \geq \underline{f}_{\text{greedy}} \geq \underline{f}_{\text{MC}}\ .
\]
Furthermore, the minimum FSs $(\underline{f}_{\text{GAE}}, \underline{f}_{\text{greedy}})$ are equal iff (a) $\forall i \in N, \mathbb{V}_{\pi \sim q^*_i}[\sigma_i(\pi)/(q_i^*(\pi)n!)] = \mathbb{V}[\sigma_i(\pi)]$ and
the minimum FSs $( \underline{f}_{\text{greedy}}, \underline{f}_{\text{MC}})$ are equal iff (b) every $i \in N$ has the same invariability w.r.t. $q_i^*$, ${r_{i,q_i^*} \coloneqq (|\phi_i| + \xi)^2 / \mathbb{V}_{\pi \sim q^*_i}[\sigma_i(\pi)/(q_i^*(\pi)n!)]}$.\footnote{As before, the invariability is the FS when using only \emph{one} sample, but the definition is updated to match the redefined $\varphi_i$.}

\end{proposition}
The proof is in \cref{app:proofs}.
In practice, equality conditions (a) and (b) are unlikely to hold. (a) only holds when our cardinality assumption ($i$-th marginal contribution to predecessor set of the same cardinality are similar) is wrong and unhelpful.
A necessary but unrealistic condition for (b) is that for two data points $i,j$ with the same SV, their marginal contributions, i.e., the set $\{\sigma_i(\pi) | \pi \in \Pi\}$, must be equal.

\paragraph{Regularising importance sampling with uniform prior.}
\cref{prop:active_greedy_mc} requires the cardinality assumption so that the importance sampling approach (using $q_i'$ which corresponds to a $q_i$) is effective by ensuring the derived $q_i$ is close to the theoretical optimal $q_i^*$.
However, in practice, there are situations where using the uniform distribution $U$ performs better than $q_i$ obtained via learning $q_i'$ using the cardinality assumption.
First, if the marginal contributions on the \emph{same} cardinality vary significantly (i.e., the cardinality assumption does not hold), then the approximation using $q_i$ is inaccurate.
Second, if the marginal contributions over \emph{different} cardinalities vary minimally, then importance sampling has little benefit as the marginal contributions are approximately equal.
Interestingly, in both cases, using $U$ (treating all cardinalities uniformly) is the mitigation because it avoids using the incorrect inductive bias (i.e., the cardinality assumption).

Therefore, to incorporate $U$, we regularise the learning of $q_i'$ with a uniform Dirichlet prior.
Specifically, from the MLE parameters $\boldsymbol{w} \in \triangle(n)$ of $q_i'^*, i \in N$ obtained via bootstrapping,\footnote{$\triangle(n)$ denotes the probability simplex of dimension $n-1$ and $\boldsymbol{w}$ is derived in \cref{app:additional-analysis}.}
and a uniform/flat Dirichet prior parameterised by $(\alpha+1)\mathbf{1}_n, \alpha \geq 0$, we can obtain the maximum a posteriori (MAP) estimate as $n\boldsymbol{w} + (\alpha+1)\mathbf{1}_n$ (more details and derivations in \cref{app:additional-analysis}).
With this, we unify the frequentist approach of learning $q_i'^*$'s parameters via purely MLE and the Bayesian approach of incorporating both likelihood and prior belief with $\alpha$ controlling how strong our prior belief is. Specifically, when $\alpha=0$, $q_i'$ reduces to MLE and as $\alpha \rightarrow \infty$, $q_i'$ tends to the uniform distribution over cardinality (due to the uniform Dirichlet prior) and thus the corresponding $q_i$ tends to $U$. In other words, if the cardinality assumption is satisfied (not satisfied), we expect the learned $q_i'$ with a small (large) $\alpha$ to perform better (\cref{sec:experiments}).

%% file: sections/experiments.tex
\section{Experiments} \label{sec:experiments}

We empirically verify that our proposed method can effectively mitigate the adverse situation described in introduction --- the violation of the original symmetry axiom and its negative consequences (e.g., identical data are valued/priced very differently).
We further compare GAE's performance w.r.t.~other axioms and PDP, and in different scenarios with real-world datasets against existing methods.

\paragraph{Specific problem scenarios and comparison baselines.}
As mentioned in Sec.~\ref{sec:preliminaries}, our method is general, so we empirically investigate several specific problem scenarios in ML:
\textbf{P1.} Data point valuation quantifies the relative effects of each training example in improving the learning performance (to remove noisy training examples or actively collect more valuable training examples) \citep{bian2021energy,Ghorbani2020ADF,pmlr-v97-ghorbani19c,Jia2019nn,jia2019towards,Kwon2022}. 
\textbf{P2.} Dataset valuation aims to provide value of a dataset among several datasets (e.g., in a data marketplace) \citep{Agarwal2019_datamaketplace,ohrimenko2019collaborative,xu2021vol}. 
\textbf{P3.} Agent valuation in the CML/FL setting determines the contributions of the agents to design their compensations \citep{Hwee2020,Song2019-profitfl,Wang2020,xu2021gradient}.
\textbf{P4}. Feature attribution studies the relative importance of features in a model's predictions \citep{Covert2021_improvingkernelSHAP,lundberg2017_kernelSHAP}.
We investigate \textbf{P1.} in detail in Sec.~\ref{sec:experiments-P1} and \textbf{P2.} \textbf{P3.} and \textbf{P4.} together in Sec.~\ref{sec:experiments-P234}.
We compare with the following estimation methods (as baselines): MC \citep{Castro2009}, stratified sampling \citep{Castro2017,maleki2013}, multi-linear extension (Owen) \citep{Okhrati2020,Owen1972}, Sobol sequences \citep{Mitchell2021} and improved KernelSHAP (kernel) \citep{Covert2021_improvingkernelSHAP}.\footnote{We follow \citep{Covert2021_improvingkernelSHAP} as it provides an unbiased estimator where the original estimator \citep{lundberg2017_kernelSHAP} is only provably consistent and empirically shown to be less efficient \citep{Covert2021_improvingkernelSHAP}.}

\subsection{Investigating \ref{axiom:nullity}-\ref{axiom:desirability} and PDP within \textbf{P1.}}
\label{sec:experiments-P1}

\paragraph{Settings.}
We fit classifiers (e.g., logistic regression) on different datasets, use test accuracy as $v$ \citep{pmlr-v97-ghorbani19c,jia2019towards}, and adopt Data Shapley \citep[Proposition 2.1]{pmlr-v97-ghorbani19c} as the SV definition.
We randomly select $50$ training examples from a dataset and duplicate each once (i.e., a total of $n=100$ training examples).
Following Sec.~\ref{sec:fairness}, we set $\xi = 1\text{e-}3$.
For bootstrapping (included for all baselines), we uniformly randomly select $20$ permutations and evaluate the marginal contributions for each $i$.
We set a budget $m=2000$ for each baseline. As the true $\boldsymbol{\phi}$ are intractable for $n=100$, $\boldsymbol{\phi}$ is approximated via MC with significantly larger budget ($200$ bootstrapping permutations and $m=30000$, averaged over $10$ independent trials) as ground truth (in order to evaluate estimation errors) \citep{jia2019towards}.

\paragraph{Effect of $\underline{f}$ on symmetry \ref{axiom:symmetry}.}
\cref{proposition:chebyshev_fidelity} implies a higher $\underline{f}$ leads to a better fairness guarantee. We specifically investigate the effect of $\underline{f}$ in mitigating mis-estimations of \emph{identical} training examples and directly verify \ref{axiom:symmetry}.
We consider three evaluation metrics: lowest FS (i.e., $\underline{f}$); the proportion of duplicate pairs $i, i'$ with estimates having a deviation larger than a threshold $t$, i.e., $|\varphi_i - \varphi_{i'}| > t = \epsilon_1|\phi_i| + \xi\epsilon_1$ (as in~\ref{axiom:symmetry}); and the (log of) sum of the deviation ratio $\rho_{i,i'} \coloneqq \max(\varphi_i / \varphi_{i'}, \varphi_{i'} / \varphi_i)$ over all $i,i'$ pairs. 
In \cref{fig:k_value_dataset}a,c, the $\underline{f}$ of our methods increases (improves) as the number of samples, $m$, increases. However, the $\underline{f}$ of all baseline methods are stagnated close to $0$. Thus, as expected, our method significantly outperforms the baselines in obtaining a high $\underline{f}$. As predicted by \cref{proposition:chebyshev_fidelity}, this results in a lower probability and extent of fairness. As compared to baseline methods, our methods consistently have a lower proportion of identical examples that do not satisfy symmetry (\ref{axiom:symmetry}) (\cref{fig:k_value_dataset}b) and smaller deviation ratio of the estimated SV (\cref{fig:k_value_dataset}d).
\begin{figure}[!ht]
    \begin{tabular}{c@{}c}
    \includegraphics[width=0.46\linewidth]{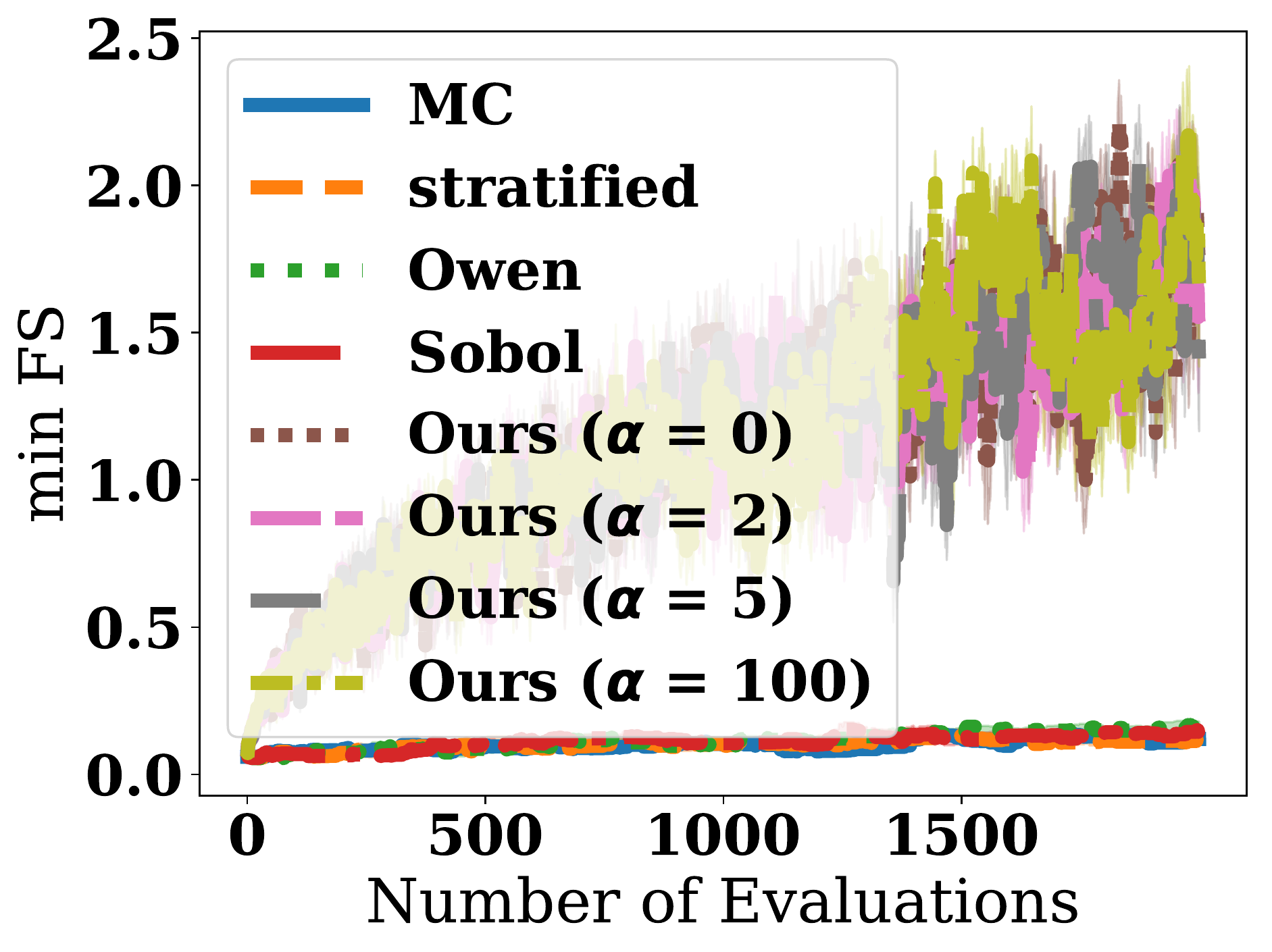} &
    \includegraphics[width=0.46\linewidth]{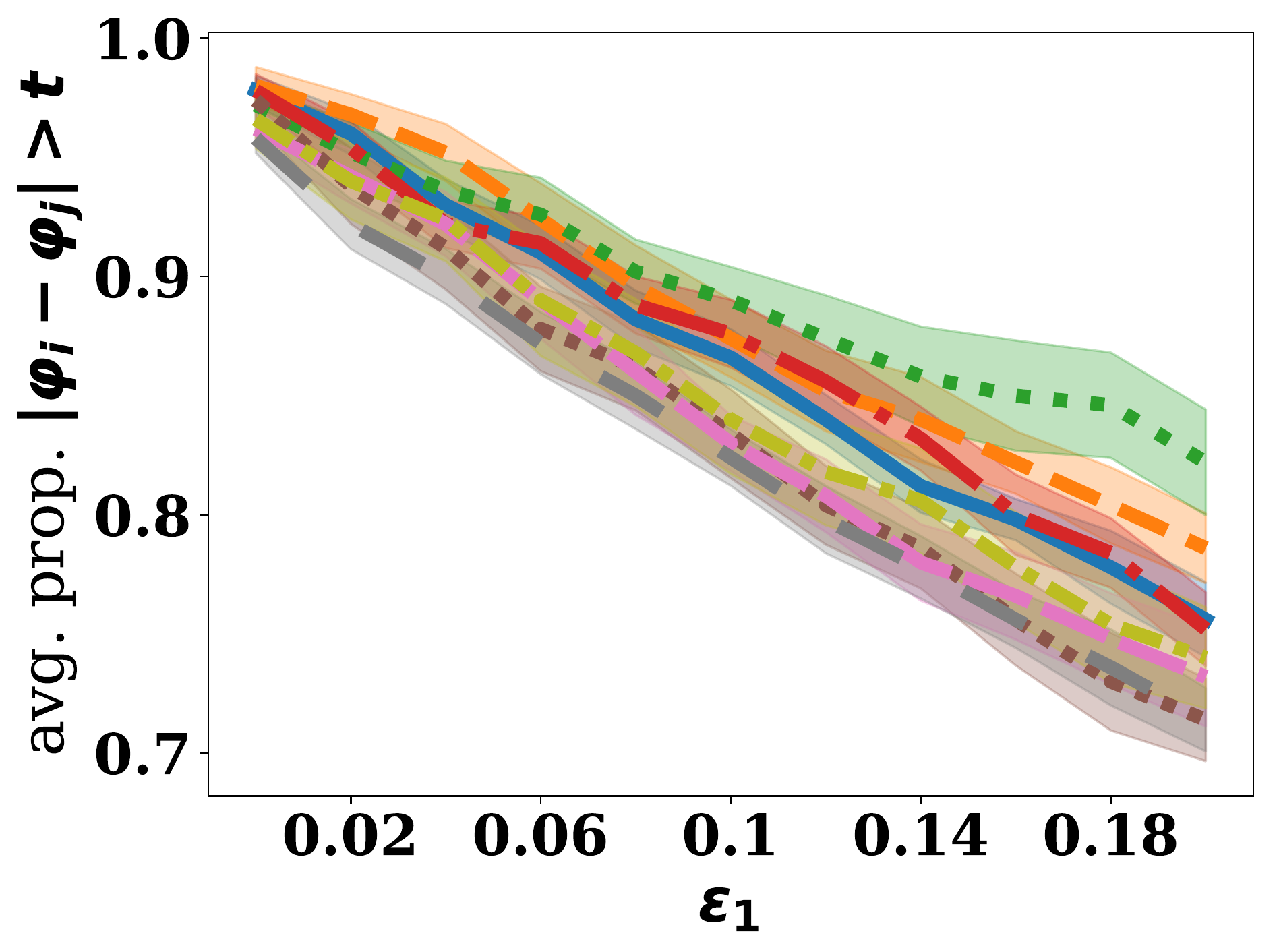} \\
    {(a) $\underline{f}$ vs. $m$} & {(b) \ref{axiom:symmetry} violation ($\%$) vs. $\epsilon_1$} \\
    \includegraphics[width=0.46\linewidth]{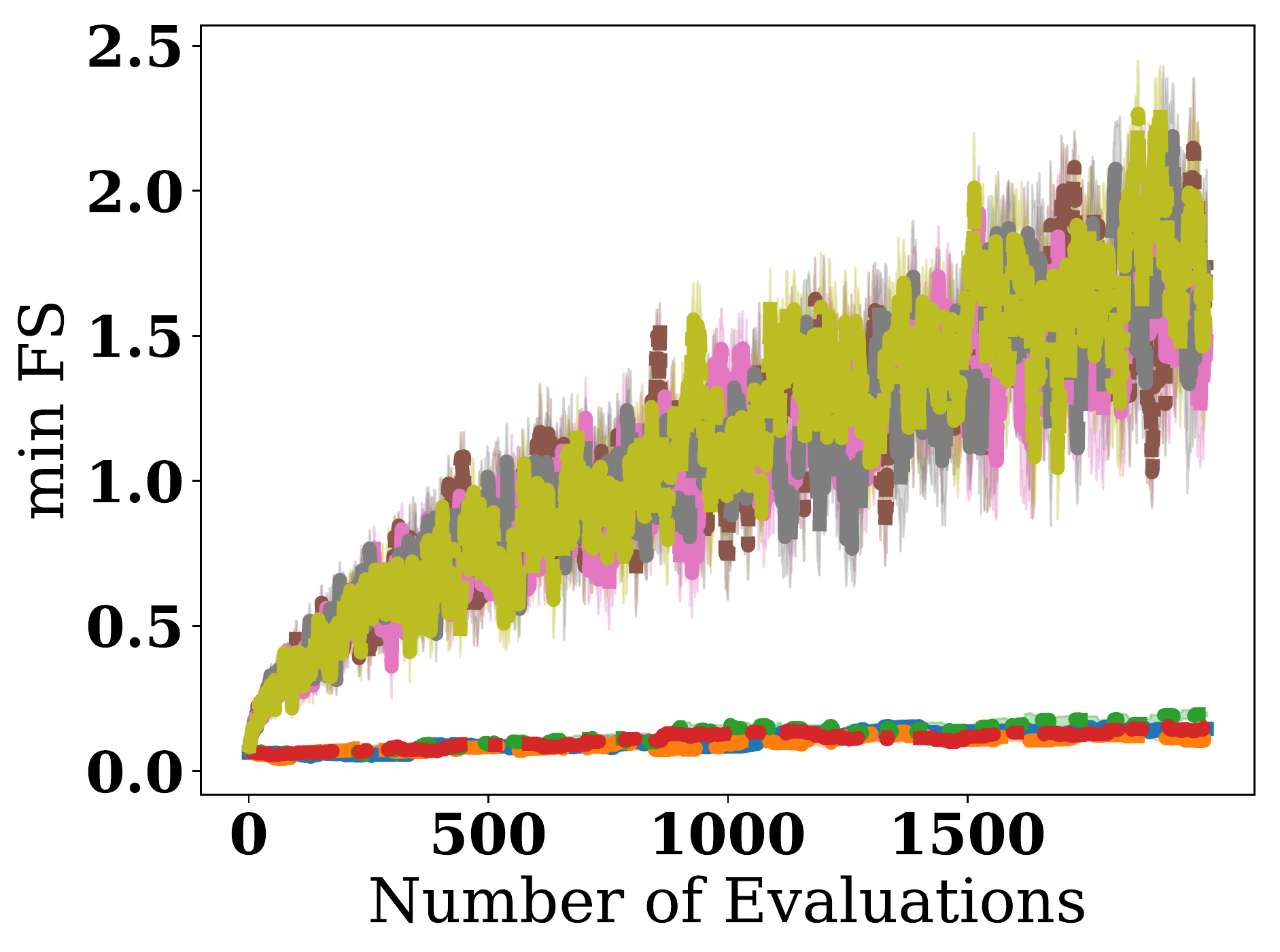} &
    \includegraphics[width=0.48\linewidth]{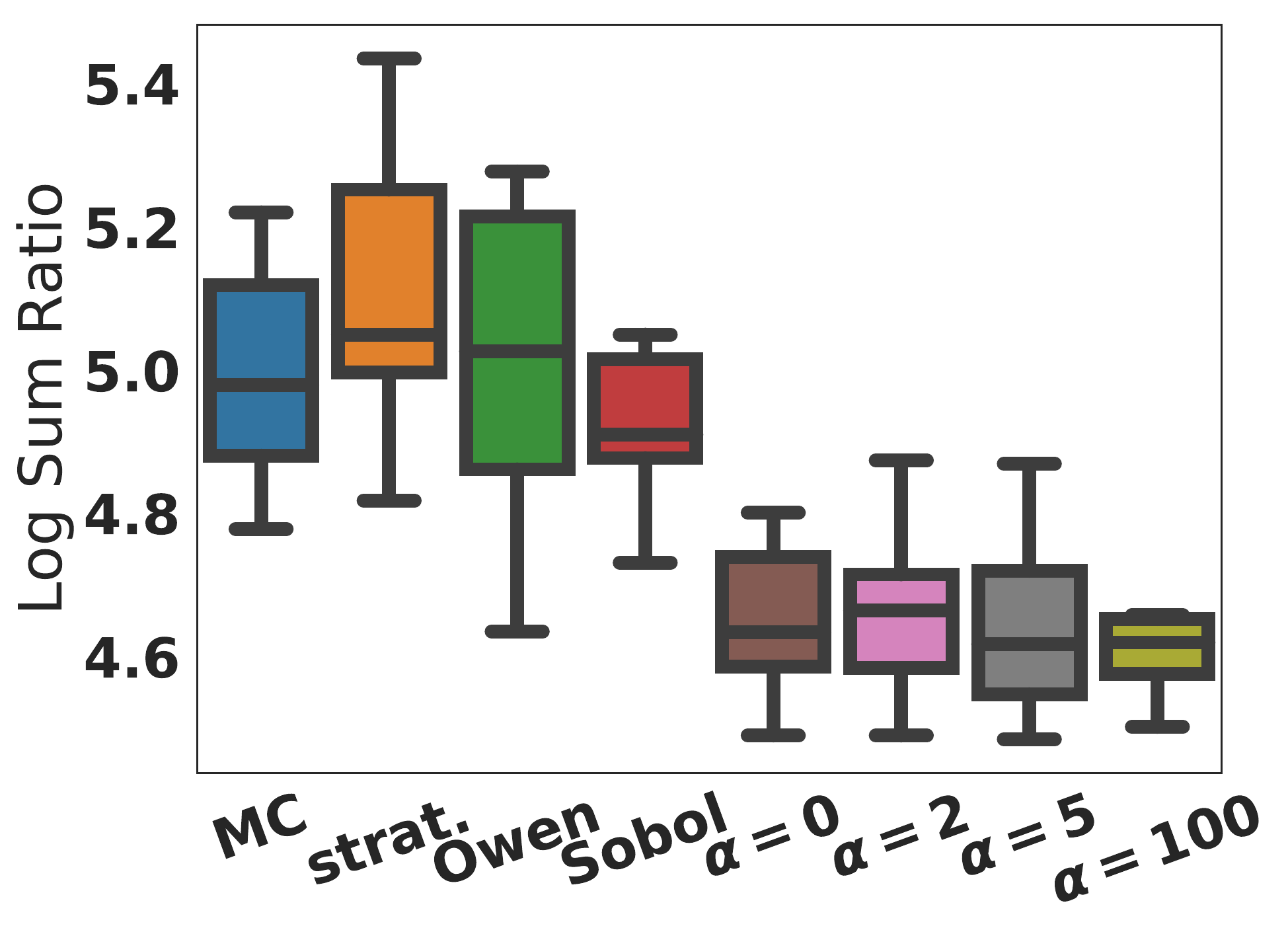} \\
    {(c)  $\underline{f}$ vs. $m$} & {(d) $\log \sum_{(i,i')} \rho_{i,i'}$} 
    \end{tabular}
    \caption{
    (a-b) and (c-d) plot results from using logistic regression on the synthetic Gaussian~\citep{Kwon2022} and Covertype~\citep{misc_covertype_31} datasets, respectively, using various baseline methods and ours.
    In (a,c), higher $\underline{f}$ is better while in (b,d), lower values are better.
    The intervals show the standard error of $10$ independent trials.
    }
    \label{fig:k_value_dataset}
\end{figure}

\paragraph{Verifying nullity \ref{axiom:nullity} and Pigou Dalton Principle.}
In practice, \ref{axiom:nullity} is rarely applicable (i.e., $\forall \pi, \sigma_i(\pi)=0$), 
so we instead investigate a more likely scenario: $|\phi_i|$ is very small for some $i$, because the training example $i$ has a minimal impact during training (e.g., $i$ is redundant).
We randomly draw $20$ training examples from the breast cancer dataset~\citep{breast_cancer} to fit a support vector machine classifier (SVC).
To verify \ref{axiom:nullity}, we standardize $\boldsymbol{\phi}$ 
and $\boldsymbol{\varphi}$ (i.e., $\boldsymbol{\phi}^\top \mathbf{1}_n  = \boldsymbol{\varphi}^\top  \mathbf{1}_n = 1$) and
calculate $\epsilon_{\text{abs}} \coloneqq \sum_{i: |\phi_i|\leq 0.01}|\varphi_i - \phi_i|$. As PDP is difficult to verify directly but it satisfies the Nash social welfare (NSW) \citep{DeClippel2010_nash_pdp,Kaneko1979}, we use the (negative log of) NSW on standardized $\boldsymbol{f}$ (i.e., $\boldsymbol{f}^{\top} \mathbf{1}_n =n$) $\text{NL NSW} \coloneqq -\log \prod_{i \in N} f_i$ (lower indicates PDP is better satisfied).
\cref{tab:nullity_pdp} shows our method obtains lowest estimation errors on (nearly) null training examples and satisfies PDP the best.

\begin{table}[!ht]
	\centering
    \setlength{\tabcolsep}{2pt}
    \caption{Average (standard errors) of error and NL NSW over $10$ independent trials.
    }
    \label{table:fairness-varying-beta-feature-noise}
    \resizebox{\linewidth}{!}{
    \footnotesize
    \begin{tabular}{lllrrrr}
    \toprule
    baselines &        $\epsilon_{\text{abs}}$ &     NL NSW \\
    \midrule
    MC &  2.12e-02 (2.75e-03) &  14.6 (5.12e-01)  \\
    Owen &  6.33e-02 (4.04e-03) &  18.2 (7.10e-01) \\
    Sobol &      6.28e-02 (5.64e-03) & 11.6 (6.24e-01) \\
    stratified &  2.89e-02 (5.84e-03) &   14.0 (7.46e-01) \\
    kernel &      6.44e-01 (1.15e-01) &  21.4 (1.48)  \\
    \midrule
    Ours ($\alpha = 0$) &  1.72e-02 (3.50e-03) &  \textbf{2.24} (7.40e-01) \\
    Ours ($\alpha = 2$) &  1.61e-02 (4.65e-03) &  2.57 (5.28e-01) \\
    Ours ($\alpha = 5$) &  \textbf{1.42e-02} (4.00e-03) &  2.55 (5.51e-01) \\
    Ours ($\alpha = 100$) &  2.13e-02 (5.95e-03) &  3.48 (5.61e-01) \\
    \bottomrule
    \end{tabular}
    }
    \label{tab:nullity_pdp}
\end{table}
\vspace{-1em}
\begin{figure}
    \centering
    \includegraphics[width=0.46\linewidth]{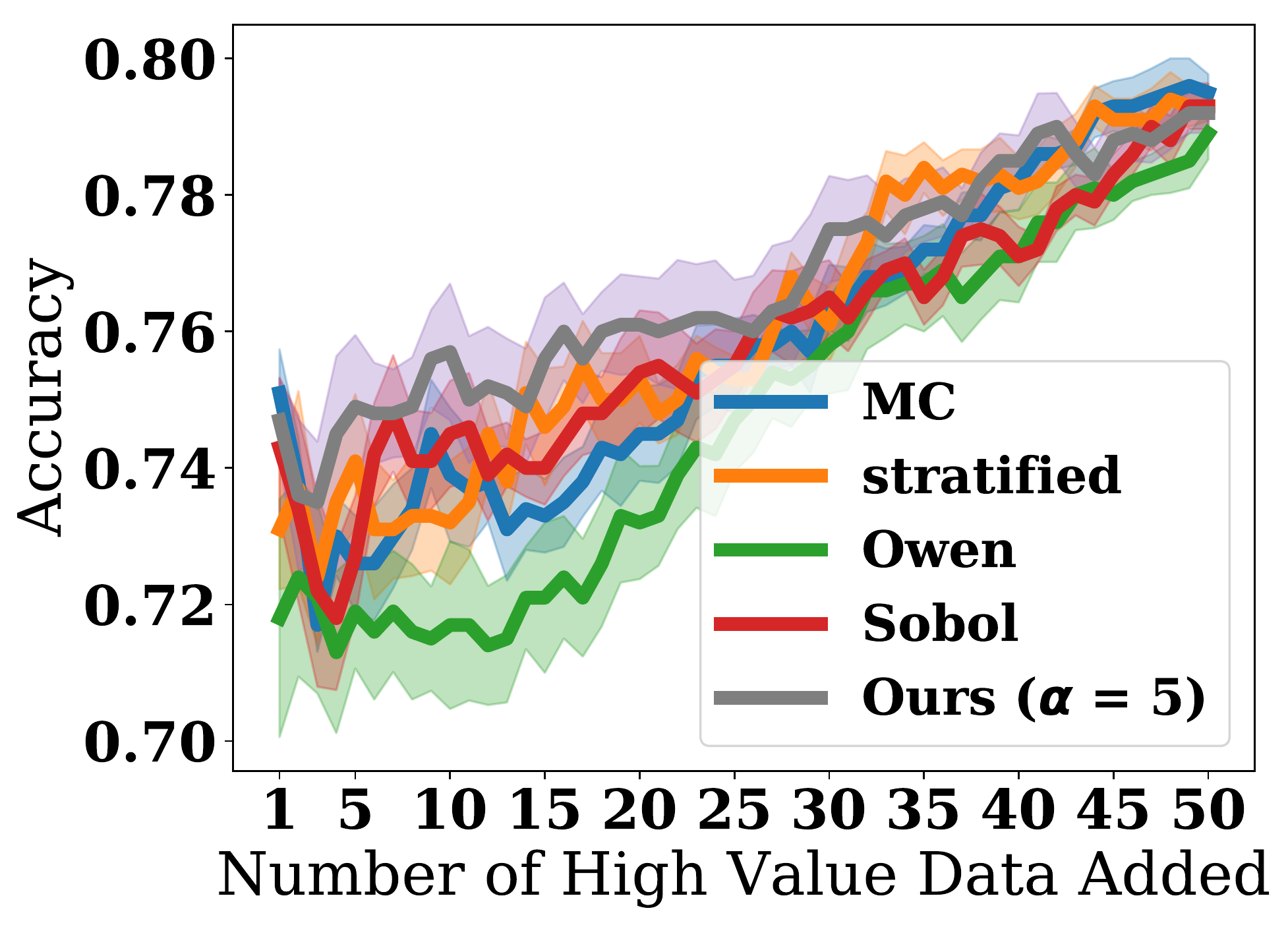}
    \includegraphics[width=0.49\linewidth]{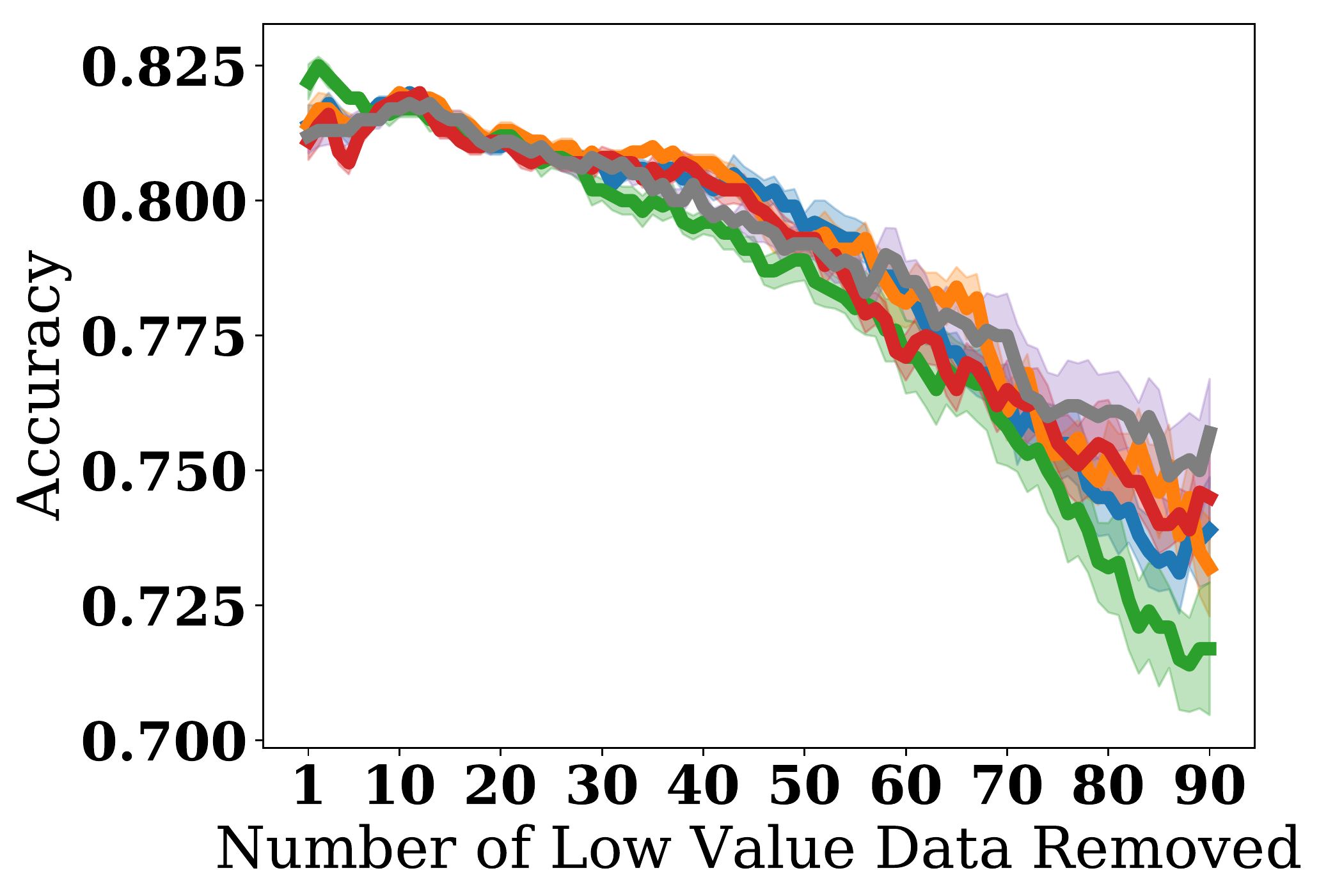}
    \captionof{figure}{Accuracy of logistic regression adding (removing) training examples generated from Gaussian distribution with highest (lowest) $\varphi_i$ in left (right).  Bootstrapping with $20$ permutations and $m=2000$. Average (standard errors in shaded color) over $10$ independent trials.  KernelSHAP  does not give clear trends in both plots because its estimates are not very accurate. Therefore we omit it here. A more detailed plot which involves KernelSHAP and other variants of our methods is provided in \cref{app:experiments}.
    }
    \label{fig:add-remove}
\end{figure}

\paragraph{Verifying approximate desirability \ref{axiom:desirability}.}
For \ref{axiom:desirability}, we verify whether the valuable training examples have high $\varphi_i$ by adding (removing) training examples according to highest (lowest) $\varphi_i$ \citep{pmlr-v97-ghorbani19c,Kwon2022} (\cref{fig:add-remove}) and via noisy label detections \citep{jia2019towards} (\cref{fig:f1_score} in \cref{app:experiments}).
\cref{fig:add-remove} left (right) shows our method is effective in identifying the most (least) valuable training examples to add (remove).

\subsection{Generalising to Other Scenarios}
\label{sec:experiments-P234}

We examine the estimation accuracy, \ref{axiom:desirability}, and PDP within \textbf{P2.}, \textbf{P3.} and \textbf{P4.}.\footnote{We exclude \textbf{P1.} and \textbf{P2.} because they are less likely to be applicable in these scenarios.}
For \textbf{P2.} we adopt robust volume SV \citep[Definition 3]{xu2021vol} (RVSV) and several real-world datasets for linear regression including used-car price prediction \citep{used_car_dataset} and credit card fraud detection \citep{credit_card_dataset} where $n$ data providers each owning a dataset (to estimate its RVSV). 
For \textbf{P3.} we consider \citep[Equation 1]{Hwee2020} (CML) and hotel reviews sentiment prediction \citep{hotel_reviews_dataset} and Uber-lyft rides price prediction \citep{uber_lyft_dataset}; in addition, we also consider \citep[Definition 1]{Wang2020} (FL) using two image recognition tasks (MNIST \citep{cnn_mnist} and CIFAR-10 \citep{cnn_cifar10}) and two natural language processing tasks (movie reviews \citep{pang2005seeing_mr} and Stanford Sentiment Treebank-5 \citep{kim2014convolutional_SST5}).
We partition the original dataset into $n$ subsets, each owned by an agent $i$ in FL/CML and we estimate each agent's contribution via the respective SV definitions.
For \textbf{P4.} we follow \citep[Theorem 1]{lundberg2017_kernelSHAP} on several datasets including adult income~\citep{adult_income}, iris~\citep{iris}, wine~\citep{wine}, and covertype~\citep{misc_covertype_31} classification with different ML algorithms including $k$NN, logistic regression, SVM, and multi-layer perceptron (MLP).
To ensure the experiments complete within reasonable time, we perform principal component analysis to obtain $7$ principal components/features for computing $\boldsymbol{\phi}$.\footnote{We find if $n\geq8$ features, the experiments take exceedingly long to complete due to the exponential complexity compounded further with the costly utility computation \citep[Equation 10]{lundberg2017_kernelSHAP}.
}
For \textbf{hyperparameters}, since the largest $n$ among these scenarios is $7$, we set the budget $m=1000$ and the bootstrapping of $300$ evaluations (a total of $1300$ evaluations for each baseline). We set $\xi = 1\text{e-}3$ and vary $\alpha \in \{0,2,5,100\}$ where $100$ simulates $\alpha \to \infty$.
Additional experimental details (datasets, ML models etc) are in \cref{app:experiments}.

\paragraph{Evaluation and results.}
We examine the mean squared error (MSE) and mean absolute percentage error
(MAPE) between $\boldsymbol{\varphi}$ and $\boldsymbol{\phi}$ for estimation accuracy, inversion counts $N_{\text{inv}}$ and errors $\epsilon_{\text{inv}}$ for \ref{axiom:desirability} and NL NSW (defined previously) for PDP.
The inversion count $N_{\text{inv}}\coloneqq \sum_{i\neq j \in N}\mathbb{1}(\phi_i > \phi_j \cap \varphi_i < \varphi_j) + \mathbb{1}(\phi_i < \phi_j \cap \varphi_i > \varphi_j)$ is the number of inverted pairings in $\boldsymbol{\varphi}$ while $\epsilon_{\text{inv}}\coloneqq \sum_{i\neq j \in N} |\phi_i - \phi_j -(\varphi_i -\varphi_j)|$ is the sum of absolute errors (w.r.t.~the true difference $\phi_i-\phi_j$).
We present one set of average (and standard errors) over $5$ repeated trials for \textbf{P2.} \textbf{P3.} and \textbf{P4.} each in 
\cref{tab:vol-credit-card-overall,tab:cml-hotelreviews-overall,tab:feature-wine-overall} (others in \cref{app:experiments}).
Overall, our method performs the best.
While most methods perform competitively to ours w.r.t.~MAPE, they are often worse (than ours) by an order of magnitude w.r.t.~MSE.
This is because our method explicitly addresses both the multiplicative and absolute errors (via $\xi = \epsilon_2 /\epsilon_1$ in FS). Specifically, reducing absolute errors when $|\phi_i|$ is large (e.g., RVSV for \textbf{P2.} or \citep[Equation 1]{Hwee2020} for \textbf{P3.} as both use the determinant of a large data matrix) is effective in reducing MSE.
In our experiments, we find kernelSHAP underperforms others, which may be attributed to it having a larger (co-)variance,\footnote{It is a co-variance matrix because kernelSHAP estimates the vector $\boldsymbol{\varphi}$ by solving a penalised regression.} empirically verified in \citep{Covert2021_improvingkernelSHAP}.

\begin{table}[!ht]
    \centering
        \caption{Evaluation of $\varphi_i$ within \textbf{P2.} using credit card dataset with $n=10$ data providers who each have a randomly sub-sampled dataset containing $100$ training examples \citep{xu2021vol}.
        }
        \resizebox{\linewidth}{!}{
        \begin{tabular}{lllrrrr}
        \toprule
        baselines &        MAPE &          MSE &  $N_{\text{inv}}$ &  $\epsilon_{\text{inv}}$ & NL NSW  \\
        \midrule
        MC &  3.87e-02 (7.9e-03) &  2.70e-03 (7.9e-04) &      3.60 (1.33) &      4.58 (0.82) &  1.72e-02 (4.6e-03) \\
        Owen &  3.06e-02 (6.7e-03) &  1.60e-03 (5.2e-04) &      4.00 (1.41) &      3.50 (0.76) &  1.31e-02 (3.6e-03) \\
        Sobol &  6.75e-02 (3.4e-03) &  9.62e-03 (1.5e-03) &      4.80 (1.20) &      7.97 (0.56) &  7.46e-02 (1.3e-02) \\
        stratified &  4.46e-02 (8.3e-03) &  3.30e-03 (8.3e-04) &      4.40 (1.17) &      5.17 (0.87) &  1.72e-02 (5.9e-03) \\
        kernel &      0.10 (2.0e-02) &  1.37e-02 (3.7e-03) &      8.80 (2.15) &  1.09e+01 (2.02) &         3.64 (0.46) \\
        \midrule
        Ours $(\alpha=0)$ &      0.10 (1.6e-02) &  2.15e-02 (8.0e-03) &  1.08e+01 (2.24) &  1.18e+01 (1.85) &         2.60 (0.88) \\
        Ours $(\alpha=2)$ &  2.30e-02 (2.5e-03) &  7.60e-04 (1.7e-04) &      2.80 (1.02) &      2.50 (0.29) &  \textbf{3.40e-04} (9.0e-05) \\
        Ours $(\alpha=5)$ &  \textbf{2.14e-02} (3.1e-03) &  \textbf{6.80e-04} (1.7e-04) &      \textbf{1.20} (0.49) &      \textbf{2.34} (0.31) &  9.90e-04 (9.0e-05) \\
        Ours $(\alpha=100)$ &  2.40e-02 (2.9e-03) &  9.90e-04 (2.8e-04) &      2.40 (0.75) &      2.77 (0.42) &  6.91e-03 (2.4e-03) \\
        \bottomrule
        \end{tabular}
        }
    \label{tab:vol-credit-card-overall}
\end{table}

\begin{table}[!ht]
    \centering
    \caption{Evaluation of $\varphi_i$ within \textbf{P3.} CML using hotel reviews dataset with $n=10$ agents who each have a randomly sub-sampled dataset containing $100$ training examples \citep{xu2021vol}. 
    }
    \resizebox{\linewidth}{!}{
    \begin{tabular}{lllllll}
    \toprule
    baselines &        MAPE &          MSE &  $N_{\text{inv}}$ &  $\epsilon_{\text{inv}}$ & NL NSW  \\
    \midrule
    MC &  3.53e-02 (6.0e-03) &  2.00e-03 (7.0e-04) &  1.32e+01 (2.94) &      4.40 (0.83) &      0.11 (3.0e-02) \\
    Owen &  3.06e-02 (1.3e-03) &  1.23e-03 (1.7e-04) &  1.00e+01 (1.41) &      3.70 (0.24) &      0.22 (5.0e-02) \\
    Sobol &  6.31e-02 (2.6e-03) &  1.17e-02 (9.4e-04) &  1.00e+01 (1.10) &      8.54 (0.33) &      0.28 (7.6e-02) \\
    stratified &  3.21e-02 (3.0e-03) &  1.67e-03 (2.8e-04) &  1.44e+01 (1.94) &      4.10 (0.37) &      0.20 (3.0e-02) \\
    kernel &      0.39 (5.7e-02) &      0.23 (6.4e-02) &  4.20e+01 (4.94) &  4.89e+01 (7.24) &         5.27 (1.34) \\
    \midrule
    Ours $(\alpha=0)$ &  1.17e-02 (9.5e-04) &  2.00e-04 (3.0e-05) &      \textbf{2.80} (0.49) &      1.48 (0.12) &  \textbf{2.74e-03} (2.1e-04) \\
    Ours $(\alpha=2)$ &  1.16e-02 (1.3e-03) &  2.00e-04 (5.0e-05) &      4.80 (1.96) &      1.47 (0.18) &  1.80e-02 (1.8e-03) \\
    Ours $(\alpha=5)$ &  1.18e-02 (7.4e-04) &  \textbf{1.80e-04} (3.0e-05) &      3.60 (1.33) &      1.43 (0.10) &  4.59e-02 (3.7e-03) \\
    Ours $(\alpha=100)$ &  \textbf{1.11e-02} (1.5e-03) &  \textbf{1.80e-04} (4.0e-05) &      4.80 (1.62) &      \textbf{1.40} (0.18) &  8.97e-02 (7.0e-03) \\
    \bottomrule
    \end{tabular}
    }
    \label{tab:cml-hotelreviews-overall}
\end{table}

\begin{table}[!ht]
    \centering
    \caption{Evaluation of $\varphi_i$ within \textbf{P4.} using the wine dataset on a randomly sampled subset of size $2000$ with $n=7$ principal features on a random forest classifier. 
    }
    \resizebox{\linewidth}{!}{
\begin{tabular}{lllllll}
\toprule
  baselines &                MAPE &                 MSE &    $N_{\text{inv}}$ &  $\epsilon_{\text{inv}}$ &              NL NSW \\
\midrule
         MC &  4.44e-02 (7.9e-03) &  6.47e-03 (1.7e-03) &         0.40 (0.40) &      2.54 (0.41) &      0.33 (3.1e-02) \\
       Owen &  9.88e-02 (1.0e-02) &  1.77e-02 (3.4e-03) &         1.20 (0.49) &      4.51 (0.44) &      0.44 (3.3e-02) \\
      Sobol &      0.24 (1.1e-02) &  7.37e-02 (2.4e-03) &         2.40 (0.40) &  1.18e+01 (0.26) &      1.32 (6.7e-02) \\
 stratified &  6.65e-02 (8.2e-03) &  9.36e-03 (2.4e-03) &  \textbf{0} (0.0e+00) &      2.62 (0.18) &      0.34 (2.7e-02) \\
     kernel &      0.15 (3.2e-02) &  1.56e-02 (4.1e-03) &         3.60 (1.47) &      6.08 (1.13) &     1.05e+01 (0.20) \\
     \midrule
   Ours $(\alpha=0)$ &  6.41e-02 (1.2e-02) &  1.12e-02 (6.2e-03) &         0.80 (0.49) &      3.16 (0.75) &      0.11 (2.7e-02) \\
   Ours $(\alpha=2)$ &  3.65e-02 (5.0e-03) &  2.49e-03 (7.7e-04) &  \textbf{0} (0.0e+00) &      1.80 (0.26) &  \textbf{8.40e-02} (1.1e-02) \\
   Ours $(\alpha=5)$ &  \textbf{3.14e-02} (6.1e-03) &  \textbf{2.02e-03} (5.7e-04) &  \textbf{0} (0.0e+00) &      \textbf{1.54} (0.29) &      0.16 (1.0e-02) \\
 Ours $(\alpha=100)$ &  3.28e-02 (4.4e-03) &  2.12e-03 (4.6e-04) &         0.40 (0.40) &      1.68 (0.23) &      0.29 (1.3e-02) \\
\bottomrule
\end{tabular}
        }
    \label{tab:feature-wine-overall}
\end{table}

%% file: sections/conclusion.tex
\section{Discussion and Conclusion}
We propose \emph{probably approximate Shapley fairness} via a re-axiomatisation of Shapley fairness and subsequently exploit an error-aware \emph{fidelity score} (FS) to provide a fairness guarantee with a polynomial (in $n$) budget complexity. 
We identify that jointly considering multiplicative and absolute errors (via their ratio $\xi$) is crucial in the quality of the fairness guarantee (which existing works did not do).
Through analysing the effect of $\xi$ on FS (used in our algorithm), we empirically find a suitable value for $\xi$.
To achieve the fairness guarantee, we propose a novel \emph{greedy active estimation} that integrates a greedy selection (which achieves a budget optimality) and active (permutation) selection via importance sampling. 
We identify that importance sampling can lead to poorer performance in practice as the necessary cardinality assumption may not be satisfied. To mitigate this, we describe a simple (via a single coefficient $\alpha$) regularisation using a uniform Dirichlet prior, that interestingly unifies the frequentist and Bayesian approaches (its effectiveness is empirically verified).
For future work, it is appealing to explore whether there exists a \emph{biased} estimator with much lower variance to provide a similar/better fairness guarantee with a competitive budget complexity.

%% file: appendix.tex
\appendix

\section{Algorithm Pseudo-code}
\label{app:algorithm}

Alg.~\ref{algo:active_valuation} presents the pseudo-code for GAE. Greedy selection iteratively picks the $\varphi_i$ with the lowest $f_i$ to update its SV estimate using permutations obtained from active selection. This process repeats until the total budget $m$ is exhausted.

\begin{algorithm}
\caption{Greedy Active Estimator}\label{algo:active_valuation}
\begin{algorithmic}
\Procedure{GAE}{$m', m, N, q', \xi$}
\State \text{\# $m'$: number of bootstrapping evaluations for each $\varphi_i$\ ,}
\State \text{\# $m$: the total budget for all evaluations, }
\State \text{\# $N$: the set to be valued,}
\State \text{\# $q'$: the proposal distribution that samples over the cardinalities $\{0,1,...,n-1\}$,}
\State \text{\# $\xi$: as in \cref{def:afs}.}
\State \text{Draw $m'$ permutations $\pi_t \sim U(\Pi),\ t \in \{1,2,...,m'\}$} \Comment{For Bootstrapping}
\State \text{Evaluate marginal contributions for each $i \in N$}
\State \text{Estimate parameters of $q$ using the marginal contributions (details in Appendix~\ref{app:estimate_params})}
\State \text{Initialize all FSs using the marginal contributions (details in the paragraph below)}
\State $m_i \gets m',\ \forall i \in N$
\For{$k \gets 1$ to $m$}
\State $j \gets \arg\min_{i\in N}\text{FS}(\varphi_i, \xi)$ \Comment{Greedy Selection}
\State $m_j \gets m_j + 1$
\State $N' \gets N \setminus \{j\}$
\State \text{Draw cardinality sample $c \sim q'$} \Comment{Active Selection}
\State \text{Draw a random permutation $\pi' \sim U(\text{perm}(N'))$}
\State $\pi_{m_j} \gets \{\pi'_0, \pi'_1, ..., \pi'_{c-1}, j, \pi'_c, .., \pi'_{n-2}\}$
\State \text{Evaluate $\sigma_j(\pi_{m_j})$}
\State $w_{m_j} \gets 1 / (nq'(c))$
\State $\varphi_j \gets (m_j-1)/m_j \times \varphi_j + w_{m_j}\sigma_j(\pi_{m_j})/m_j$
\State $s_j^2 \gets 1/(m_j-1) \times \sum_{t=1}^{m_j}(w_{t}\sigma_j(\pi_{t}) - \varphi_j)^2$
\State $f_j \gets m_j \times (\varphi_j + \xi)^2 / s_j^2$
\EndFor
\State \textbf{return} $\varphi_i$ for all $i \in N$
\EndProcedure
\end{algorithmic}
\end{algorithm}

\paragraph{Bootstrapping the estimation of FS and proposal distribution $q$.} 
To mitigate the cold-start problem of obtaining accurate FS (and a good proposal distribution $q$), we draw $m'$ bootstrapping permutations via MC \cite{bootstrap} to obtain $m'$ samples/marginal contributions for each $i$.
Subsequently, for each $i$, the sample mean  $\varphi_i = 1/m'\times \sum_{t=1}^{m'}\sigma_i(\pi_t)$ and sample variance $s_i^2 = 1/(m'-1)\times \sum_{t=1}^{m'} (\sigma_i(\pi_t) - \varphi_i)^2$ are then used to estimate the population mean $\mathbb{E}[\varphi_i] = \phi_i$ and population variance $\mathbb{E}[s_i^2] = \mathbb{V}[\sigma_i(\pi)] = m' \mathbb{V}[\varphi_i]$.
Both sample mean and sample variance are also used to estimate $f_i = \text{FS}(\varphi_i, \xi) \approx m' \times (\varphi_i + \xi)^2 / s_i^2$ used in greedy selection. 
Additionally, the marginal contributions obtained during bootstrapping are used to estimate the parameters of the proposal distribution $q$ to improve its fit to the actual distribution of marginal contributions.
While MC is used in bootstrapping, this choice is not restrictive and other approaches are also possible. For instance, if a good proposal distribution $q$ is known in advance (from prior knowledge), then $q$ can be applied both in bootstrapping and in active selection.

While our theoretical results (Propositions~\ref{proposition:optimality_of_greedy},\ref{prop:active_greedy_mc}) require all marginal contributions used to estimate $f_i$ to be drawn from a fixed distribution $q$ (i.e., technically we should discard the marginal contributions obtained from bootstrapping if the permutations are drawn from $U$ which may be different from $q$), as $m \gg m'$ in implementation, the $m'$ samples from bootstrapping has a marginal impact, so we keep the marginal contributions (for estimating $\varphi_i$) as an implementation choice.

\section{Additional Analysis and Discussion} \label{app:additional-analysis}

\subsection{Equivalent Probabilistic Formulation of \ref{original_axioms:nullity}-\ref{original_axioms:desirability} Using Conditional Events}

\begin{enumerate}[label*=F\arabic*., leftmargin=0.75cm, topsep=1pt]
    \item  Nullity: 
    let $E_{F_1}$ be the (conditional) event that for any $i \in N$, conditioned on $(\forall \pi \in \Pi, \sigma_i(\pi) = 0)$,  then $\phi_i = 0$.    
    \item  Symmetry: let $E_{F_2}$ be the (conditional) event that for all $i \neq j \in N$, conditioned on  $(\forall C \subseteq N \setminus \{i,j\}, v(C \cup \{i\}) = v(C \cup \{j\}))$, then $\phi_i = \phi_j$. 
    \item Strict desirability: let $E_{F_3}$ be the (conditional) event that for all $i \neq j \in N$, conditioned on
    $(\exists B \subseteq N \setminus \{i,j\}, v(B \cup \{i\}) > v(B \cup \{j\})) \wedge (\forall C \subseteq N \setminus \{i, j\}, v(C \cup \{i\}) \geq v(C \cup \{j\}))$, then $\phi_i > \phi_j$. 
\end{enumerate}
It can be verified that the probability of $E_{F_1}$-$E_{F_3}$ occurring is $1$, which is equivalent to the original formulation in the main text. Note that \ref{original_axioms:nullity}-\ref{original_axioms:desirability} and \ref{axiom:nullity}-\ref{axiom:desirability} are \textit{conditional} events, so that the probability of the clause being satisfied is \textit{not} relevant, instead the (conditional) probability of the implication, conditioned on the clause is satisfied, is what these axioms are designed to characterise and guarantee.

\subsection{Variance Results for Some Existing Estimators}
Continuing from the discussion in Sec.~\ref{sec:fairness}, we summarise the variance results of some existing estimators in \cref{tab:var-fairness-result}.

\begin{table}[!ht]
    \centering
    \caption{Variance results for some existing estimators. The variances for some estimators (e.g., Owen) are not available.}
    \label{tab:var-fairness-result}
    \resizebox{\linewidth}{!}{
    \begin{tabular}{c|c|c|c}
    \toprule
    estimators         &  variance & remark & reference\\
    \midrule
    MC     &  $\mathbb{V}_{\pi \sim U}[\sigma_i(\pi)] / m$ & N.A. & \citep[Proposition 3.1]{Castro2009}\\
    stratified & not available & a probabilistic error bound for each $\varphi_i$ & \citep{maleki2013} \\
    Owen & not available  & N.A. & \citep{Okhrati2020} \\
    antithetic  & $\mathbb{V}_{\pi \sim U}[\sigma_i(\pi)] / m \times [1 + \text{Corr}(\sigma_i(X), \sigma_i(Y))]$  &  dependent on correlation & \citep[Equation 4]{Mitchell2021}\\
    orthogonal  & $1/m\sum_{l=1}^{n/k}\sum_{j,r=1}^{k}\text{Cov}(\sigma_i(\pi_{lj}), \sigma_i(\pi_{lr}))$  & dependent on covariance & \citep[Section 4.2]{Mitchell2021} \\
    Sobol & not available & N.A. & \citep{Mitchell2021}\\
    kernel & $M \Sigma_{\boldsymbol{\varphi}} M^\top$ & covariance matrix & \citep[Equation 12]{Covert2021_improvingkernelSHAP}. \\
    Ours     & $\mathbb{V}_{\pi \sim q_i}[\sigma_{i,q_i}(\pi)] / m$ & dependent on the proposal distribution $q_i$ & Equ.~\eqref{eq:importance_var}  \\
    \bottomrule
    \end{tabular}
    }
\end{table}

Stratified sampling has a probabilistic error bound (which additionally requires $\max_C \nu(C)$ and $\min_C \nu(C)$) for each $\varphi_i$ individually \citep[Equation 15]{maleki2013}:
with probability at least $(1-\delta)^n$,
\begin{equation*}
    |\varphi_i - \phi_i| \leq \frac{d\sqrt{-\ln\frac{\delta}{2}}}{n\sqrt{m_i}}\left(\sum_{k=0}^{n-1}(k+1)^{\frac{2}{3}}\right)^{\frac{3}{2}},
\end{equation*}
where $d = 2(\max v(C)/|C| - \min v(C)/|C|)$ and $m_i$ is the budget for evaluating $\varphi_i$.
The contrast with our probably approximate fairness is that it does not consider the interaction between $\varphi_i,\varphi_j$.

Antithetic is used in Owen and orthogonal is an extension of antithetic \citep{Mitchell2021}. $m$ refers to the number of sample permutations drawn.
The variance of antithetic sampling estimator \citep[Equation 4]{Mitchell2021} depends on the correlation of a randomly sampled vector from a unit cube $X\sim U(0,1)^n$ and its complement $Y = \mathbf{1}_n - X$ element-wise.
The orthogonal estimator \citep[Equation 14]{Mitchell2021} extends the antithetic sampling estimator, so its variance expression also extends that of the antithetic sampling estimator.
Hence, to analyse the fairness guarantee, additional assumptions on the correlation or covariance are required.

Unfortunately, the variance results are not provided for the Sobol estimator \citep[Algorithm 4]{Mitchell2021} or Owen estimator \citep[Algorithm 1]{Okhrati2020} (compared in our experiments because they have empirically good performance \citep{Mitchell2021,Okhrati2020}), so it is difficult to apply our proposed fairness framework to theoretically analyse their fairness.

Note that KernelSHAP is not a sampling-based estimator as it obtains $\boldsymbol{\varphi}$ by solving a regression. Hence, its variance result is a covariance matrix shown above with $M\coloneqq A^{-1} -  (A^{-1} \mathbf{1} \mathbf{1}^\top A^{-1})   /  (\mathbf{1}^\top A^{-1} \mathbf{1}) $ where $A_{ii} = 1/2$ and
\[A_{ij,i\neq j} = \frac{1}{n(n-1)} \frac{\sum_{l=2}^{n-1}\frac{l-1}{n-l} }{\sum_{j=1}^{n-1} 1/(l (n-l))}\ , \]
and $\Sigma_{\boldsymbol{\varphi}}$ is the covariance matrix of (the vector) $\boldsymbol{\varphi}$.

From \cref{tab:var-fairness-result}, it is appealing to provide the corresponding variance results for the existing \emph{unbiased} estimators in order to analyse their fairness guarantee as in \cref{definition:fairness-guarantee}.
Furthermore, it is an interesting future direction to consider \emph{biased} estimators with lower variance to achieve competitive fairness guarantees. For instance, a weighted importance sampling method for off-policy MC evaluation \citep{Sutton1998} leads to a biased estimator but empirically leads to faster convergence on some reinforcement learning tasks. More broadly, exploring how the variance vs.~bias trade-off affects the fairness of an SV estimator is also an interesting direction to explore. 

\subsection{Limitations of MC}

Firstly, MC is commonly adopted for each $i\in N$ separately \citep{Castro2009,maleki2013} (i.e., estimating $\varphi_i$ independently of others).
As we have demonstrated in \cref{fig:ri_mape} and \cref{fig:ri_ape_app}, $f_i \neq f_j$ in general even with the same number of permutation samples - some $\varphi_i$ observes larger variations than others due to large variation in their marginal contributions $\sigma_i(\pi)$. However, MC does not take this into consideration at all and as a result, cannot effectively optimise $\max \min_i f_i$ (optimising which provides the fairness guarantee as in \cref{proposition:chebyshev_fidelity} and \cref{corollary:k_axioms})

Secondly, in MC, the permutations are selected uniformly randomly. However, it cannot guarantee that each $i$ receives permutations that are equally helpful in obtaining an accurate $\varphi_i$ (i.e., some permutations are more helpful in estimating $\varphi_j$ while some others more helpful for $\varphi_i$). For instance, the marginal contributions to subsets with smaller cardinality tend to be much larger in magnitude \cite{maleki2013,pmlr-v97-ghorbani19c,Kwon2022}.
Intuitively, this echos diminishing returns where the performance improvement from adding a single training example to a small training set is more significant than that from adding the same training example to a large training set \citep{wang2021learnability,xu2021vol}. Without accounting for this, MC produces marginal contributions that have high variance.

The first limitation is addressed using greedy selection (Sec.~\ref{sec:greedy_selection}). The second limitation is commonly tackled using importance sampling~\cite{kloek1978}: Permutations are drawn (randomly selected) from a proposal distribution $q_i(\pi)$ with densities according to $|\sigma_i(\pi)|$. 
Importance sampling can produce an unbiased estimator,
\begin{equation*}
    \phi_i = \mathbb{E}_{\pi \sim U(\pi)}[\sigma_i(\pi)] = \mathbb{E}_{\pi \sim q_i(\pi)}\left[\frac{U(\pi)\sigma_i(\pi)}{q_i(\pi)}\right] = \mathbb{E}_{\pi \sim q_i(\pi)}\left[\frac{\sigma_i(\pi)}{n!q_i(\pi)}\right],
\end{equation*}
where $U$ is a uniform distribution over all permutations $\Pi$. Theoretically, the optimal proposal distribution $q_i^*$ corresponds to when $\mathbb{V}[\sigma_i(\pi) / (q_i^*(\pi)n!)]$ is minimised.
It implies the density is proportional to the magnitude of the true marginal contribution, i.e., $q_i^* \propto |\sigma_i(\pi)|$~\cite{importance_sampling_analysis}. 
However, $q_i^*$ is unattainable in practice without knowing all the marginal contributions.

\subsection{Estimating the Parameters of the Categorical Distribution}
\label{app:estimate_params}

It is natural to model the distribution of $|P_i^\pi|$ as a categorical distribution with a support as the different cardinalities $\{0,1,\ldots,n-1\}$. To obtain the parameters of the categorical distribution, one way is to directly estimate $\sqrt{\mathbb{E}_{\pi \sim U_c}[\sigma_i(\pi)^2]}$ (i.e., an MLE estimate as in Appendix~\ref{app:var}) where $U_c$ refers to the uniform distribution over all permutations of cardinality $c$: $\{\pi: \pi \in \Pi, |P_i^\pi|=c\}$. Another approach balances this estimation with a prior belief via a Dirichlet distribution (i.e., an MAP estimate) paramameterised by $\{\alpha_0,\alpha_1,\ldots,\alpha_{n-1}\}$, denoted as $\text{Dir}((\alpha_0,\alpha_1,\ldots,\alpha_{n-1}))$.
The density of a $n$-dimensional probability vector $\{x_0, x_1,\ldots,x_{n-1}\}$ w.r.t.~this Dirichlet distribution is
\begin{equation*}
    f(x_0, x_1, ..., x_{n-1}; \alpha_0, \alpha_1, ..., \alpha_{n-1}) = \frac{\Gamma\left( \sum_{k=0}^{n-1}\alpha_k\right)}{\prod_{k=0}^{n-1}\Gamma(\alpha_k)} \prod_{k=0}^{n-1}x_k^{\alpha_k-1}\ .
\end{equation*}

We combine both approaches by setting an uninformative prior $\text{Dir}((\alpha+1)\boldsymbol{1}_n)$ (i.e., $\alpha_0=\alpha_1=\ldots=\alpha_{n-1}=\alpha+1$) with a controllable strength via $\alpha$ (a larger $\alpha$ means we trust this prior more). We use the marginal contributions gathered from bootstrapping to update the Dirichlet distribution $\text{Dir}((\alpha+1)\boldsymbol{1}_n)$. The goal is for the proposal distribution $q'$ to approximate the optimal distribution well, specifically $q'^*$ (as shown in Equ.~\eqref{eq:optimal_q_card} later).
To do so, we approximate $\mathbb{E}_{\pi \sim U_c}[\sigma_i(\pi)^2]$ which is a key component for $q'^*$.

\textbf{Interpreting the MAP perspective.}
Let the MLE estimate be $\boldsymbol{w}$, the MAP estimate is then obtained by adding a constant term $\alpha$ to a scaled version of $\boldsymbol{w}$, denoted as $\boldsymbol{\tilde{w}} = n\boldsymbol{w}$. The scaling operation can be
approximately viewed as drawing $n$ samples from $q'^*$,\footnote{The choice of number of samples is artificial. The relative magnitude of number of samples and $\alpha$ determines the relative importance between $\boldsymbol{w}$ and the prior belief (Dirichlet parameterised by $\alpha$). We choose $n$ for convenience of notation although other sample sizes can achieve the same effect (with $\alpha$ adjusted accordingly).} with $\tilde{w}_c$ falling into the $c$-th category. Therefore, $\boldsymbol{\tilde{w}}$ can be treated as coming from a categorical distribution $q'^*$, which combines with the Dirichlet prior to give an MAP estimate $\alpha + \tilde{w}_c$ for each category $c$.

Specifically, using the marginal contributions from bootstrapping, a probability simplex $\boldsymbol{w}$ (s.t.~$\sum_{k=0}^{n-1} w_k = 1$) is constructed where $w_c$ is proportional to the mean absolute marginal contribution with cardinality $c$, i.e., $w_c \propto 1/m_i \sum_{t=1}^{m_i} \sigma_i(\pi_t)^2$ for $m_i$ samples drawn from $U_c$ to evaluate $\varphi_i$. 
Then the scaled $\tilde{\boldsymbol{w}} \gets n\boldsymbol{w}$ is used to obtain the MAP estimate for the parameters $\boldsymbol{\theta}$ of the proposal distribution $q'$ as $\theta_c = (\tilde{w}_c+ \alpha) / \sum_{k=0}^{n-1}(\tilde{w}_k+ \alpha)$. 
In particular, $\alpha = 0$ recovers the MLE estimate. A detailed proof is given below.

\begin{proof}
Notice that $\tilde{w}_c$ can be seen as the number of data belonging to the $c$-th category from $n$ observations. Further let $\boldsymbol{\theta} \sim \text{Dir}((\alpha+1)\boldsymbol{1}_n)$. Then, the posterior distribution of $\boldsymbol{\theta}$ is
\begin{equation*}
    \begin{aligned}
    \text{Pr}(\boldsymbol{\theta}|\boldsymbol{\tilde{w}}) &\propto \text{Pr}(\boldsymbol{\tilde{w}}|\boldsymbol{\theta})\text{Pr}(\boldsymbol{\theta}|\alpha) \\
    \implies \log \text{Pr}(\boldsymbol{\theta}|\boldsymbol{\tilde{w}}) &\propto \log \text{Pr}(\boldsymbol{\tilde{w}}|\boldsymbol{\theta}) + \log \text{Pr}(\boldsymbol{\theta}|\alpha)\ .
    \end{aligned}
\end{equation*}
Incorporate this equation with the constraint $\sum_{k=0}^{n-1}\theta_k = 1$ to form the Lagrangian,
\begin{equation*}
\begin{aligned}
   \mathcal{L}(\alpha, \boldsymbol{\theta}) &= \log \text{Pr}(\boldsymbol{\tilde{w}}|\boldsymbol{\theta}) + \log \text{Pr}(\boldsymbol{\theta}|\alpha) + \lambda \left(1 - \sum_{k=0}^{n-1}\theta_k \right) \\
   &= \sum_{k=0}^{n-1}\tilde{w}_k \log \theta_k + \sum_{k=0}^{n-1} \alpha\log \theta_k + \lambda \left(1 - \sum_{k=0}^{n-1}\theta_k \right).
\end{aligned}
\end{equation*}
Set its partial derivative w.r.t.~$\theta_c$ to $0$ and solve for $\theta_c$, 
we get the MAP estimate for $\theta_c$
\begin{equation*}
    \theta_c^{\text{MAP}} = \argmax_{\theta_c} \mathcal{L}(\alpha, \boldsymbol{\theta})  = \frac{\tilde{w}_c + \alpha}{\sum_{k=0}^{n-1} (\tilde{w}_k + \alpha)}\ ,
\end{equation*}
which completes the proof.
\end{proof}

\subsection{Greedy Active Estimator is Unbiased}

We show that GAE produces unbiased estimates for all $\varphi_i,\ i\in N$.

\begin{proposition}[\bf Unbiasedness of GAE] \label{proposition:unbiasednes_of_fae}
Given a proposal distribution with support $\{0,1,2,...,n-1\}$, $\forall i, \mathbb{E}[\varphi_i] = \phi_i$ where each $\varphi_i$ is obtained from applying GAE.
\end{proposition}
\begin{proof}[Proof of \cref{proposition:unbiasednes_of_fae}]
Let $q$ be the sampling distribution that samples permutations according to our importance sampling method and proposal distribution. Since the support of our proposal distribution is the set of all cardinalities and every permutation of a fixed cardinality has equal (and therefore positive) probability of being selected, the support of $q$ is $\Pi$. Next, consider that for each $i \in N$, GAE produces an estimate $\varphi_i = \mathbb{E}_{\pi \sim q}[U(\pi)\sigma_i(\pi)/q(\pi)]$. Since the support of $q$ is $\Pi$, we have $\forall \pi \in \Pi, q(\pi) > 0$. Hence $\varphi_i = \mathbb{E}_{\pi \sim q}[U(\pi)\sigma_i(\pi)/q(\pi)] = \mathbb{E}_{\pi \sim U}[\sigma_i(\pi)] = 1/(n!)\sum_{\pi \in \Pi}\sigma_i(\pi) = \phi_i$ by Equ.~\eqref{eq:shapley}.
\end{proof}

\section{Proofs and Derivations} \label{app:proofs}

\subsection{Proof of \cref{proposition:chebyshev_fidelity}}

To aid the proofs of the guarantee of Axioms~\ref{axiom:nullity}-\ref{axiom:desirability}, we introduce the following intermediate \textit{fidelity} axiom: 

\begin{enumerate}[label*=A\arabic*., ref=A\arabic*,leftmargin=0.75cm, topsep=1pt, start=0]
\item \label{axiom:general} Closeness: For any $i \in N$, the estimated SV $\varphi_i$, can deviate from the true SV $\phi_i$ by a relative error $\epsilon_1$ and an absolute error $\epsilon_2$:
\begin{equation}\label{equ:A0}
    |\phi_i - \varphi_i |\leq \epsilon_1 |\phi_i| + \epsilon_2\ .
\end{equation}
\end{enumerate}

\begin{lemma} \label{lem:fidelity_axiom}
Denote $\underline{f} \coloneqq \min_{i\in N} f_i$, then the fidelity axiom ~\ref{axiom:general} with error parameters $\epsilon_1,\epsilon_2$ in \eqref{equ:A0} is satisfied for any $i\in N$ w.p.~$\geq 1- 1/(\epsilon_1^2 f_i) \geq 1 - 1/(\epsilon_1^2 \underline{f})$.

\begin{proof}[Proof of \cref{lem:fidelity_axiom}]
From \eqref{equ:A0},
\begin{align*}
    \begin{split}
        \text{Pr}[ |\phi_i - \varphi_i |\leq \epsilon_1 |\phi_i| + \epsilon_2]  &= 1 -  \text{Pr}[ |\phi_i - \varphi_i |> \epsilon_1 |\phi_i| + \epsilon_2 ] \\
        &\geq 1 - \frac{\mathbb{V}[\varphi_i]}{(\epsilon_1^2|\phi_i|+\epsilon_2)^2} \\
        &= 1 - \frac{1}{\epsilon_1^2 f_i}\\
        &\geq 1 - \frac{1}{\epsilon_1^2 \underline{f}}\ .
    \end{split}
\end{align*}
The first inequality is by Chebyshev's inequality and the last inequality is by $\underline{f} = \min_{i\in N} f_i$. 
\end{proof}

\end{lemma}

\begin{proof}[Proof of \cref{proposition:chebyshev_fidelity}]
By Chebyshev's inequality, A0 in Equ.~\eqref{equ:A0} holds for some $i\in N$ w.p.~$\geq 1 - 1/(\epsilon_1^2 f_i)$ (\cref{lem:fidelity_axiom}). Let A0 hold for all $i\in N$ simultaneously, w.p.~$\geq 1-\delta_{A0}$ (the expression for $\delta_{A0}$ is derived later).

The case for \ref{axiom:nullity} is straightforward by substituting $\phi_i=0$ into \eqref{equ:A0}.

For \ref{axiom:symmetry} and \ref{axiom:desirability}, first the apply triangle inequality,
\[ |\phi_i - \varphi_i -\phi_j + \varphi_j| \leq | \phi_i - \varphi_i| +|-\phi_j + \varphi_j|\ , \]
then apply \eqref{equ:A0},
\begin{equation} \label{equ:intermediate}
    |\phi_i - \varphi_i -\phi_j + \varphi_j| \leq \epsilon_1(|\phi_i|+|\phi_j|) + 2 \epsilon_2\ .
\end{equation}

For \ref{axiom:symmetry}, substitute $\phi_i = \phi_j$ into \eqref{equ:intermediate}, 
\[ | - \varphi_i + \varphi_j| \leq \epsilon_1(|\phi_i|+|\phi_j|) + 2 \epsilon_2\ .\]

For \ref{axiom:desirability}, substitute $\delta_{ij}\coloneqq \phi_i - \phi_j$ and $\hat{\delta}_{ij}\coloneqq \varphi_i -\varphi_j$ into \eqref{equ:intermediate},
\[ |\delta_{ij} - \hat{\delta}_{ij}| \leq \epsilon_1(|\phi_i|+|\phi_j|) + 2 \epsilon_2\ .\]

Hence, \ref{axiom:nullity}-\ref{axiom:desirability} are satisfied (for all $i\in N$ or $i,j\in N$) if A0 is holds for all $i\in N$ simultaneously. 

To derive $\delta_{A0}$, consider two cases: i) all $\varphi_i$'s are independent, then $\delta_{A0} \leq (1-1/(\epsilon_1^2\underline{f}))^n$.

Otherwise, ii) $\delta_{A0} \leq n/(\epsilon_1^2\underline{f})$ (by the union bound of the complement):
\begin{align*}
    \text{Pr}[\forall i\in N, \text{A0 holds}] &= 1 - \text{Pr}[\exists i\in N, \text{A0 does not hold}] \\ 
    &\geq 1 - \sum_i \text{Pr}[\text{A0 does not hold for } i] \\  
    &\geq 1 - n / (\epsilon_1^2 \underline{f})
\end{align*}
where the first inequality is by the union bound and the second inequality is by \cref{lem:fidelity_axiom}.
Lastly, substituting \cref{definition:fairness-guarantee} completes the proof.
\end{proof}

\paragraph{Remark 1.} 
If all $\varphi_i$'s are independent, a tighter bound for fairness (shown above) is $p\coloneqq \prod_i(1 - 1/(\epsilon_1^2f_i))$ (w.p.~$\geq 1-p$ the properties are satisfied w.r.t.~$\epsilon_1,\epsilon_2$). Therefore, to improve fairness, we can instead improve its lower bound $p$. Specifically, we consider which $i$ to evaluate one additional budget so as to maximise the improvement in $p$. Let $f_i$ and $f_i'$ denote the fidelity score of $i$ when it has received $m_i$ budget and after being evaluated one additional budget (i.e., $m_i+1$ based on the $\varphi_i$ of the $m_i$ budget), respectively.

The (multiplicative) improvement $\Delta_p$ of $p$ is defined as $\Delta_p \coloneqq p' / p$ where $p'$ refers to the fairness bound after receiving one additional evaluation. Then, 
\begin{equation*}
    \Delta_p = \frac{p'}{p} = \frac{1 - \frac{1}{\epsilon_1^2f_i'}}{1-\frac{1}{\epsilon_1^2f_i}} = \frac{1 - \frac{1}{\epsilon_1^2(m_i+1)r_i}}{1-\frac{1}{\epsilon_1^2m_ir_i}} = \frac{\epsilon_1^2r_i - \frac{1}{m_i+1}}{\epsilon_1^2r_i - \frac{1}{m_i}}\ .
\end{equation*}
Observe that to maximise $\Delta_p$ (equivalently to improve $p$) which depends on $r_i, m_i$, we should select $i$ s.t., $r_i,m_i$ maximise $\Delta_p$ each time. This observation gives rise to a modified version of the Greedy Active Estimator (GAE): Instead of evaluating $\argmin_i f_i$, the modified algorithm evaluates $\argmax_i \Delta_p(r_i, m_i)$. However, one practical limitation is that $\epsilon_1$ has to be specified, which implies the improved bound (i.e., better $p$) is specific to some fixed $\epsilon_1$ and may not generalise to other values of $\epsilon_1$. Note that GAE circumvents this limitation since $\underline{f}$ does not require a specified $\epsilon_1$. 

Moreover, through analysing the partial derivatives of $\Delta_p$ w.r.t $r_i,m_i$, we believe finding $i$ via $\argmin f_i$ is a good surrogate for $\argmax \Delta_p(r_i,m_i)$ without having to pre-specify $\epsilon_1$.
Note that
\begin{equation*}
    \Delta_p = 1 + \frac{1}{\epsilon_1^2r_i + 1}\left( \frac{1}{m_i + 1} + \frac{\epsilon_1^2r_i}{1 - \epsilon_1^2r_im_i} \right)\ ,
\end{equation*}
which gives the partial derivative w.r.t. $r_i$ and $m_i$ as
\begin{equation*}
\begin{aligned}
        \frac{\partial \Delta_p}{\partial r_i} &= - \frac{\epsilon_1^2m_i}{(\epsilon_1^2r_im_i)^2(m_i+1)}\ , \\
        \frac{\partial \Delta_p}{\partial m_i} &= \frac{1}{\epsilon_1^2r_i + 1}\left( \frac{\epsilon_1^4r_i^2}{(1-\epsilon_1^2r_im_i)^2} - \frac{1}{(m_i+1)^2} \right)\ .
\end{aligned}
\end{equation*}
Observe that $\frac{\partial \Delta_p}{\partial r_i} < 0$ always holds, while $\frac{\partial \Delta_p}{\partial m_i} < 0$ when $\epsilon_1^2 < \frac{1}{r_i(2m_i+1)}$. Therefore, if $\epsilon_1$ is (sufficiently) small, then $\Delta_p$ is larger for smaller $r_i,m_i$ (i.e., select $i\in N$ with small $r_i,m_i$). As a surrogate, $\underline{f} = \min_i f_i = \min_i r_i m_i$ can capture this relationship and is thus the design choice in our algorithm (i.e., greedy selection).

\paragraph{Remark 2.} \textbf{Proof of independence of $\varphi_i, \varphi_j$.} Express $\varphi_i = f_i(\mathcal{X}_i)$ and $\varphi_j = f_j(\mathcal{X}_j)$, where $f_i, f_j: 2^\Pi \rightarrow \mathbb{R}$ are two deterministic functions that map (a set of) permutations to a real value (i.e., the respective SV estimate) and $\mathcal{X}_i$ is the random variable denoting the set of randomly sampled permutations for calculating $\varphi_i \gets f_i(\mathcal{X}_i)$ (elaborating lines 36-37).
Note that we assume all marginal contributions $\sigma_i(\pi), \sigma_j(\pi), \forall \pi\in \Pi$ are fixed (though unknown), given a well defined problem. It can be seen that $\varphi_i,\varphi_j$ are independent if $\mathcal{X}_i, \mathcal{X}_j$ are independent.

\subsection{Proof of \cref{corollary:k_axioms}}
\begin{proof}[Proof of Corollary \ref{corollary:k_axioms}]
We prove the corollary by considering two cases:

i) All $\varphi_i$'s are independent. From $\delta = 1- (1-1/(\epsilon_1^2\underline{f}))^n$, solve for $\underline{f}$,
\[ \underline{f} \geq  \frac{1}{\epsilon_1^2  (1 -  (1 - \delta)^{1/n})}\ ,\]
and since $r_i m_i \geq \underline{f}$
\[ m_i \geq \frac{1}{\epsilon_1^2  (1 - (1 - \delta)^{1/n})r_i} \ , \]
which implies the total budget
\begin{equation*}
\begin{aligned}
    m  
    &= \sum_{i\in N} m_i \\
    &\geq \sum_{i \in N} \frac{1}{\epsilon_1^2   (1 - (1 - \delta)^{1/n}) r_i } \\
    &= \epsilon_1^{-2}(1-(1-\delta)^{1/n})^{-1}\sum_{i \in N} \frac{1}{r_i}\\
    &\geq \epsilon_1^{-2}(1-(1-\delta)^{1/n})^{-1}\frac{n}{\max_{i\in N} r_i} \\
    &= \mathcal{O}(n\epsilon_1^{-2}(1-(1-\delta)^{1/n})^{-1})\ .
\end{aligned}
\end{equation*}

ii) Otherwise, $\varphi_i$'s may be dependent of each other. From $\delta = n /(\epsilon_1^2\underline{f})$, solve for $\underline{f}$,
\[\underline{f} \geq \frac{n}{\epsilon_1^2 \delta}\ , \]
and since $r_i m_i \geq \underline{f}$
\[ m_i \geq \frac{n}{\epsilon_1^2 \delta r_i} \ , \]
which implies the total budget
\begin{equation*}
\begin{aligned}
    m  
    &= \sum_i m_i \\
    &\geq \sum_i \frac{n}{\epsilon_1^2 \delta r_i}\\
    &= n\delta^{-1}\epsilon_1^{-2}\sum_{i\in N}\frac{1}{r_i} \\
    &\geq n\delta^{-1}\epsilon_1^{-2}\frac{n}{\max_{i \in N} r_i} \\
    &= \mathcal{O}(n^2\epsilon_1^{-2}\delta^{-1})\ .
\end{aligned}
\end{equation*}

\end{proof}

From the above derivations, it can be implied that any estimator which keeps track of $\underline{f}$ and stops soon long as the target $\underline{f}$ is achieved is able to obtain the upper bound budget complexity. As our proposed Greedy Active Estimator (GAE) iteratively updates $\underline{f}$, by terminating it when the current evaluation improves $\underline{f}$ to the desired value, GAE runs in the budget upper bound.

\subsection{Importance Sampling for Variance Reduction}
\label{app:var}

Following \citep{Castro2017,Kwon2022}, we restrict the sampling probability of a permutation $\pi$, $q_i(\pi)$, to depend on the cardinality of the predecessor set, $|P_i^{\pi}|$. 
Precisely, let $q_i'$ be a discrete distribution defined on the support $\{0,1,2,\ldots ,n-1\}$ that maps the cardinality $c$ of the predecessor set to the probability that any permutation of cardinality $c$ is drawn.
The sampling probability $q_i(\pi) = q_i'(|P_i^{\pi}|)/(n-1)!$ for a specific permutation $\pi$ because there are other $(n-1)!$ equally likely permutations $\pi'$ where $|P^\pi_i|= |P^{\pi'}_i|$ (i.e., the position of $i$ in permutations $\pi$ and $\pi'$ is the same).

We minimise the variance of the importance sampling estimator to derive the optimal importance sampling distributions $q_i'^*(c)$ (over the support $\{0,1,\ldots,n-1\}$) and $q_i^*(\pi)$ (over the support $\Pi$). Let $\tilde{\mathbb{V}}_{q_i}[\sigma_i(\pi)] \coloneqq \mathbb{V}_{\pi \sim q_i}[\sigma_i(\pi)/(q_i(\pi)n!)]$ denotes the \emph{probability adjusted} variance of $\sigma_i(\pi)$ \cite{importance_sampling}, using the sampling distribution $q_i$ supported on $\Pi$;\footnote{$q_i(\pi)$ is the probability of $\pi$ w.r.t.~$q_i$ (e.g., If $q_i = U$, then $q_i(\pi) = 1/n!$).} The variance of marginal contribution $\sigma_i(\pi)$ of $i$ w.r.t.~a permutation $\pi$ sampled from a proposal distribution $q_i$ is
\begin{equation}
\begin{aligned}
        \tilde{\mathbb{V}}_{q_i}[\sigma_i(\pi)]
        &= \sum_{\pi \in \Pi} \left(\frac{(\sigma_i(\pi)/n! - \phi_i q_i(\pi))^2}{q_i(\pi)}\right) \\
        &= \sum_{\pi \in \Pi} \frac{1}{n!} \left(\frac{(\sigma_i(\pi)/n! - \phi_i q_i(\pi))^2}{q_i(\pi)}\times n!\right).
\end{aligned}
\label{eq:importance_var}
\end{equation}
Using the Lagrange multiplier and differentiating Equ.~\eqref{eq:importance_optimal_var} to minimise the above variance, we will show that the following optimal distributions can be obtained:
\begin{equation}
    q_i'^*(c) =  \sqrt{\mathbb{E}_{\pi \sim U_c}[\sigma_i(\pi)^2]} \Big/ \left(\sum_{k=0}^{n-1}\sqrt{\mathbb{E}_{\pi \sim U_k}[\sigma_i(\pi)^2]} \right) \ ,
\label{eq:optimal_q_card}
\end{equation}
and therefore, by substituting $q_i(\pi) = q_i'(|P_i^{\pi}|)/(n-1)!$, we get the optimal sampling distribution of permutation $q_i^*(\pi)$
\begin{equation}
    q^*_i(\pi) = \frac{q'^*_i(|P^{\pi}_i|)}{(n - 1)!}\ .
\label{eq:optimal_q_perm}
\end{equation}
The optimality of $q_i'^*(c)$ and hence $q_i^*(\pi)$ can be proved as follows. Denote $\mathbb{E}_{U_c} = \mathbb{E}_{\pi \sim U_c}$ where $U_c$ is defined as in Appendix~\ref{app:estimate_params} to represent uniform distribution of permutations with a fixed cardinality. Further denote $\sigma = \sigma_i(\pi)$. Then, the variance can be rewritten as the average of $\mathbb{E}_{U_c}$ over all cardinalities $c \in \{0,1,2,. . .,n-1\}$
\begin{equation}
\begin{aligned}
    \tilde{\mathbb{V}}_{q_i}[\sigma_i(\pi)]
    &= \sum_{\pi \in \Pi} \frac{1}{n!} \left(\frac{(\sigma_i(\pi)/n! - \phi_i q_i(\pi))^2}{q_i(\pi)}\times n!\right) \\
    &= \sum_{c=0}^{n-1} \sum_{\pi \in \{\pi: |P_i^\pi|=c\}} \frac{1}{n!} \left(\frac{(\sigma_i(\pi)/n! - \phi_i q_i(\pi))^2}{q_i(\pi)}\times n!\right) \\
    &= \sum_{c=0}^{n-1} \sum_{\pi \in \{\pi: |P_i^\pi|=c\}} \frac{1}{n!} \left(\frac{(\sigma_i(\pi)/n! - \phi_i q_i'(c)/(n-1)!)^2}{q_i'(c)/(n-1)!}\times n!\right) \\
    &= \sum_{c=0}^{n-1} \sum_{\pi \in \{\pi: |P_i^\pi|=c\}} \frac{1}{n!} \left(\frac{(\sigma_i(\pi) - n\phi_i q_i'(c))^2}{n!nq_i'(c)}\times n!\right) \\
    &= \frac{1}{n} \sum_{c=0}^{n-1} \sum_{\pi \in \{\pi: |P_i^\pi|=c\}} \frac{1}{(n-1)!} \left(\frac{(\sigma_i(\pi) - n\phi_i q_i'(c))^2}{n!nq_i'(c)}\times n!\right) \\
    & = \frac{1}{n} \sum_{c=0}^{n-1} \mathbb{E}_{U_c}\left[\frac{(\sigma - n\phi_i q'^*_i(c))^2}{n!nq'_i(c)} \times n! \right] \\
    & = \frac{1}{n} \sum_{c=0}^{n-1} n! \times \left( \frac{\mathbb{E}_{U_c}[\sigma^2]}{n!nq'_i(c)} - \frac{2\mathbb{E}_{U_c}[\sigma]\phi_i}{n!} + \frac{n\phi_i^2q'_i(c)}{n!} \right) \\
    &= \frac{1}{n} \sum_{c=0}^{n-1} \left( \frac{\mathbb{E}_{U_c}[\sigma^2]}{nq'_i(c)} - 2\mathbb{E}_{U_c}[\sigma]\phi_i + n\phi_i^2q'_i(c) \right) \\
    &= \frac{1}{n} \sum_{c=0}^{n-1} \left( \frac{\mathbb{E}_{U_c}[\sigma^2]}{nq'_i(c)} - 2\mathbb{E}_{U_c}[\sigma]\phi_i \right) + \phi_i^2\ .
\end{aligned}
\label{eq:var_card}
\end{equation}
For a fixed $c$, minimise the summands in Equ.~\eqref{eq:var_card} over the choice of $q_i'(c)$ s.t.~$\sum_{c=0}^{n-1}q_i'(c) = 1$, which can be reformulated as the following Lagrangian after discarding the constant term $n\phi_i^2$,
\begin{align*}
    \begin{split}
            \mathcal{G} &= \frac{1}{n} \sum_{c=0}^{n-1} \left( \frac{\mathbb{E}_{U_c}[\sigma^2]}{nq'_i(c)} - 2\mathbb{E}_{U_c}[\sigma]\phi_i
            \right) - \lambda \left(1- \sum_{c=0}^{n-1}q_i'(c) \right).
    \end{split}
\end{align*}
Set its partial derivative w.r.t.~$q_i'(c)$ to $0$,
\begin{align*}
    \begin{split}
    \frac{\partial \mathcal{G}}{\partial q_i'(c)} &= (
    - \frac{\mathbb{E}_{U_c}[\sigma^2]}{n^2q_i'(c)^2}) + \lambda = 0 \\
    q_i'^*(c) &= \frac{\sqrt{\mathbb{E}_{U_c}[\sigma^2]}}{\sqrt{
    n^2\lambda}} \propto \sqrt{\mathbb{E}_{U_c}[\sigma^2]}\ .
    \end{split}
\end{align*}
The direct proportionality uses the fact that $q_i'^*(c)$ is independent of $n$ and $\lambda$.
To satisfy $\sum_{c=0}^{n-1}q_i'^*(c) = 1$, we standardise all $q_i'^*(c)$ and obtain Equ.~\eqref{eq:optimal_q_card}, which completes the proof.

Now, we prove that importance sampling by cardinality strictly outperforms MC. Denote $S = \sum_{k=0}^{n-1} \sqrt{\mathbb{E}_k[\sigma^2]}$. Plug Equ.~\eqref{eq:optimal_q_card} back into Equ.~\eqref{eq:var_card},
\begin{align*}
    \begin{split}
    \tilde{\mathbb{V}}_{q_i^*}[\sigma_i(\pi)] 
    & = \frac{1}{n} \sum_{c=0}^{n-1} \left( \frac{\sqrt{\mathbb{E}_{U_c}[\sigma^2]}S}{n} - 2\mathbb{E}_{U_c}[\sigma]\phi_i + \frac{n\phi_i^2\sqrt{\mathbb{E}_{U_c}[\sigma^2]}}{S} \right).
    \end{split}
\end{align*}
For MC, we take advantage of the linearity of expectation and evaluate its variance by breaking variance into strata of different cardinalities. Substitute $q_i(\pi)$ with $1/n!$,
\begin{align*}
    \begin{split}
    \tilde{\mathbb{V}}_{U}[\sigma_i(\pi)]
    & = \frac{1}{n} \sum_{c=0}^{n-1} \mathbb{E}_{U_c}\left[\frac{(\sigma/n! - \phi_i/n!)^2}{(1/n!)^2}\right] \\
    & = \frac{1}{n} \sum_{c=0}^{n-1} \mathbb{E}_{U_c}\left[(\sigma - \phi_i)^2\right] \\
    & = \frac{1}{n} \sum_{c=0}^{n-1} \left( \mathbb{E}_{U_c}[\sigma^2] - 2\phi_i\mathbb{E}_{U_c}[\sigma] + \phi_i^2 \right).
    \end{split}
\end{align*}
The two variances are compared via a subtraction,
\begin{equation*}
    \begin{aligned}
        \tilde{\mathbb{V}}_{q_i^*}[\sigma_i(\pi)]  - \tilde{\mathbb{V}}_{U}[\sigma_i(\pi)]
        &= \frac{\phi_i^2}{n}\sum_{c=0}^{n-1} \left(\frac{n\sqrt{\mathbb{E}_{U_c}[\sigma^2]}}{S} - 1\right) + \sum_{c=0}^{n-1}\frac{1}{n}\left(\frac{\sqrt{\mathbb{E}_{U_c}[\sigma^2]}S}{n} - \mathbb{E}_{U_c}[\sigma^2]\right) \\
        &= 0 + \sum_{c=0}^{n-1}\frac{1}{n}\left(\frac{\sqrt{\mathbb{E}_{U_c}[\sigma^2]}S}{n} - \mathbb{E}_{U_c}[\sigma^2]\right) \\
        &= \frac{1}{n} \left(\frac{S^2}{n} - \sum_{c=0}^{n-1}\mathbb{E}_{U_c}[\sigma^2]\right).
    \end{aligned}
\end{equation*}
Next, note that $S^2 = (\sum_{k=0}^{n-1}\sqrt{\mathbb{E}_k[\sigma^2]})^2 \leq n\sum_{k=0}^{n-1}\mathbb{E}_k[\sigma^2]$ by the Cauchy Schwarz inequality. Then,
\begin{equation}
    \begin{aligned}
        \tilde{\mathbb{V}}_{q_i^*}[\sigma_i(\pi)]  - \tilde{\mathbb{V}}_{U}[\sigma_i(\pi)]
        &= \frac{1}{n} \left(\frac{S^2}{n} - \sum_{c=0}^{n-1}\mathbb{E}_{U_c}[\sigma^2]\right) \leq 0 \\
        \implies  \tilde{\mathbb{V}}_{q_i^*}[\sigma_i(\pi)]  &\leq \tilde{\mathbb{V}}_{U}[\sigma_i(\pi)]
    \end{aligned}
    \label{eq:importance_optimal_var}
\end{equation}
where equality holds if and only if the equality condition for the Cauchy Schwarz inequality holds, i.e., $\mathbb{E}_{U_0}[\sigma^2] = \mathbb{E}_{U_1}[\sigma^2] = ... = \mathbb{E}_{U_{n-1}}[\sigma^2]$.

\subsection{Proof of \cref{proposition:optimality_of_greedy}}

Before proving \cref{proposition:optimality_of_greedy}, we first provide a formal version of the proposition.
\begin{proposition*}[{\emph{formal} \bf Budget Optimality of Greedy Selection}]
For a fixed budget $m$, denote $\underline{f}_{A}$ as the minimum FS obtained by estimation algorithm $A$. Let $\mathcal{A}_q$ be the set of all sampling-based estimation methods (defined in Sec.~\ref{sec:preliminaries}) that sample permutations from a fixed distribution $q$. Then, greedy selection (this is, iteratively selecting $i = \argmin_{j\in N} f_j$) with the same underlying distribution $q$ achieves the optimal minimum FS, $\underline{f}_{\text{GS}, q}$. In other words, $\underline{f}_{\text{GS}, q} = \max_{A \in \mathcal{A}_q} \underline{f}_{A}$.
\end{proposition*}
In short, \cref{proposition:optimality_of_greedy} suggests that if we break up a sampling-based estimation method into two parts: 1) choosing a $\varphi_i$ to evaluate, and 2) sampling a permutation for it, then greedy selection chooses the $\varphi_i$'s in such a way that the resulting minimum fidelity score is the highest, given a fixed underlying sampling distribution.

\begin{proof}[Proof of \cref{proposition:optimality_of_greedy}]
First recall the \emph{invariability} of $\varphi_i$
\begin{equation*} 
    r_{i} \coloneqq \frac{(|\phi_i| + \xi)^2}{\mathbb{V}[\sigma_i(\pi)]} = \frac{f_i}{m_i}
\end{equation*}
where $m_i$ is the budget $i$ receives. Since the sampling distribution $q$ is fixed, $r_i$ is fixed and does not vary with $m_i$. 
Suppose the optimal $\underline{f}$ is $f^*$. To achieve $f^*$, every $i \in N$ with an initial FS (after bootstrapping) $f_i^{(0)} < f^*$ must receive at least $e_i = \lceil (f^* - f_i^{(0)}) / r_i \rceil$ budget to ensure that its resulting FS reaches/exceeds $f^*$. Now, suppose greedy selection has made $e_j$ evaluations for some $j \in N$. If every $i\in N\setminus \{j\}$ has FS greater than $f^*$, then $f^*$ is achieved and the algorithm terminates. Otherwise it evaluates some $i$ with FS smaller than $f^*$. As such, each $i \in N$ receives exactly $e_i$ evaluations, the minimum budget needed to achieve $f^*$. Equivalently, with the same total budget, greedy selection achieves the highest $\underline{f}$ among all sampling-based algorithms with the underlying distribution $q$.
\end{proof}

\subsection{Proof of \cref{prop:active_greedy_mc}}
\label{app:proof_greedy_slection}

\begin{proof}[Proof of \cref{prop:active_greedy_mc}]

We first derive $\underline{f}_{\text{greedy}}$ and $\underline{f}_{\text{MC}}$. For MC, as each $i \in N$ receives $m / n$ evaluations,
\begin{equation} 
    \underline{f}_{\text{MC}} = \min f_i(\varphi_i) = \min_{i \in N} \frac{m}{n} r_i
    = \frac{m\min_{i \in N}r_i}{n}\ .
    \label{eq:mc_k}
\end{equation}

For greedy selection, to simplify analysis, we made a mild assumption that each $r_i$ is relatively small as compared to the number of samples $m$ so that with a careful allocation of $m_i$ samples to $i$, it is possible to achieve $r_i m_i = r_j m_j$ for all $i,j \in N$. Then,
\begin{equation}
    \begin{aligned}
     \forall i, j \in N, r_i m_i & = r_j m_j \\
    \sum_{i \in N} m_i & = m \\
    \Longrightarrow m_i & = \frac{m}{r_i \sum_{j \in N} 1/r_j} \\
    \Longrightarrow \underline{f}_{\text{greedy}} & = m_i r_i = \frac{m}{\sum_{j \in N} 1/r_j}\ .
    \end{aligned}
    \label{eq:greedy_k}
\end{equation}

With Equ.~\eqref{eq:mc_k} and Equ.~\eqref{eq:greedy_k}, $\underline{f}_{\text{GAE}}$ can be related to $\underline{f}_{\text{greedy}}$. First, from $\tilde{\mathbb{V}}_{q_i^*}[\sigma_i(\pi)]  \leq \tilde{\mathbb{V}}_{U}[\sigma_i(\pi)]$ in Equ.~\eqref{eq:importance_optimal_var}, so for any $i \in N$,
\begin{equation*}
\begin{aligned}
        \frac{1}{r_{i,U}} - \frac{1}{r_{i, q_i^*}} &= \frac{\tilde{\mathbb{V}}_{U}[\sigma_i(\pi)]}{(|\phi_i| + \xi)^2} - \frac{\tilde{\mathbb{V}}_{q_i^*}[\sigma_i(\pi)]}{(|\phi_i| + \xi)^2} \\
    &= \frac{\tilde{\mathbb{V}}_{U}[\sigma_i(\pi)] - \tilde{\mathbb{V}}_{q_i^*}[\sigma_i(\pi)]}{(|\phi_i| + \xi)^2} \geq 0\ , 
\end{aligned}
\end{equation*}
which implies $\forall i \in N, 1/r_{i,U} \geq 1/r_{i, q_i^*}$ and hence $\sum_{j \in N} 1/r_{j,U} \geq \sum_{j \in N} 1/r_{j,q_j^*}$. Then,
\begin{equation*}
    \begin{aligned}
    \frac{\underline{f}_{\text{GAE}}}{\underline{f}_{\text{greedy}}} &= \frac{\frac{m}{\sum_{j \in N} 1/r_{j,q_j^*}}}{\frac{m}{\sum_{j \in N} 1/r_{j,U}}} \\
    &= \frac{\sum_{j \in N} 1/r_{j,U}}{\sum_{j \in N} 1/r_{j,q_j^*}} \\
    & \geq 1
    \end{aligned}
\end{equation*}
where equality is taken if and only if $\forall i \in N,\  r_{i,U} = r_{i, q_i^*}$, which completes the first part of the proof. 

For the second part of the proof, with Equ.~\eqref{eq:mc_k} and Equ.~\eqref{eq:greedy_k}, similarly divde  $\underline{f}_{\text{greedy}}$ by $\underline{f}_{\text{MC}}$,
\begin{equation}
\begin{aligned}
    \frac{\underline{f}_{\text{greedy}}}{\underline{f}_{\text{MC}}} 
    & = \frac{\frac{m}{\sum_{j \in N} 1/r_j}}{\frac{m\min_{i \in N}r_i}{n}} \\
    & = \frac{n}{\min_{i \in N}r_i \sum_{j \in N} 1/ r_j} \\
    & = \frac{n}{\sum_{j \in N}\frac{\min_{i \in N}r_i}{r_j}} \\
    & \geq \frac{n}{\sum_{j \in N}\frac{r_j}{r_j}} \\
    & = \frac{n}{n} = 1
\end{aligned}
\label{eq:greedy_k_division}
\end{equation}
where the inequality becomes equality if and only if $\forall i,j \in N,\  r_i = r_j$.

\end{proof}

\subsection{Proof of \cref{proposition:pe_pdp}}

\begin{proposition} \label{proposition:pe_pdp}
Given fixed $r_i$ for each $i$, greedy selection satisfies PDP.
\begin{proof}[Proof of \cref{proposition:pe_pdp}]

We show greedy selection satisfies PDP. To see this, let an alternative estimation process produce final fidelity scores $\boldsymbol{f}'$ satisfying $k^*$ such that $\forall k \in N \setminus \{i,j\}, f_k' = f_k$ and $f_i' + f_j' = f_i + f_j$. Without loss of generality, let $f_i < f_j$. Further let the initial fidelity score of each $i \in N$ be $f_i^{(0)}$ and consider two cases: 1) $f_i^{(0)} = f_i$. In this case, greedy selection makes no evaluations on $i$. Hence $f_i' \geq f_i$ since making evaluations on $i$ can only improve its fidelity score. Thus $|f_i - f_j| = f_j - f_i = f_i' + f_j' - 2f_i \leq f_i' + f_j' - 2f_i' = f_i' - f_j' \leq |f_i' - f_j'|$; 2) $f_i^{(0)} < f_i$. In this case greedy selection makes evaluations on $i$. However, notice that greedy selection stops improving $f_i$ once $f_i \geq k^*$. Hence, for any $f_i'$ that satisfies $k^*$, it must be the case that $f_i' \geq f_i$. As such, reuse the argument for case 1, we have $|f_i - f_j| < |f_i' - f_j'|$. Therefore, under no circumstance is an alternative estimation process preferred over greedy selection, namely, greedy selection satisfies PDP.
\end{proof}
\end{proposition}

To give an intuitive illustration of PDP, consider a dataset (of $3$ training examples) evaluated using two estimators $p_1$ and $p_2$ that lead to FSs of $(25,5,5)$ and $(20,10,5)$ respectively. Though they both have $\underline{f}=5$ (thus the same fairness guarantee by \cref{corollary:k_axioms}), $p_1$ is less fair (than $p_2$) in its allocation of estimation budget as $p_1$ overly focuses on the $1$st training example (instead of the other two).
Note PDP prefers $(20,10,5)$ to $(25,5,5)$, as desired.
Because PDP is w.r.t.~the estimation method and not individual $\varphi_i$, it can be considered a characteristic of the estimation and thus not included in the fairness guarantee (\cref{definition:fairness-guarantee}).

\section{Additional Experimental Results} \label{app:experiments}

\subsection{Dataset License}

Covertype \citep{misc_covertype_31}: Apache License 2.0;
Breast cancer \citep{breast_cancer}: Apache License 2.0;
Iris \citep{iris}: Apache License 2.0;
Adult income \citep{adult_income}: Apache License 2.0;
Wine \citep{wine}: Apache License 2.0;
Diabetes \citep{Efron_2004}: : Attribution-NonCommercial 4.0 International (CC BY-NC 4.0);
MNIST \citep{deng2012mnist}: Creative Commons Attribution-Share Alike 3.0;
CIFAR-10 \citep{cnn_cifar10}: The MIT License (MIT);
Movie reviews \citep{pang2005seeing_mr}: BSD 3-Clause "New" or "Revised" License;
Stanford Sentiment Treebank \citep{kim2014convolutional_SST5}: BSD 3-Clause "New" or "Revised" License;
Credit Card~\cite{credit_card_dataset}: Database Contents License (DbCL); 
Uber \& Lyft~\cite{uber_lyft_dataset}: CC0 1.0 Universal (CC0 1.0); 
Used Car~\cite{used_car_dataset}: CC0 1.0 Universal (CC0 1.0);
Hotel Reviews~\cite{hotel_reviews_dataset}: Attribution-NonCommercial 4.0 International (CC BY-NC 4.0).

\subsection{Computational Resources}

For all our experiments requiring only CPUs, we use a server with 2 $\times$ Intel Xeon Silver 4116 ($2.1$Ghz) as the computing resource. 
For experiments requiring GPUs (i.e., in \textbf{P3.} we train one model for one agent for multiple agents), we utilise a server with Intel(R) Xeon(R) Gold 6226R CPU @$2.90$GHz and four NVIDIA GeForce RTX 3080's.

\subsection{Additional Experimental Settings}

\paragraph{Utility of the null set.} As there is no standard way of defining the utility of the null set (i.e., $v(\emptyset)$), in all our experiments using test accuracy as the utility function, we initialise the ML learner by randomly picking a coalition $C_0$ comprising of two entries, one from each class \textbf{outside} the training set, and set $v(\emptyset)$ as the utility value of $C_0$. For each coalition $C$, we define $v(C)$ as the utility value of $C \cup C_0$. This way, we make sure that at least one entry from each class is present to train the classifier. We do the same for all experiments using negative MSE as the utility function, but picking only one entry outside the training set as we only require one entry to train a regressor. 
For experiments on other specific scenarios, $\nu(\emptyset)$ follows the respective references.

\paragraph{Approximation of $f_i$.} In \cref{def:afs}, $f_i$ depends on $\phi_i$ and $\mathbb{V}[\varphi_i]$ which are intractable to obtain in practice. Therefore, suppose that we have made $m$ evaluations, $f_i$ is approximated by first estimating $\phi_i \approx \varphi_i = 1/m \sum_{t=1}^{m} \sigma_i(\pi_t)$ and $\mathbb{V}_{\pi \sim q}[\varphi_i] \approx s^2 / m = 1/[m(m-1)] \sum_{t=1}^{m} (w_t\sigma_i(\pi_t) - \varphi_i)^2$, where $w_t = 1 / (q(\pi_t)n!)$ (for sampling over cardinalities, we use $w_t = 1 / (q'(|P_i^\pi|)n)$). Then $f_i \approx m \times \varphi_i / s^2$.

\paragraph{Additional details on other SV estimation baselines}
We compare against Sobol sequences out of the four proposed methods in \citep{Mitchell2021} because it is the most computationally efficient and is reported to have the best/competitive performance.

\paragraph{Additional details for \textbf{P2.}.}
We partition a dataset (e.g., used-car price prediction) to $n=8$ agents/data providers, so that each data provider exclusively owns data (price information) of cars from a particular manufacturer (e.g., Audi, Ford, etc). For the two datasets considered for \textbf{P2.}, we set $n=8$ or $n=10$, which is larger than $n=3$ in the original work \citep{xu2021vol}.

While \textbf{P2.} provides a dataset valuation function based on the theoretical framework of linear regression, our experiments (following \citep{xu2021vol}) explore beyond linear regression and consider applications of DNNs
For instance, we perform additional feature pre-processing such as the GloVe word embeddings \citep{pennington2014glove} to create a $8$-dimensional pre-processed feature from a bidirectional long short-term memory model for hotel reviews dataset.
Subsequently, the linear regression/RVSV is w.r.t.~the pre-processed $8$-dimensional feature.

\paragraph{Additional details for \textbf{P3.}.}

In the CML experiments, the uber lyft ride price (hotel review sentiment) prediction dataset presents a $12$ ($8$)-dimensional regression task.
Therefore, the experiment setting follows \citep{Hwee2020} to calculate the \emph{mutual information} of the $12$ or $8$ parameters of a Bayesian linear regression using a Gaussian process (with a radial basis kernel whose lengthscale parameter is automatically learnt).
In the FL experiments, the dataset partition, used ML models and training hyperparameters such as batch size, learning rate etc, follow \citep{xu2021gradient}.\footnote{\url{https://github.com/XinyiYS/Gradient-Driven-Rewards-to-Guarantee-Fairness-in-Collaborative-Machine-Learning}.}

\subsection{Additional Results for $f_i$ Against APE} \label{app:ri_mape}

Under the same experiment setting as \cref{fig:ri_mape}, additional experimental results for $f_i$ Against APE using different dataset and ML algorithm combinations are in \cref{fig:ri_ape_app}.

We randomly select a subset of $50$ training examples from the Covertype dataset to train using a logistic regression classifier and set test accuracy as the utility function $v$ \citep{pmlr-v97-ghorbani19c}. 
Since the exact SV (for comparison in APE) is intractable, we approximate it (as the ground truth in APE) using the MC estimate (denoted by $\bar{\phi}$), i.e., an average of $5000$ marginal contributions w.r.t.~uniformly randomly selected permutations. As $f_i$ is also intractable, we approximate it using $f_i \approx 50 \times \Bar{\phi} / s^2$ where $s^2 = \left( (\sum_{t=1}^{50} \sigma_i(\pi_t)^2) - (\sum_{t=1}^{50} \sigma_i(\pi_t))^2 / 50 \right) / 49$. We evaluate each training example with $50$ permutations drawn from simple random sampling and plot $f_i$ against $\sqrt{1/\text{APE}}$ in \cref{fig:ri_mape} which shows a higher $f_i$ corresponds to a lower APE, as expected. 
Note the `ground truth' is approximated with MC estimates with $100$ times the budget than in the actual estimation \citep{jia2019towards,xu2021gradient}.

\begin{figure}[!ht]
    \centering
    \begin{subfigure}[t]{0.23\textwidth} \centering
    \includegraphics[width=\textwidth]{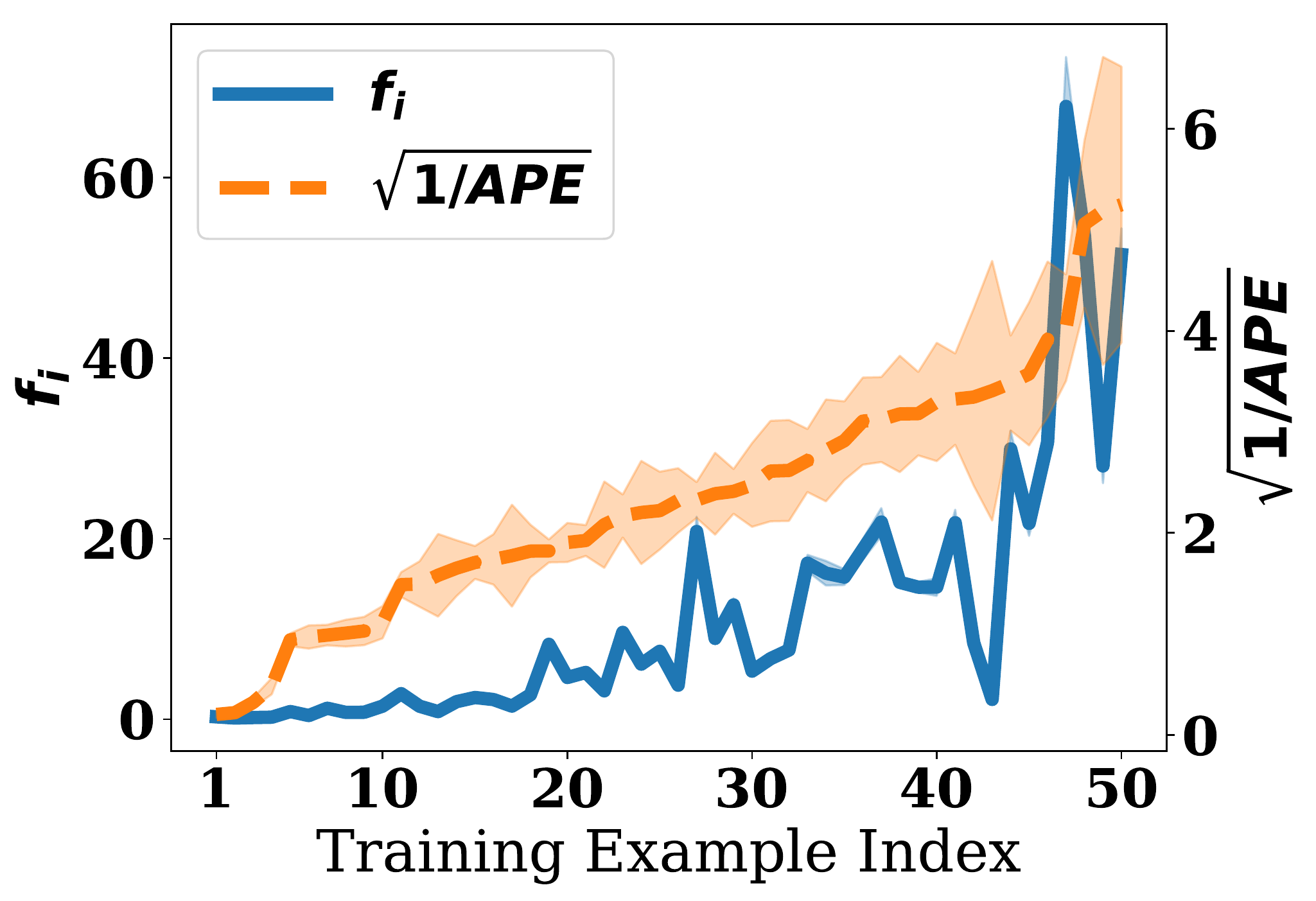}
    \caption{Covertype Logistic: $(0.701 \pm 0.019)$}
    \end{subfigure}
\hfill
    \begin{subfigure}[t]{0.23\textwidth} \centering
    \includegraphics[width=\textwidth]{figures/ri_mape_breast_cancer_logistic.pdf}
    \caption{Beast Cancer Logistic: $(0.678 \pm 0.015)$}
    \end{subfigure}
\hfill
    \begin{subfigure}[t]{0.23\textwidth} \centering
    \includegraphics[width=\textwidth]{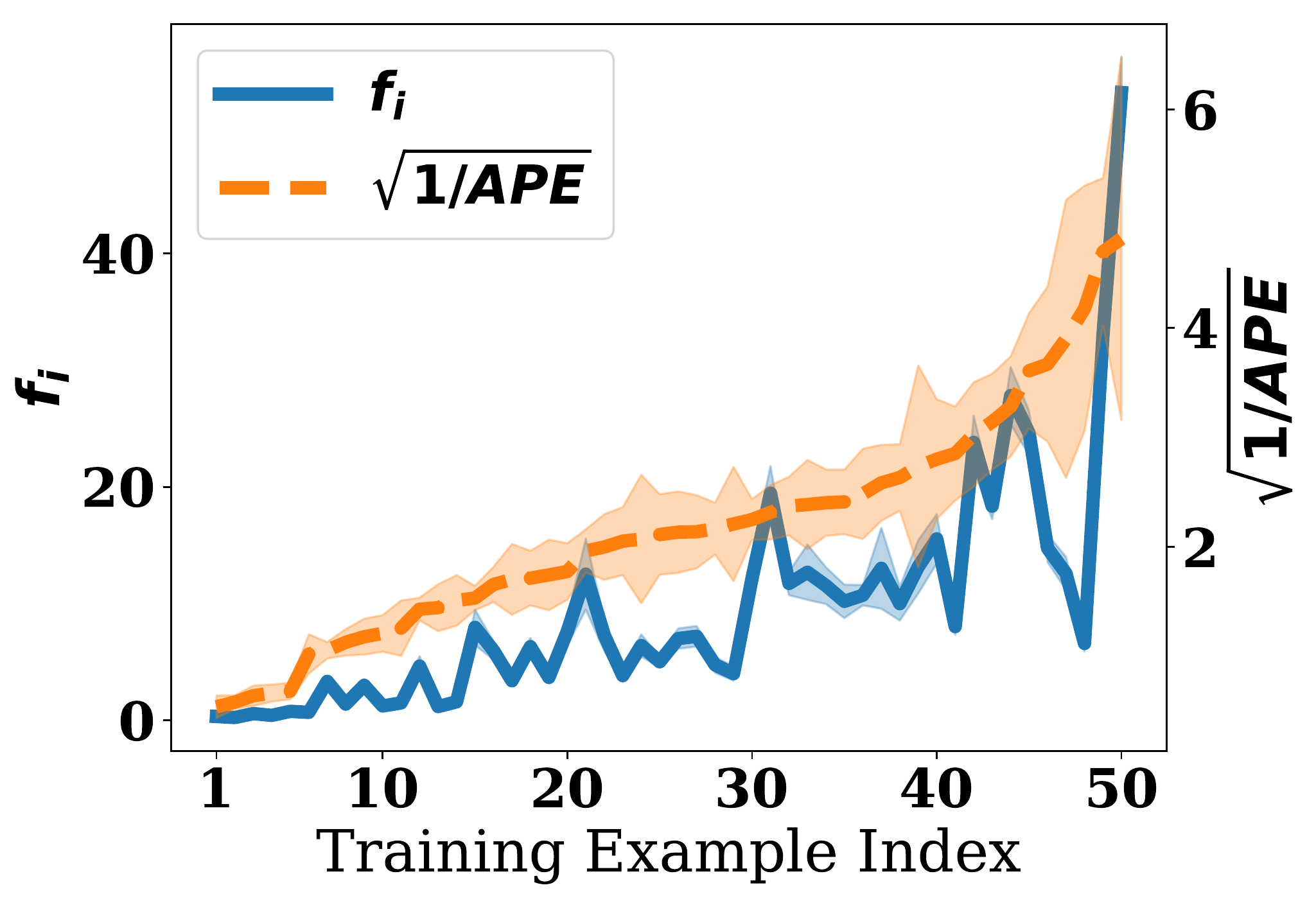}
    \caption{MNIST Logistic: $(0.580 \pm 0.020)$} 
    \end{subfigure}
\hfill
    \begin{subfigure}[t]{0.23\textwidth} \centering
    \includegraphics[width=\textwidth]{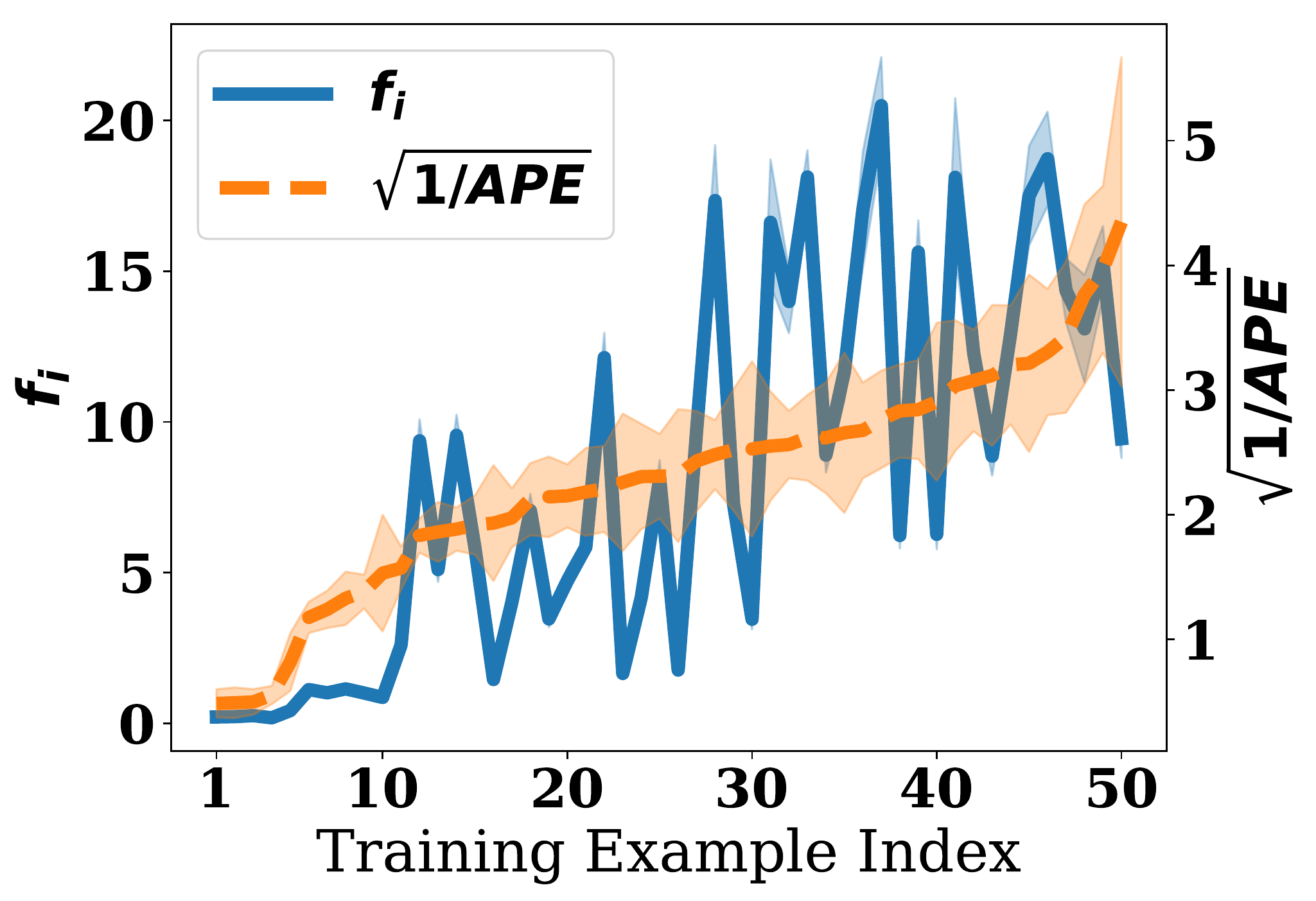}
    \caption{Synthetic Gaussian Logistic: $(0.536 \pm 0.029)$}
    \end{subfigure}
\hfill
    \begin{subfigure}[t]{0.23\textwidth} \centering
    \includegraphics[width=\textwidth]{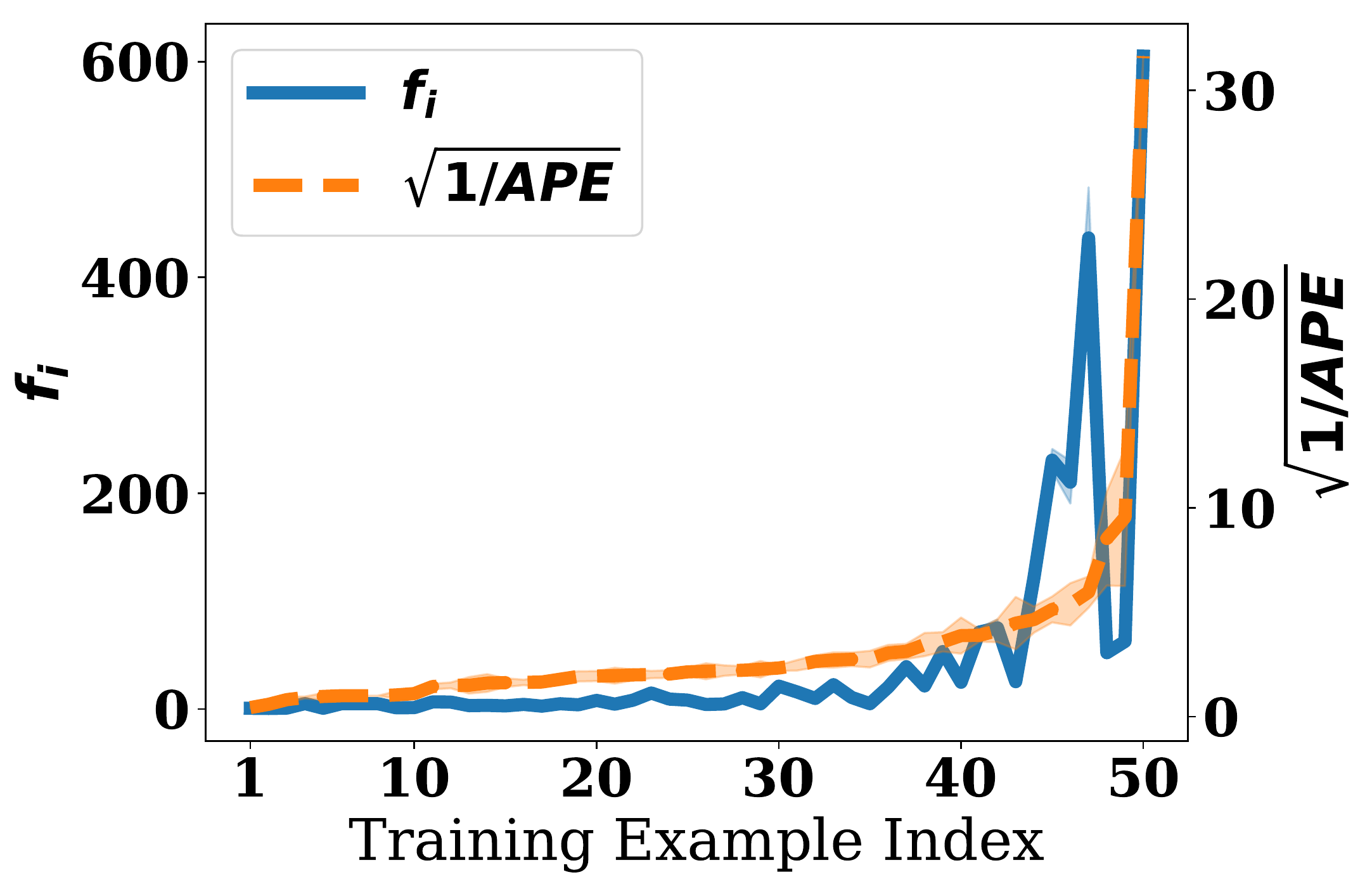}
    \caption{Covertype $k$NN: $(0.713 \pm 0.015)$}
    \end{subfigure}
\hfill
    \begin{subfigure}[t]{0.23\textwidth} \centering
    \includegraphics[width=\textwidth]{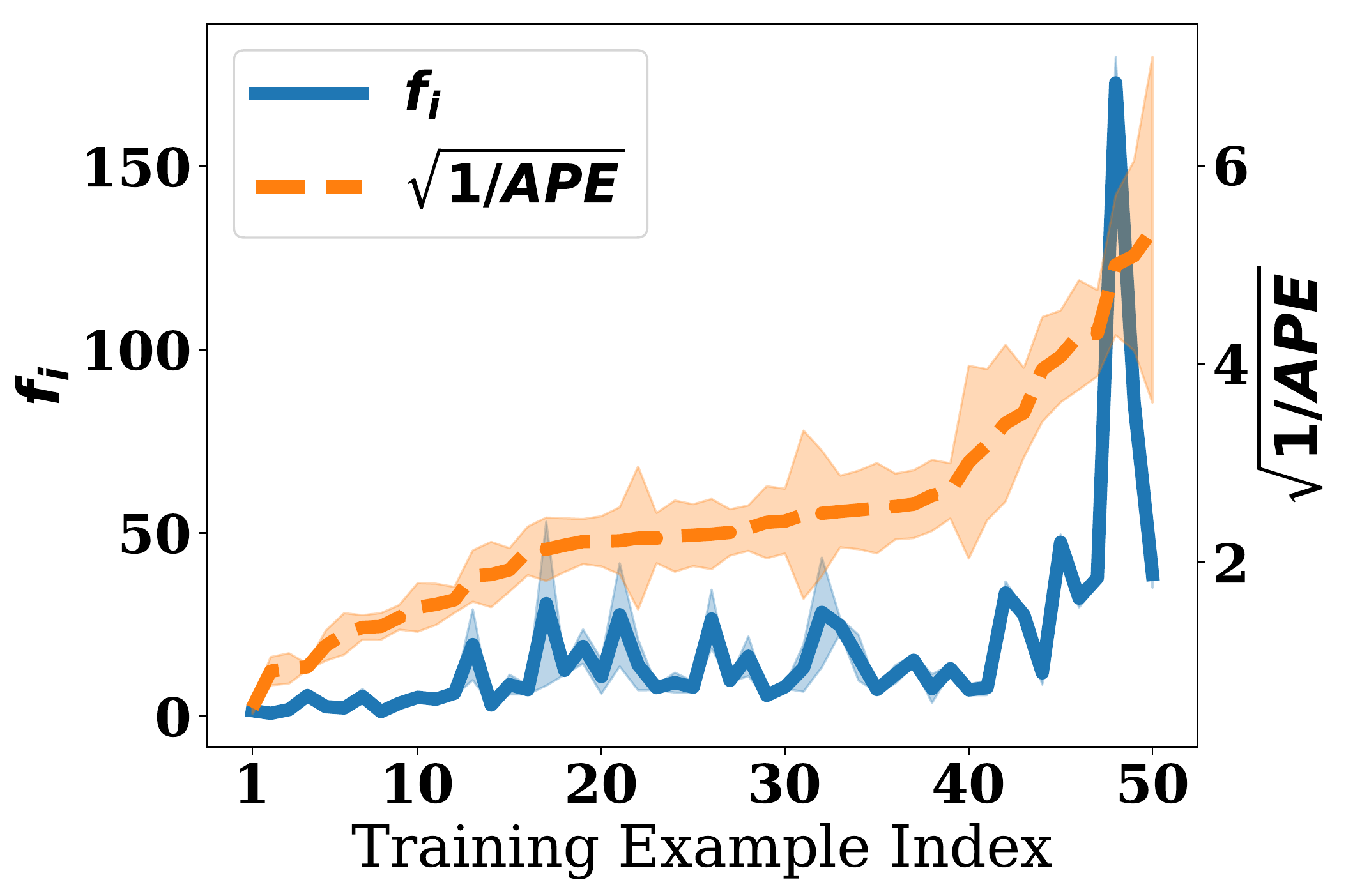}
    \caption{Breast Cancer $k$NN: $(0.409 \pm 0.024)$}
    \end{subfigure}
\hfill
    \begin{subfigure}[t]{0.23\textwidth} \centering
    \includegraphics[width=\textwidth]{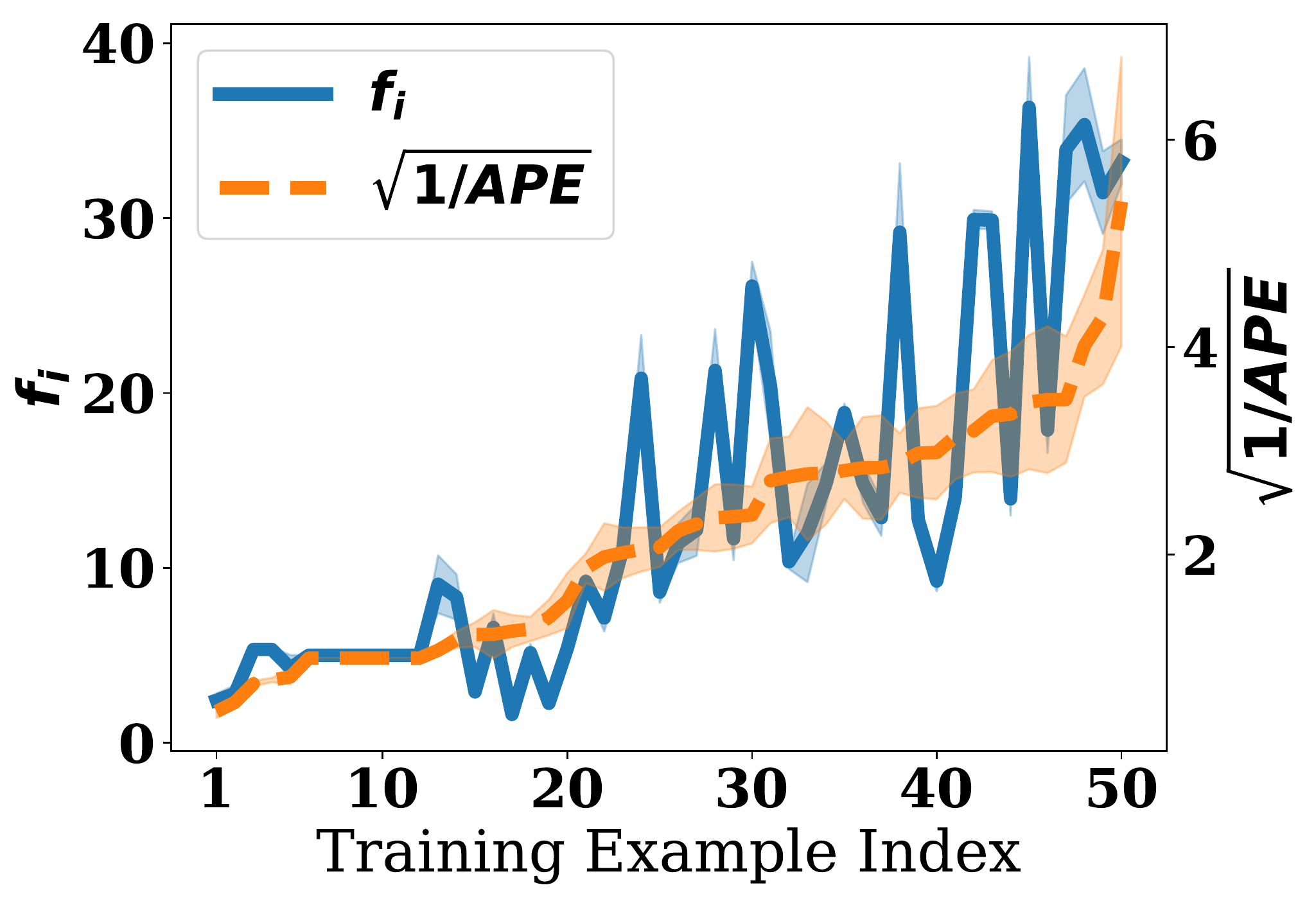}
    \caption{MNIST $k$NN: $(0.689 \pm 0.012)$}
    \end{subfigure}
\hfill
    \begin{subfigure}[t]{0.23\textwidth} \centering
    \includegraphics[width=\textwidth]{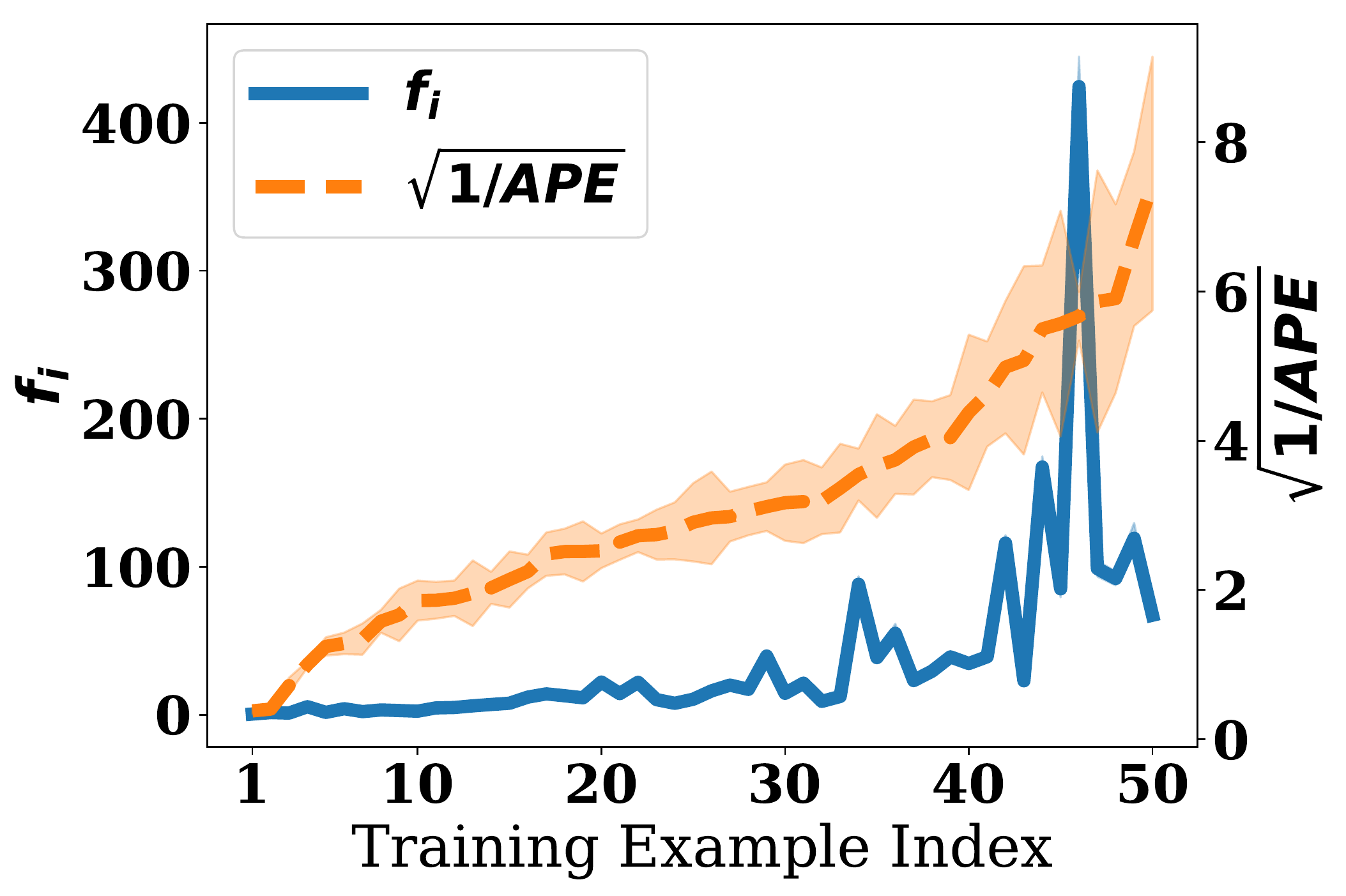}
    \caption{Synthetic Gaussian $k$NN: $(0.681 \pm 0.013)$}
    \end{subfigure}
\hfill
    \begin{subfigure}[t]{0.23\textwidth} \centering
    \includegraphics[width=\textwidth]{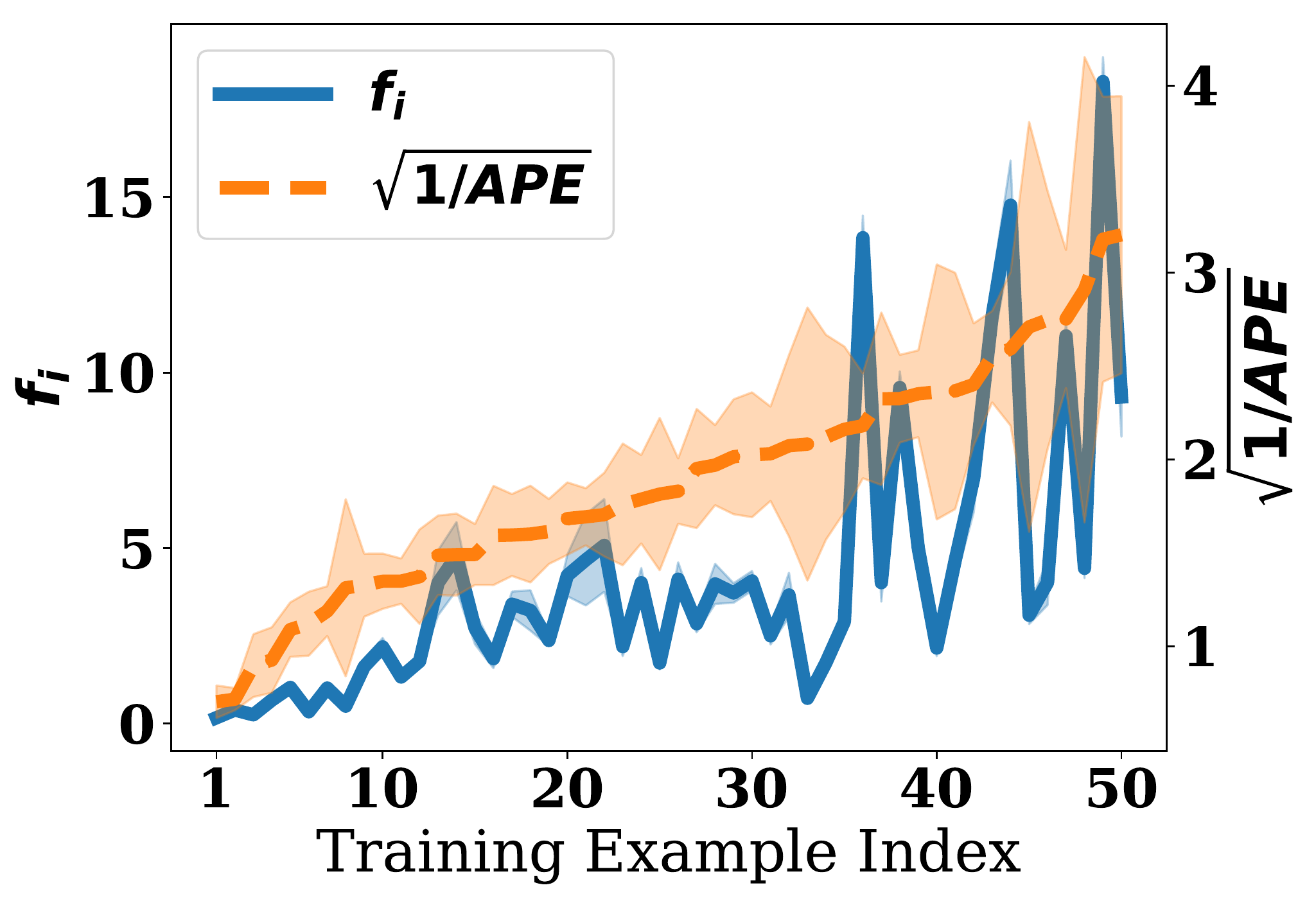}
    \caption{Covertype SVC: $(0.529 \pm 0.031)$}
    \end{subfigure}
\hfill
    \begin{subfigure}[t]{0.23\textwidth} \centering
    \includegraphics[width=\textwidth]{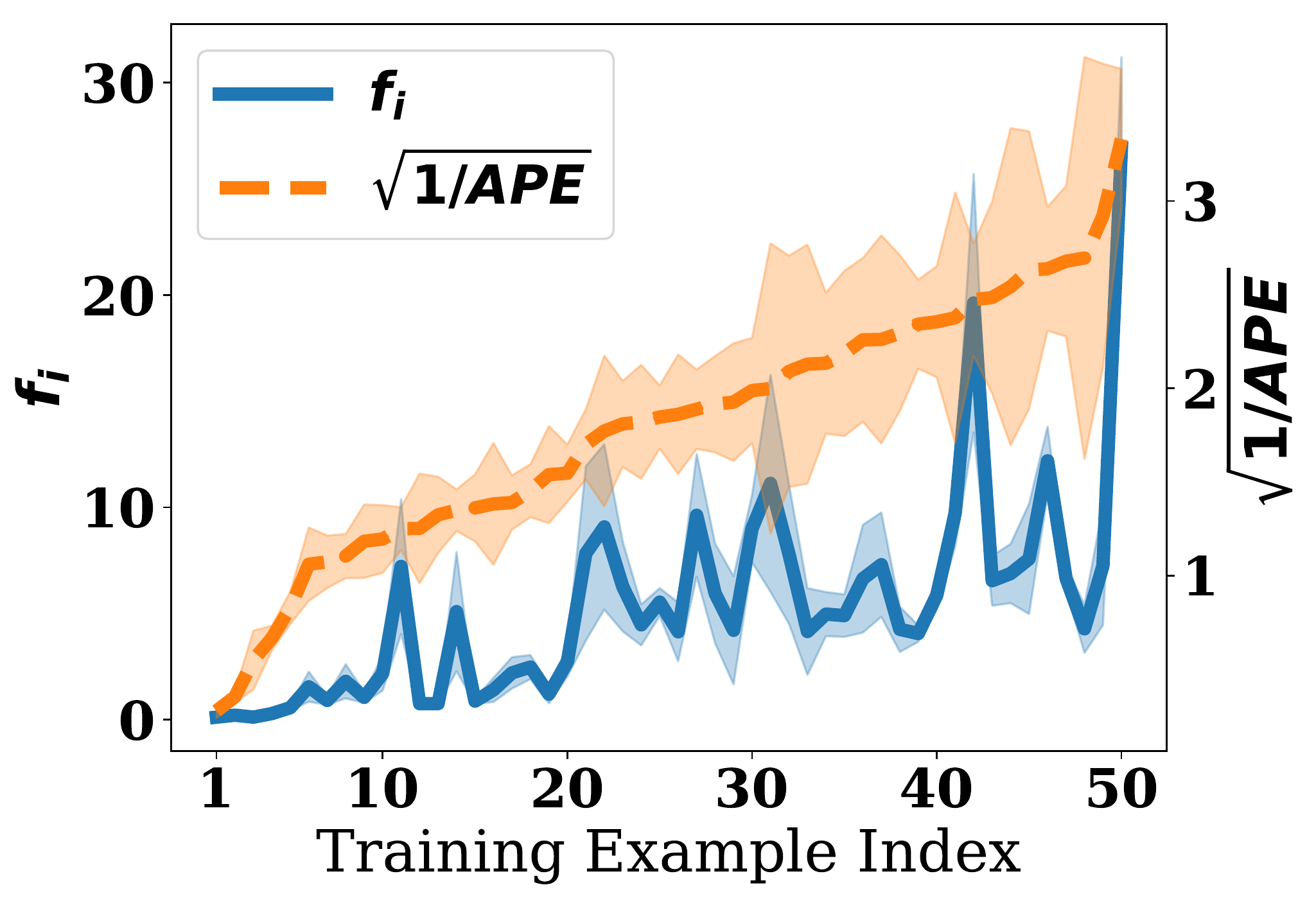}
    \caption{Breast Cancer SVC: $(0.498 \pm 0.021)$}
    \end{subfigure}
\hfill
    \begin{subfigure}[t]{0.23\textwidth} \centering
    \includegraphics[width=\textwidth]{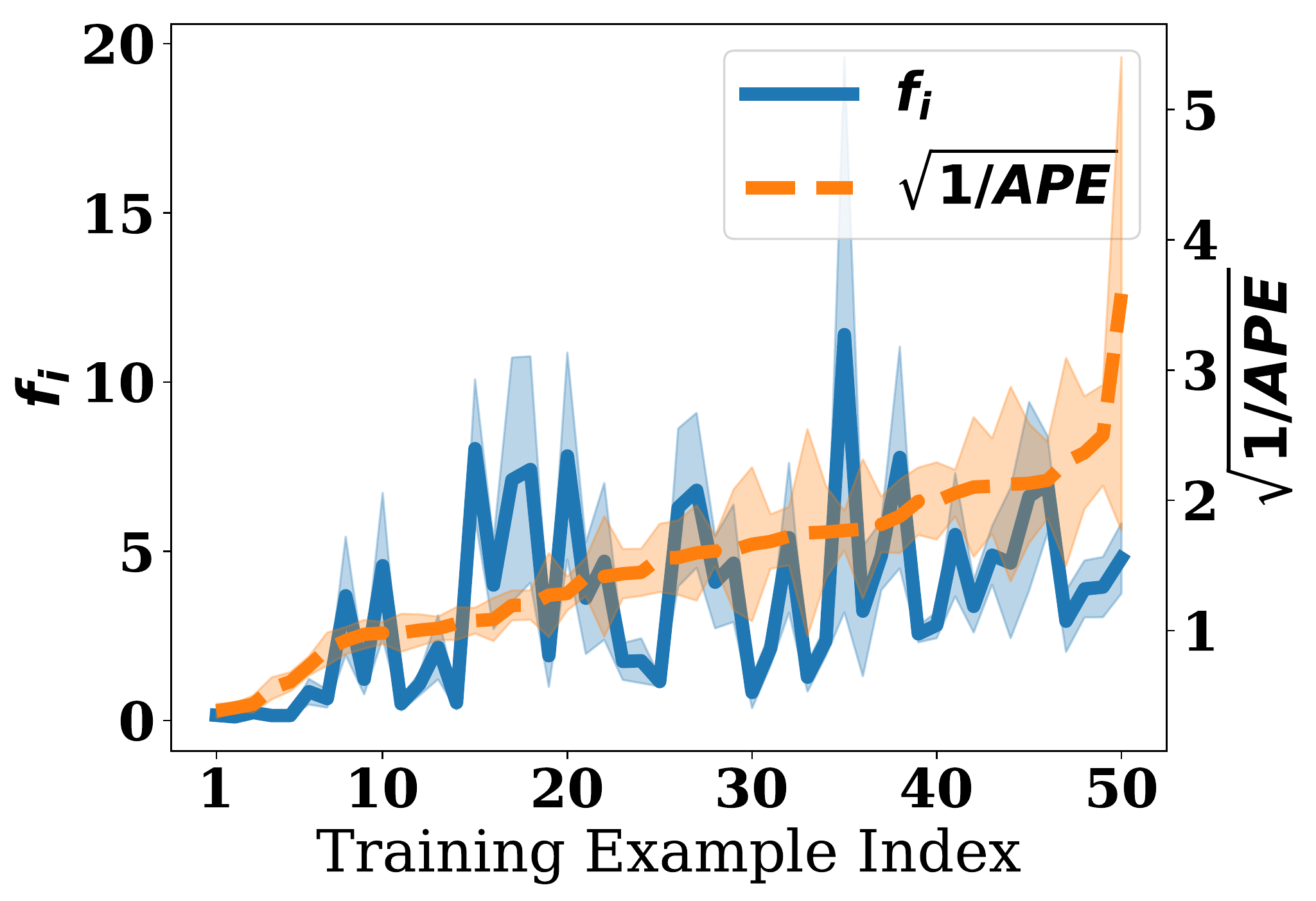}
    \caption{MNIST SVC: $(0.444 \pm 0.017)$}
    \end{subfigure}
\hfill
    \begin{subfigure}[t]{0.23\textwidth} \centering
    \includegraphics[width=\textwidth]{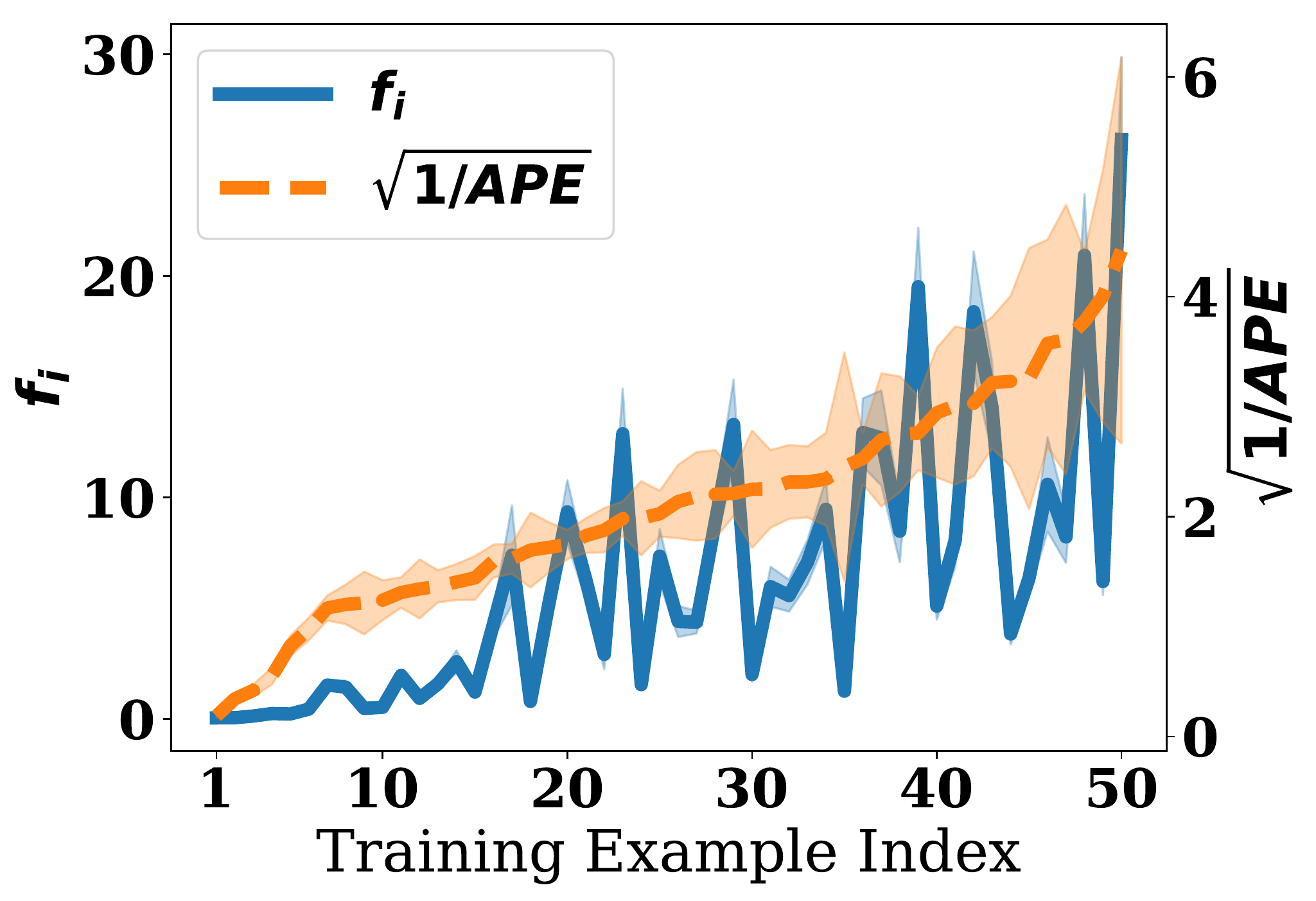}
    \caption{Synthetic Gaussian SVC: $(0.583 \pm 0.025)$}
    \end{subfigure}
\hfill
    \caption{ The plot of $f_i$ vs.~$\sqrt{1/\text{APE}}$ for $50$ training examples over $20$ repeated trials. Caption of each plot shows the dataset and learner, 
    followed by the average and standard error of the Spearman Rank coefficient.
    }
    \label{fig:ri_ape_app}
\end{figure}

\subsection{Additional Experimental Results}
Under the same experimental setting in Sec.~\ref{sec:experiments-P234} and the additional settings described above where applicable, additional experimental results are presented: average (standard errors) over $5$ independent trials for dataset valuation (\cref{tab:vol=used-car-overall}), CML (\cref{tab:cml-uberlyft-overall}), FL (\cref{tab:fl-mnistuni-overall,tab:fl-sst}), and feature attribution (\cref{tab:feature-adult-overall,tab:feature-covertype-overall}).
For all metrics, lower is better.

\begin{table}[!ht]
    \centering
        \caption{Evaluation of $\varphi_i$ within \textbf{P2.} using used car price with $n=8$ data providers who each have a randomly sub-sampled dataset containing $100$ training examples \citep{xu2021vol}. 
        }
        \resizebox{\linewidth}{!}{
        \begin{tabular}{lllrrrr}
        \toprule
        baselines &        MAPE &          MSE &  $N_{\text{inv}}$ &  $\epsilon_{\text{inv}}$ & NL NSW  \\
        \midrule
    MC &  2.23e-02 (2.0e-03) &  9.00e-04 (1.8e-04) &      0.80 (0.49) &     1.66 (0.14) &  2.66e-02 (5.1e-03) \\
    Owen &  1.67e-02 (2.8e-03) &  5.30e-04 (1.6e-04) &      0.80 (0.80) &     1.31 (0.21) &  2.70e-02 (2.6e-03) \\
    Sobol &  5.59e-02 (2.3e-03) &  4.35e-03 (4.1e-04) &      2.80 (0.49) &     4.12 (0.18) &      0.12 (1.1e-03) \\
    stratified &  2.52e-02 (3.3e-03) &  1.14e-03 (2.7e-04) &      1.20 (0.80) &     1.97 (0.25) &  4.20e-02 (6.2e-03) \\
    kernel &      0.11 (1.5e-02) &  1.57e-02 (4.7e-03) &      5.20 (0.49) &     7.22 (0.88) &         3.10 (0.26) \\
    \midrule
    Ours (a=0) &  2.67e-02 (3.9e-03) &  1.08e-03 (2.4e-04) &      1.20 (0.80) &     1.96 (0.28) &      0.23 (4.9e-02) \\
    Ours (a=2) &  1.42e-02 (1.6e-03) &  3.10e-04 (6.0e-05) &      0.80 (0.49) &     1.07 (0.12) &  \textbf{3.78e-03} (1.1e-03) \\
    Ours (a=5) &  \textbf{1.19e-02} (8.5e-04) &  \textbf{2.50e-04} (6.0e-05) &      0.80 (0.49) &  \textbf{0.90} (7.7e-02) &  8.40e-03 (2.3e-03) \\
    Ours (a=100) &  1.55e-02 (2.7e-03) &  4.70e-04 (2.0e-04) &      \textbf{0.40} (0.40) &     1.09 (0.17) &  1.79e-02 (4.3e-03) \\
        \bottomrule
        \end{tabular}
        }
    \label{tab:vol=used-car-overall}
\end{table}

\begin{table}[!ht]
    \centering
    \caption{Evaluation of $\varphi_i$ within \textbf{P3.} CML uber lyft ride price dataset with $n=10$ agents who each have a randomly sub-sampled dataset containing $50$ training examples \citep{xu2021vol}. 
    }
    \resizebox{\linewidth}{!}{
    \begin{tabular}{lllllll}
    \toprule
    baselines &        MAPE &          MSE &  $N_{\text{inv}}$ &  $\epsilon_{\text{inv}}$ & NL NSW  \\
    \midrule
    MC &  4.91e-02 (9.8e-03) &  3.80e-03 (1.3e-03) &  1.88e+01 (3.88) &      6.24 (1.19) &  1.66e-02 (6.6e-03) \\
    Owen &  3.76e-02 (7.8e-03) &  2.32e-03 (7.6e-04) &  1.28e+01 (3.01) &      4.76 (0.98) &  1.37e-02 (4.2e-03) \\
    Sobol &  7.72e-02 (3.6e-03) &  1.02e-02 (6.2e-04) &  3.44e+01 (2.48) &      9.92 (0.55) &  6.96e-02 (3.3e-03) \\
    stratified &  6.28e-02 (1.1e-02) &  5.83e-03 (1.5e-03) &  1.80e+01 (2.28) &      7.76 (1.31) &  3.19e-02 (1.1e-02) \\
    kernel &  7.37e-02 (4.1e-03) &  8.57e-03 (1.1e-03) &  2.84e+01 (2.93) &      9.70 (0.50) &         0.86 (0.17) \\
    \midrule
    Ours (a=0) &      0.15 (2.1e-02) &  3.54e-02 (9.7e-03) &  3.52e+01 (4.59) &  1.83e+01 (2.50) &         2.84 (0.76) \\
    Ours (a=2) &  3.67e-02 (4.4e-03) &  1.98e-03 (4.4e-04) &  1.64e+01 (1.94) &      4.61 (0.54) &  \textbf{8.40e-04} (2.2e-04) \\
    Ours (a=5) &  \textbf{2.51e-02} (1.3e-03) &  9.70e-04 (2.1e-04) &  1.04e+01 (1.72) &      3.18 (0.25) &  8.90e-04 (1.9e-04) \\
    Ours (a=100) &  2.55e-02 (4.0e-03) &  \textbf{9.60e-04} (2.6e-04) &      \textbf{9.60} (2.56) &      \textbf{3.14} (0.48) &  6.01e-03 (1.6e-03) \\
    \bottomrule
    \end{tabular}
        }
    \label{tab:cml-uberlyft-overall}
\end{table}
\begin{table}[!ht]
    \centering
    \caption{Evaluation of $\varphi_i$ within \textbf{P3.} FL using MNIST with $n=10$ agents on \textbf{I.I.D} partition.
    }
    \resizebox{\linewidth}{!}{
    \begin{tabular}{lllllll}
    \toprule
    baselines &        MAPE &          MSE &  $N_{\text{inv}}$ &  $\epsilon_{\text{inv}}$ & NL NSW  \\
    \midrule
    MC &      0.13 (6.3e-03) &  2.51e-02 (2.9e-03) &  4.16e+01 (5.81) &  1.71e+01 (1.11) &      0.52 (1.2e-02) \\
    Owen &      0.16 (1.2e-02) &  4.20e-02 (4.0e-03) &  4.32e+01 (7.61) &  2.23e+01 (1.14) &      0.63 (8.9e-03) \\
    Sobol &      0.22 (1.5e-02) &  9.39e-02 (1.4e-02) &  4.48e+01 (4.84) &  3.07e+01 (2.07) &      0.52 (1.9e-03) \\
    stratified &      0.11 (1.4e-02) &  1.99e-02 (3.9e-03) &  4.80e+01 (6.78) &  1.53e+01 (1.59) &      0.49 (3.7e-03) \\
    kernel &      0.49 (7.3e-02) &         0.39 (0.10) &  4.40e+01 (3.79) &  6.50e+01 (9.87) &  \textbf{2.45e-02} (1.2e-02) \\
    \midrule
    Ours (a=0) &      0.47 (5.9e-02) &      0.33 (8.2e-02) &  \textbf{4.12e+01} (4.50) &  5.87e+01 (7.21) &         2.04 (0.53) \\
    Ours (a=2) &  6.73e-02 (1.4e-02) &  7.59e-03 (3.1e-03) &  4.68e+01 (4.54) &      9.02 (1.81) &      0.62 (1.0e-02) \\
    Ours (a=5) &  \textbf{5.95e-02} (8.5e-03) &  \textbf{5.58e-03} (1.6e-03) &  4.64e+01 (5.74) &      \textbf{7.96} (1.16) &      0.56 (4.6e-03) \\
    Ours (a=100) &  6.51e-02 (8.2e-03) &  5.89e-03 (1.1e-03) &  4.44e+01 (3.12) &      8.22 (0.82) &      0.53 (9.2e-04) \\
    \bottomrule
    \end{tabular}
        }
    \label{tab:fl-mnistuni-overall}
\end{table}
\begin{table}[!ht]
    \centering
    \caption{Evaluation of $\varphi_i$ within \textbf{P3.} FL using MNIST with $n=10$ agents on \textbf{powerlaw} partition: the agents data are from the same distribution (containing images of all $10$ digits) but their sizes variy superlinearly \citep{xu2021gradient}. 
    In this partition, agent $10$ has the highest $\phi_i$. 
    Lowest NL NSW is not bolded due to to many ties.
    }
    \resizebox{\linewidth}{!}{
    \begin{tabular}{lllllll}
    \toprule
    baselines &        MAPE &          MSE &  $N_{\text{inv}}$ &  $\epsilon_{\text{inv}}$ & NL NSW  \\
    \midrule
    MC &      0.13 (6.7e-03) &  2.73e-02 (3.1e-03) &  \textbf{3.84e+01} (1.72) &   1.78e+01 (1.15) &  1.08e+02 (2.0e-02) \\
    Owen &      0.17 (1.3e-02) &  4.67e-02 (4.7e-03) &  4.04e+01 (3.54) &   2.35e+01 (1.29) &  1.08e+02 (1.7e-02) \\
    Sobol &      0.23 (1.5e-02) &      0.10 (1.5e-02) &  4.64e+01 (1.94) &   3.18e+01 (2.16) &  1.08e+02 (1.2e-02) \\
    stratified &      0.12 (1.4e-02) &  2.13e-02 (4.1e-03) &  3.20e+01 (3.85) &   1.58e+01 (1.65) &  1.08e+02 (1.4e-02) \\
    kernel &      0.55 (8.4e-02) &         0.49 (0.13) &  4.56e+01 (4.35) &  7.32e+01 (11.12) &  1.08e+02 (3.1e-02) \\
    \midrule
    Ours (a=0) &      0.84 (5.6e-02) &         1.12 (0.21) &  4.88e+01 (7.17) &   9.96e+01 (6.91) &     1.12e+02 (0.56) \\
    Ours (a=2) &  7.27e-02 (1.4e-02) &  8.22e-03 (3.1e-03) &  4.52e+01 (5.28) &       9.50 (1.63) &  1.08e+02 (1.7e-02) \\
    Ours (a=5) &  \textbf{6.19e-02} (7.8e-03) &  \textbf{5.84e-03} (1.5e-03) &  4.60e+01 (4.77) &       \textbf{8.22} (1.09) &  1.08e+02 (1.3e-02) \\
    Ours (a=100) &  7.06e-02 (9.6e-03) &  6.80e-03 (1.4e-03) &  4.08e+01 (2.33) &       8.89 (0.93) &  1.08e+02 (4.9e-03) \\
    \bottomrule
    \end{tabular}
        }
    \label{tab:fl-mnistpow-overall}
\end{table}
\begin{table}[!ht]
    \centering
    \caption{Evaluation of $\varphi_i$ within \textbf{P3.} FL using MNIST with $n=10$ agents on \textbf{classimbalance} partition: Agent $1$ only has images of digit $1$, agent $2$ has images of digits $1,2$ and so on \citep{xu2021gradient}. In this partition, agent $10$ has the highest $\phi_i$.
    }
    \resizebox{\linewidth}{!}{
    \begin{tabular}{lllllll}
    \toprule
    baselines &        MAPE &          MSE &  $N_{\text{inv}}$ &  $\epsilon_{\text{inv}}$ & NL NSW  \\
    \midrule
    MC &      0.13 (9.7e-03) &  2.45e-02 (2.7e-03) &  1.36e+01 (1.72) &  1.48e+01 (1.14) &  1.02 (4.5e-02) \\
    Owen &      0.18 (7.4e-03) &  4.78e-02 (5.2e-03) &  1.40e+01 (1.10) &  2.06e+01 (0.87) &  1.80 (3.5e-02) \\
    Sobol &      0.20 (1.1e-02) &  9.25e-02 (1.4e-02) &  1.44e+01 (1.17) &  2.46e+01 (1.48) &  1.10 (8.1e-03) \\
    stratified &      0.11 (1.2e-02) &  1.77e-02 (3.6e-03) &  1.24e+01 (1.47) &  1.27e+01 (1.42) &  0.70 (1.8e-02) \\
    kernel &      0.53 (7.3e-02) &      0.37 (9.3e-02) &  3.24e+01 (4.35) &  5.85e+01 (8.35) &     \textbf{0.59} (0.10) \\
    \midrule
    Ours $(\alpha=0)$ &      0.15 (2.0e-02) &  2.76e-02 (6.4e-03) &  1.48e+01 (3.26) &  1.57e+01 (1.95) &  8.12 (9.1e-02) \\
    Ours $(\alpha=2)$ &  6.58e-02 (8.0e-03) &  6.41e-03 (2.3e-03) &      8.00 (0.89) &      7.51 (1.08) &  1.24 (1.4e-02) \\
    Ours $(\alpha=5)$ &  5.92e-02 (7.2e-03) &  5.03e-03 (1.7e-03) &      \textbf{5.20} (0.80) &      6.62 (0.96) &  1.06 (3.4e-03) \\
    Ours $(\alpha=100)$ &  \textbf{5.71e-02} (7.3e-03) &  \textbf{4.17e-03} (8.5e-04) &      6.00 (1.79) &      \textbf{6.17} (0.67) &  1.01 (4.9e-03) \\
    \bottomrule
    \end{tabular}
        }
    \label{tab:fl-mnistcla-overall}
\end{table}
\begin{table}[!ht]
    \centering
        \caption{Evaluation of $\varphi_i$ within \textbf{P3.} via FL with $n=10$ agents on CIFAR-10 dataset with \textbf{I.I.D} data partition.
        }
        \resizebox{\linewidth}{!}{
        \begin{tabular}{lllrrrr}
        \toprule
        baselines &        MAPE &          MSE &  $N_{\text{inv}}$ &  $\epsilon_{\text{inv}}$ & NL NSW\\
        \midrule
        MC &      0.13 (6.6e-03) &  2.70e-02 (3.1e-03) &  3.96e+01 (5.31) &   1.78e+01 (1.15) &  2.31e+02 (2.9e-02) \\
        Owen &      0.17 (1.3e-02) &  4.61e-02 (4.6e-03) &  3.92e+01 (5.24) &   2.34e+01 (1.28) &  2.32e+02 (1.4e-02) \\
        Sobol &      0.23 (1.5e-02) &  9.86e-02 (1.5e-02) &  4.08e+01 (4.22) &   3.18e+01 (2.16) &  2.31e+02 (1.4e-03) \\
        stratified &      0.12 (1.4e-02) &  2.12e-02 (4.0e-03) &  4.80e+01 (6.66) &   1.59e+01 (1.61) &  2.31e+02 (7.4e-03) \\
        kernel &      0.55 (8.4e-02) &         0.49 (0.13) &  4.16e+01 (2.23) &  7.38e+01 (11.22) &  \textbf{2.30e+02} (3.2e-02) \\
        \midrule
        Ours (a=0) &      0.43 (6.8e-02) &         0.32 (0.11) &  4.40e+01 (4.34) &   5.53e+01 (7.86) &     2.52e+02 (5.28) \\
        Ours (a=2) &  7.06e-02 (1.3e-02) &  7.88e-03 (2.8e-03) &  \textbf{3.84e+01} (3.82) &       9.35 (1.59) &  2.31e+02 (1.4e-02) \\
        Ours (a=5) &  \textbf{5.98e-02} (8.2e-03) &  \textbf{5.67e-03} (1.8e-03) &  4.28e+01 (4.36) &       8.07 (1.16) &  2.31e+02 (4.4e-03) \\
        Ours (a=100) &  6.60e-02 (9.1e-03) &  5.85e-03 (1.3e-03) &  4.24e+01 (4.45) &       \textbf{8.22} (0.93) &  2.31e+02 (7.3e-04) \\
        \bottomrule
        \end{tabular}
        }
    \label{tab:fl-cifar10uni-overall}
\end{table}
\begin{table}[!ht]
    \centering
        \caption{Evaluation of $\varphi_i$ within \textbf{P3.} via FL with $n=10$ agents on CIFAR-10 dataset with \textbf{powerlaw} data partition. Lowest NL NSW is not bolded due to to many ties.
        }
        \resizebox{\linewidth}{!}{
        \begin{tabular}{lllrrrr}
        \toprule
        baselines &        MAPE &          MSE &  $N_{\text{inv}}$ &  $\epsilon_{\text{inv}}$ & NL NSW\\
        \midrule
        MC &      0.14 (1.1e-02) &  2.72e-02 (3.1e-03) &  1.44e+01 (0.40) &  1.56e+01 (1.05) &  2.31e+02 (5.4e-02) \\
        Owen &      0.21 (4.6e-03) &  5.57e-02 (4.1e-03) &  1.56e+01 (2.23) &  2.25e+01 (0.67) &  2.32e+02 (1.7e-02) \\
        Sobol &      0.22 (1.3e-02) &  9.91e-02 (1.5e-02) &  1.72e+01 (3.20) &  2.71e+01 (1.69) &  2.31e+02 (6.2e-03) \\
        stratified &      0.12 (1.3e-02) &  1.95e-02 (3.7e-03) &  1.36e+01 (3.12) &  1.32e+01 (1.44) &  2.31e+02 (9.6e-03) \\
        kernel &      0.61 (8.8e-02) &         0.49 (0.13) &  3.52e+01 (6.22) &  6.62e+01 (9.80) &  2.31e+02 (2.9e-02) \\
        \midrule
        Ours (a=0) &      0.11 (1.4e-02) &  1.52e-02 (3.1e-03) &  1.52e+01 (2.06) &  1.11e+01 (1.44) &     2.46e+02 (0.14) \\
        Ours (a=2) &  7.13e-02 (7.6e-03) &  8.11e-03 (2.2e-03) &  1.08e+01 (1.36) &      8.32 (0.89) &  2.31e+02 (1.4e-02) \\
        Ours (a=5) &  6.16e-02 (7.4e-03) &  5.95e-03 (1.4e-03) &      8.80 (1.62) &      7.21 (0.85) &  2.31e+02 (1.4e-02) \\
        Ours (a=100) &  \textbf{5.96e-02} (8.6e-03) &  \textbf{4.65e-03} (1.3e-03) &      \textbf{6.80} (1.20) &      \textbf{6.38} (0.90) &  2.31e+02 (1.7e-02) \\
        \bottomrule
        \end{tabular}
        }
    \label{tab:fl-cifar10pow-overall}
\end{table}
\begin{table}[!ht]
    \centering
        \caption{Evaluation of $\varphi_i$ within \textbf{P3.} via FL with $n=10$ agents on CIFAR-10 dataset with \textbf{classimbalance} data partition.
        Lowest NL NSW is not bolded due to to many ties.
        }
        \resizebox{\linewidth}{!}{
        \begin{tabular}{lllrrrr}
        \toprule
        baselines &        MAPE &          MSE &  $N_{\text{inv}}$ &  $\epsilon_{\text{inv}}$ & NL NSW\\
        \midrule
        MC &      0.14 (1.1e-02) &  2.72e-02 (3.1e-03) &  1.44e+01 (0.40) &  1.56e+01 (1.05) &  2.31e+02 (5.4e-02) \\
        Owen &      0.21 (4.6e-03) &  5.57e-02 (4.1e-03) &  1.56e+01 (2.23) &  2.25e+01 (0.67) &  2.32e+02 (1.7e-02) \\
        Sobol &      0.22 (1.3e-02) &  9.91e-02 (1.5e-02) &  1.72e+01 (3.20) &  2.71e+01 (1.69) &  2.31e+02 (6.2e-03) \\
        stratified &      0.12 (1.3e-02) &  1.95e-02 (3.7e-03) &  1.36e+01 (3.12) &  1.32e+01 (1.44) &  2.31e+02 (9.6e-03) \\
        kernel &      0.61 (8.8e-02) &         0.49 (0.13) &  3.52e+01 (6.22) &  6.62e+01 (9.80) &  2.31e+02 (2.9e-02) \\
        \midrule
        Ours (a=0) &      0.11 (1.4e-02) &  1.52e-02 (3.1e-03) &  1.52e+01 (2.06) &  1.11e+01 (1.44) &     2.46e+02 (0.14) \\
        Ours (a=2) &  7.13e-02 (7.6e-03) &  8.11e-03 (2.2e-03) &  1.08e+01 (1.36) &      8.32 (0.89) &  2.31e+02 (1.4e-02) \\
        Ours (a=5) &  6.16e-02 (7.4e-03) &  5.95e-03 (1.4e-03) &      8.80 (1.62) &      7.21 (0.85) &  2.31e+02 (1.4e-02) \\
        Ours (a=100) &  \textbf{5.96e-02} (8.6e-03) &  \textbf{4.65e-03} (1.3e-03) &      \textbf{6.80} (1.20) &      \textbf{6.38} (0.90) &  2.31e+02 (1.7e-02) \\
        \bottomrule
        \end{tabular}
        }
    \label{tab:fl-cifar10cla-overall}
\end{table}
\begin{table}[!ht]
    \centering
        \caption{Evaluation of $\varphi_i$ within \textbf{P3.} via FL with $n=5$ agents on movie reviews dataset with powerlaw data distribution.
        }
        \resizebox{\linewidth}{!}{
        \begin{tabular}{lllrrrr}
        \toprule
        baselines &        MAPE &          MSE &  $N_{\text{inv}}$ &  $\epsilon_{\text{inv}}$ & NL NSW\\
        \midrule
        MC &  7.51e-02 (1.6e-02) &  1.08e-02 (4.3e-03) &      9.60 (0.98) &      2.43 (0.53) &  1.19 (1.2e-02) \\
        Owen &      0.11 (1.4e-02) &  1.50e-02 (3.8e-03) &      9.20 (1.85) &      3.18 (0.43) &  1.78 (3.2e-02) \\
        Sobol &      0.48 (2.2e-02) &      0.37 (3.2e-02) &      8.40 (0.75) &  1.35e+01 (0.62) &  1.15 (5.3e-03) \\
        stratified &  6.70e-02 (1.1e-02) &  6.17e-03 (1.4e-03) &      9.20 (1.62) &      2.00 (0.28) &  0.98 (1.3e-02) \\
        kernel &      0.25 (6.3e-02) &  8.95e-02 (3.5e-02) &      9.20 (1.36) &      7.26 (1.58) &  \textbf{0.39} (6.5e-02) \\
        \midrule
        Ours (a=0) &  9.44e-02 (1.7e-02) &  1.61e-02 (6.5e-03) &      8.00 (1.10) &      2.99 (0.54) &  1.78 (6.3e-02) \\
        Ours (a=2) &  3.85e-02 (5.0e-03) &  2.31e-03 (5.7e-04) &      \textbf{7.60} (1.72) &      1.21 (0.16) &  1.27 (6.1e-03) \\
        Ours (a=5) &  3.42e-02 (5.1e-03) &  1.73e-03 (4.6e-04) &      8.40 (1.17) &      1.08 (0.14) &  1.19 (5.0e-03) \\
        Ours (a=100) &  \textbf{2.98e-02} (4.9e-03) &  \textbf{1.32e-03} (3.3e-04) &      8.00 (1.41) &      \textbf{0.94} (0.13) &  1.19 (5.3e-03) \\
        \bottomrule
        \end{tabular}
        }
    \label{tab:fl-mr}
\end{table}
\begin{table}[!ht]
    \centering
        \caption{Evaluation of $\varphi_i$ within \textbf{P3.} via FL with $n=5$ agents on SST-5 dataset with powerlaw data distribution.
        Lowest NL NSW is not bolded due to to many ties.
        }
        \resizebox{\linewidth}{!}{
        \begin{tabular}{lllrrrr}
        \toprule
        baselines &        MAPE &          MSE &  $N_{\text{inv}}$ &  $\epsilon_{\text{inv}}$ & NL NSW\\
        \midrule
        MC &  6.40e-02 (2.1e-02) &  9.40e-03 (4.8e-03) &      \textbf{5.60} (1.60) &      2.09 (0.70) &  7.38e+02 (2.5e-02) \\
        Owen &  8.65e-02 (2.5e-02) &  1.22e-02 (4.8e-03) &      6.00 (2.00) &      2.63 (0.75) &  7.38e+02 (1.6e-02) \\
        Sobol &      0.48 (2.2e-02) &      0.37 (3.2e-02) &  1.00e+01 (1.10) &  1.38e+01 (0.60) &  7.38e+02 (6.9e-03) \\
        stratified &  6.38e-02 (1.2e-02) &  5.65e-03 (1.6e-03) &      7.60 (1.94) &      1.94 (0.33) &  7.38e+02 (1.2e-02) \\
        kernel &      0.17 (5.9e-02) &  5.35e-02 (3.0e-02) &  1.16e+01 (1.72) &      5.30 (1.67) &  7.38e+02 (2.8e-02) \\
           \midrule
        Ours (a=0) &      0.11 (2.7e-02) &  2.19e-02 (7.3e-03) &  1.00e+01 (2.61) &      3.59 (0.88) &     8.50e+02 (9.91) \\
        Ours (a=2) &  3.43e-02 (8.5e-03) &  2.19e-03 (6.4e-04) &      7.60 (2.14) &      1.12 (0.28) &  7.38e+02 (1.4e-03) \\
        Ours (a=5) &  2.91e-02 (6.9e-03) &  1.43e-03 (4.0e-04) &      8.00 (2.61) &      0.95 (0.22) &  7.38e+02 (4.4e-04) \\
        Ours (a=100) &  \textbf{2.07e-02} (8.4e-03) &  \textbf{8.40e-04} (5.1e-04) &      8.00 (2.28) &      \textbf{0.62} (0.25) &  7.38e+02 (6.0e-05) \\
        \bottomrule
        \end{tabular}
        }
    \label{tab:fl-sst}
\end{table}

\begin{table}[!ht]
    \centering
    \caption{Evaluation of $\varphi_i$ within \textbf{P4.} using the adult income dataset with $n=7$ principal features and $2000$ randomly drawn data samples trained on a random forest classifier.
    }
    \resizebox{\linewidth}{!}{
\begin{tabular}{lllllll}
\toprule
  baselines &                MAPE &                 MSE & $N_{\text{inv}}$ &  $\epsilon_{\text{inv}}$ &               NL NSW \\
\midrule
         MC &  6.88e-02 (7.2e-03) &  8.70e-03 (1.0e-03) &      4.00 (1.10) &      3.88 (0.26) &  4.97e-02 (1.2e-02)  \\
       Owen &      0.15 (1.1e-02) &  3.47e-02 (4.5e-03) &      7.20 (1.85) &      8.05 (0.56) &      0.12 (2.3e-02)  \\
      Sobol &      0.27 (1.1e-02) &  9.55e-02 (6.4e-03) &  1.40e+01 (1.41) &  1.33e+01 (0.49) &      0.55 (4.5e-02)  \\
 stratified &  5.80e-02 (1.3e-02) &  6.31e-03 (2.2e-03) &      5.60 (1.17) &      3.01 (0.68) &      0.20 (2.2e-02)  \\
     kernel &  7.11e-02 (8.8e-03) &  6.42e-03 (1.2e-03) &      6.80 (1.62) &      3.77 (0.42) &         6.19 (0.20)  \\
     \midrule
   Ours (0) &      0.22 (5.6e-02) &      0.13 (8.4e-02) &      7.60 (1.60) &  1.05e+01 (2.46) &         1.14 (0.56)  \\
   Ours (2) &  6.34e-02 (9.7e-03) &  7.28e-03 (1.9e-03) &      2.80 (1.20) &      3.56 (0.49) &  \textbf{3.70e-03} (1.4e-03)  \\
   Ours (5) &  \textbf{3.02e-02} (3.5e-03) &  \textbf{1.73e-03} (3.9e-04) &      4.40 (0.75) &      \textbf{1.68} (0.18) &  3.83e-03 (1.2e-03)  \\
 Ours (100) &  3.76e-02 (3.2e-03) &  3.12e-03 (8.2e-04) &      \textbf{1.60} (0.75) &      2.18 (0.22) &  2.65e-02 (3.6e-03)  \\
\bottomrule
\end{tabular}
        }
    \label{tab:feature-adult-overall}
\end{table}
\begin{table}[!ht]
    \centering
    \caption{Evaluation of $\varphi_i$ within \textbf{P4.} using the iris dataset trained on a $k$-NN classifier.
    }
    \resizebox{\linewidth}{!}{
\begin{tabular}{lllllll}
\toprule
  baselines &                MAPE &                 MSE &    $N_{\text{inv}}$ &  $\epsilon_{\text{inv}}$&           NL NSW \\
\midrule
         MC &  9.80e-02 (1.7e-02) &  6.44e-03 (1.2e-03) &         0.80 (0.49) &       1.21 (0.21) &     2.29 (0.12) \\
       Owen &      0.33 (2.0e-02) &      0.12 (1.4e-02) &         1.20 (0.49) &       3.28 (0.27) &  \textbf{1.06} (6.0e-02) \\
      Sobol &      1.11 (1.8e-02) &      0.62 (1.9e-02) &  0.00e+00 (0.0e+00) &   1.39e+01 (0.24) &  1.44 (4.6e-02) \\
 stratified &      0.56 (3.4e-02) &      0.38 (3.4e-02) &         0.80 (0.49) &       5.57 (0.47) &     6.78 (0.14)  \\
     kernel &         4.53 (0.88) &     2.08e+01 (8.15) &         3.20 (1.36) &  5.83e+01 (12.91) &     3.16 (1.07)  \\
     \midrule
   Ours (0) &  \textbf{5.24e-02} (1.3e-02) &  \textbf{2.63e-03} (5.2e-04) &         \textbf{0.40} (0.40) &       \textbf{0.75} (0.14) &  2.21 (5.6e-02) \\
   Ours (2) &  6.20e-02 (9.3e-03) &  3.09e-03 (6.5e-04) &         \textbf{0.40} (0.40) &       0.86 (0.12) &  2.37 (2.8e-02) \\
   Ours (5) &  8.29e-02 (9.9e-03) &  4.59e-03 (1.1e-03) &         0.80 (0.49) &       1.04 (0.13) &  2.30 (1.0e-01)  \\
 Ours (100) &      0.10 (9.2e-03) &  5.78e-03 (1.0e-03) &         1.60 (0.40) &       1.31 (0.12) &  2.42 (2.1e-02)  \\
\bottomrule
\end{tabular}
        }
    \label{tab:feature-iris-overall}
\end{table}
\begin{table}[!ht]
    \centering
    \caption{Evaluation of $\varphi_i$ within \textbf{P4.} using the covertype dataset with $n=7$ principal features and $2000$ randomly drawn data samples trained on a MLP.
    }
    \resizebox{\linewidth}{!}{
\begin{tabular}{lllllll}
\toprule
  baselines &                MAPE &                 MSE & $N_{\text{inv}}$ &  $\epsilon_{\text{inv}}$ &               NL NSW \\
\midrule
         MC &  5.05e-02 (8.8e-03) &  4.86e-03 (2.0e-03) &      1.60 (0.40) &      2.67 (0.54) &      0.16 (1.6e-02)  \\
       Owen &  9.98e-02 (1.5e-02) &  1.52e-02 (3.2e-03) &      1.60 (0.40) &      4.76 (0.61) &      0.18 (3.5e-02)  \\
      Sobol &      0.25 (8.3e-03) &      0.16 (1.1e-02) &      8.00 (0.63) &  1.37e+01 (0.64) &      0.95 (5.0e-02)  \\
 stratified &  4.30e-02 (5.8e-03) &  2.64e-03 (7.1e-04) &      0.40 (0.40) &      2.04 (0.28) &      0.50 (4.5e-03)  \\
     kernel &      0.13 (2.6e-02) &  1.25e-02 (4.0e-03) &      2.40 (1.47) &      4.83 (0.84) &         7.27 (0.49)  \\
     \midrule
   Ours (0) &      0.22 (7.2e-02) &      0.12 (7.0e-02) &      4.00 (1.67) &  1.07e+01 (3.52) &         2.03 (1.09)  \\
   Ours (2) &  \textbf{3.56e-02} (1.0e-02) &  2.75e-03 (1.1e-03) &      \textbf{1.20} (0.80) &      1.79 (0.44) &  1.43e-02 (3.1e-03)  \\
   Ours (5) &  3.84e-02 (5.7e-03) &  2.68e-03 (1.4e-03) &      \textbf{1.20} (0.49) &      1.79 (0.41) &  \textbf{4.73e-02} (7.6e-03)  \\
 Ours (100) &  3.19e-02 (6.4e-03) &  \textbf{2.12e-03} (9.6e-04) &      \textbf{1.20} (0.49) &      1.64 (0.40) &      0.13 (4.7e-03)  \\
\bottomrule
\end{tabular}
        }
    \label{tab:feature-covertype-overall}
\end{table}

\subsection{Noisy Label Detection}
Data valuation is a popular use case of Shapley value~\citep{pmlr-v97-ghorbani19c,Ghorbani2020ADF}. One application scenario of data valuation is to detect label noise. Data can be mis-labelled during crowd-sourcing or if the dataset itself is poisoned. While~\cite{pmlr-v97-ghorbani19c,Kwon2022} have demonstrated the effectiveness of data valuation in noisy label detection using, the accuracy is compromised when training examples receive valuations with low FS. A training example which receives a valuation much lower than its theoretical valuation may be mis-classified as noisy and \textit{vice versa}. In this experiment, we show that GAE outperforms other estimation methods on the noisy label detection task. We randomly select $100$ training examples from MNIST, breast cancer, and synthetic Gaussian datasets. We run bootstrapping with $25$ permutations and then set $20000$ total budget using various estimation methods. We compare the performance on two valuation metrics, namely Data Shapley~\citep{pmlr-v97-ghorbani19c} and Beta Shapley~\citep{Kwon2022}. Both metrics are designed for data valuation. Beta Shapley builds on Data Shapley by assigning different weights to marginal contributions of different cardinalities (we choose $\text{Beta}(16,1)$ as the weight distribution as it has the best empirical performance in \citep{Kwon2022}). \cref{fig:f1_score} shows the performance comparison. GAE is able to perform well on various combinations of datasets, valuation metrics, and learning algorithms.
\begin{figure}[h!]
    \centering
    \begin{subfigure}[b]{0.4\textwidth}\centering
    \includegraphics[width=\textwidth]{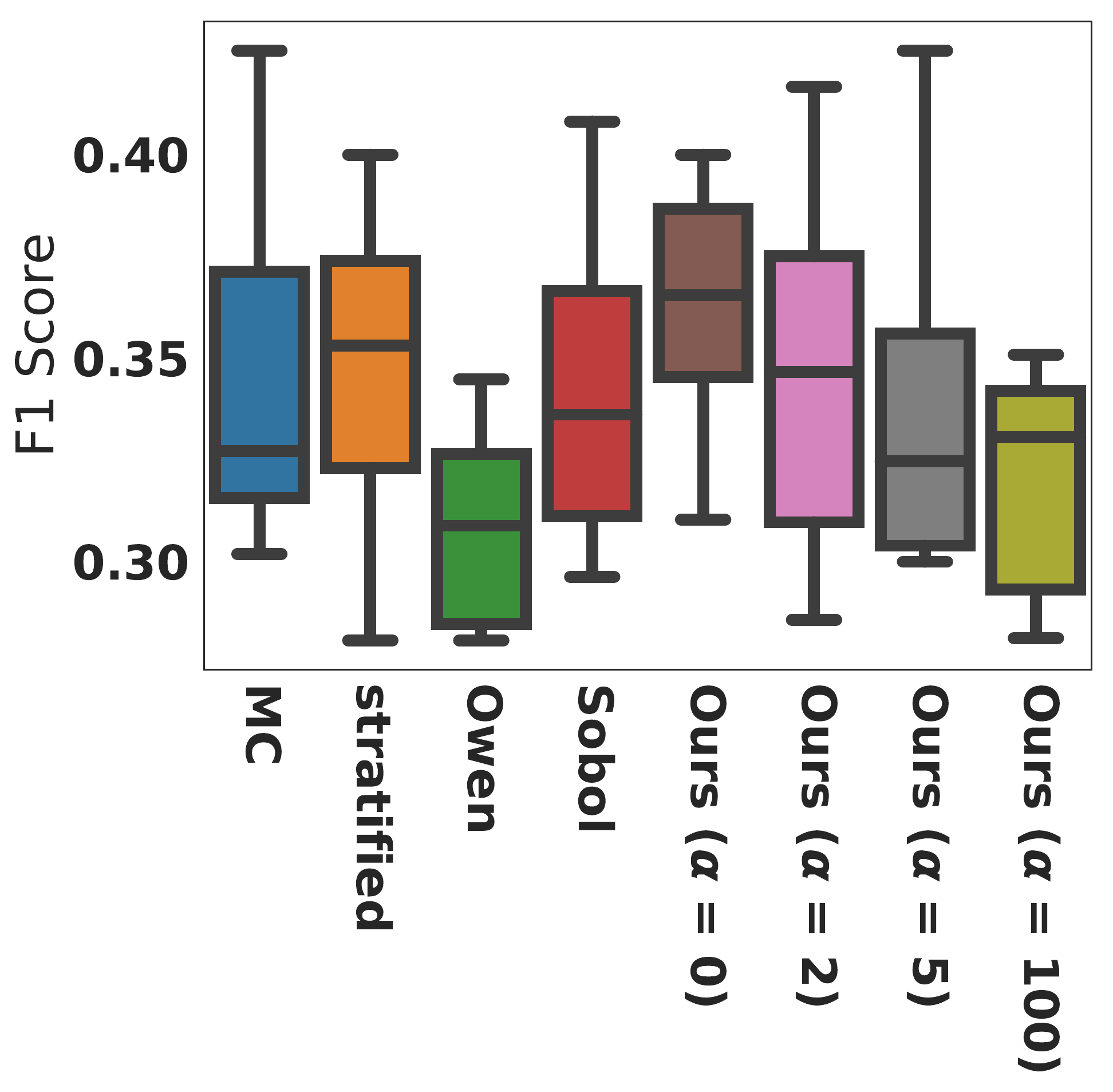}
    \caption{Data Shapley (MNIST SVC)}
    \end{subfigure}
     \begin{subfigure}[b]{0.4\textwidth}\centering
    \includegraphics[width=\textwidth]{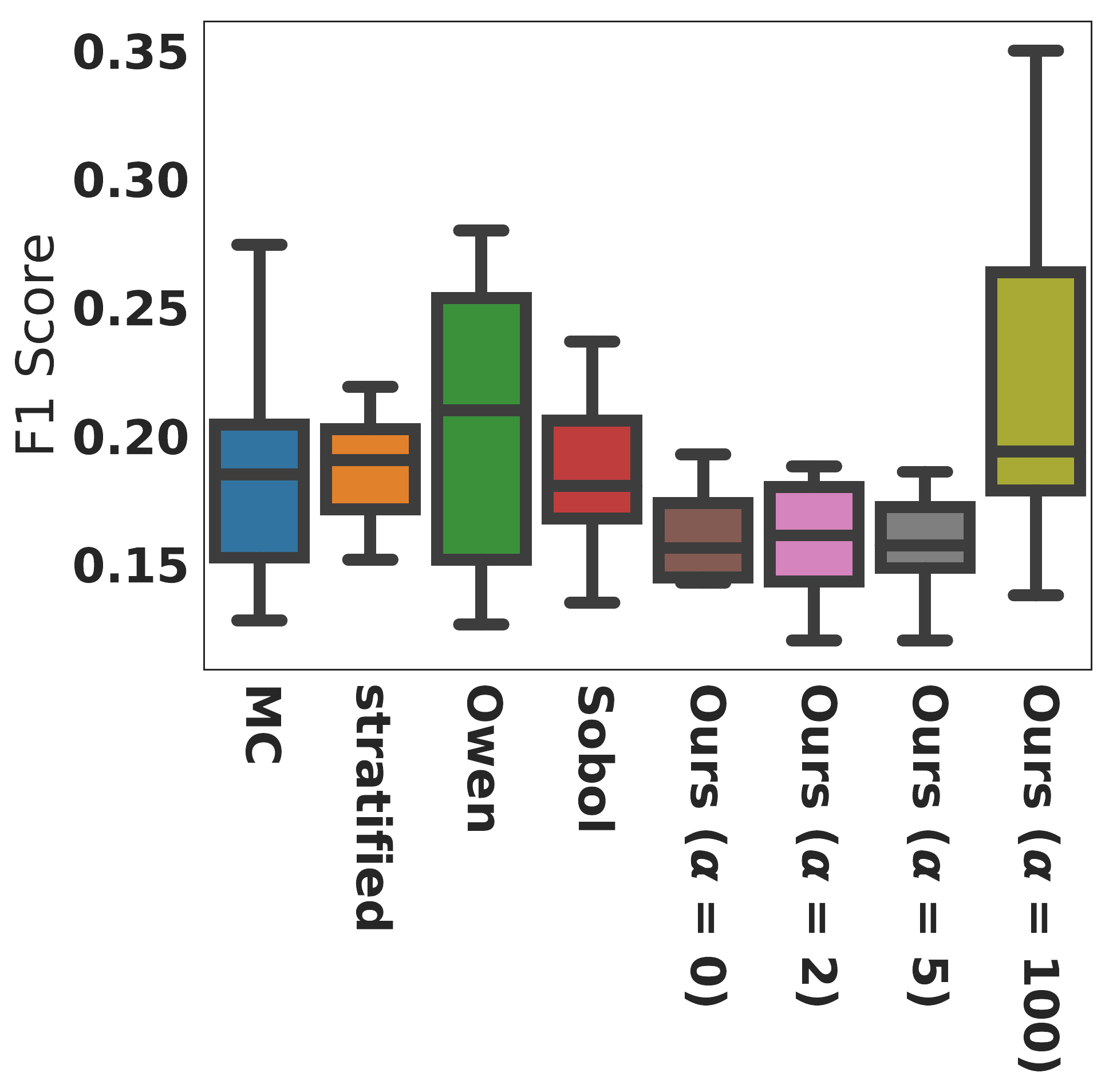}
    \caption{Beta Shapley (breast cancer SVC)}
    \end{subfigure}
    \\
    \begin{subfigure}[b]{0.4\textwidth}\centering
    \includegraphics[width=\textwidth]{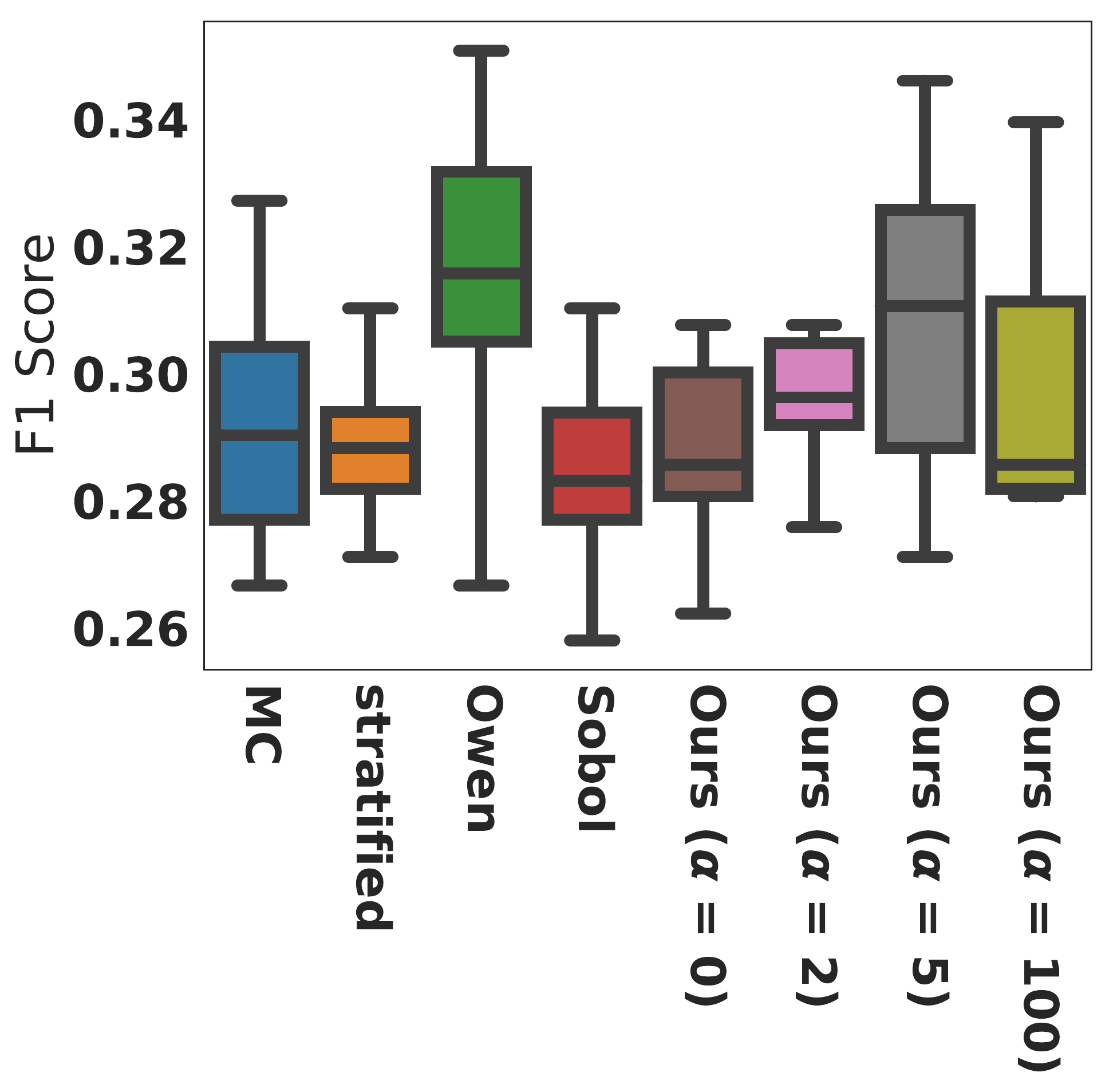}
    \caption{Data Shapley (Gaussian SVC)}
    \end{subfigure}
    \begin{subfigure}[b]{0.4\textwidth}\centering
    \includegraphics[width=\textwidth]{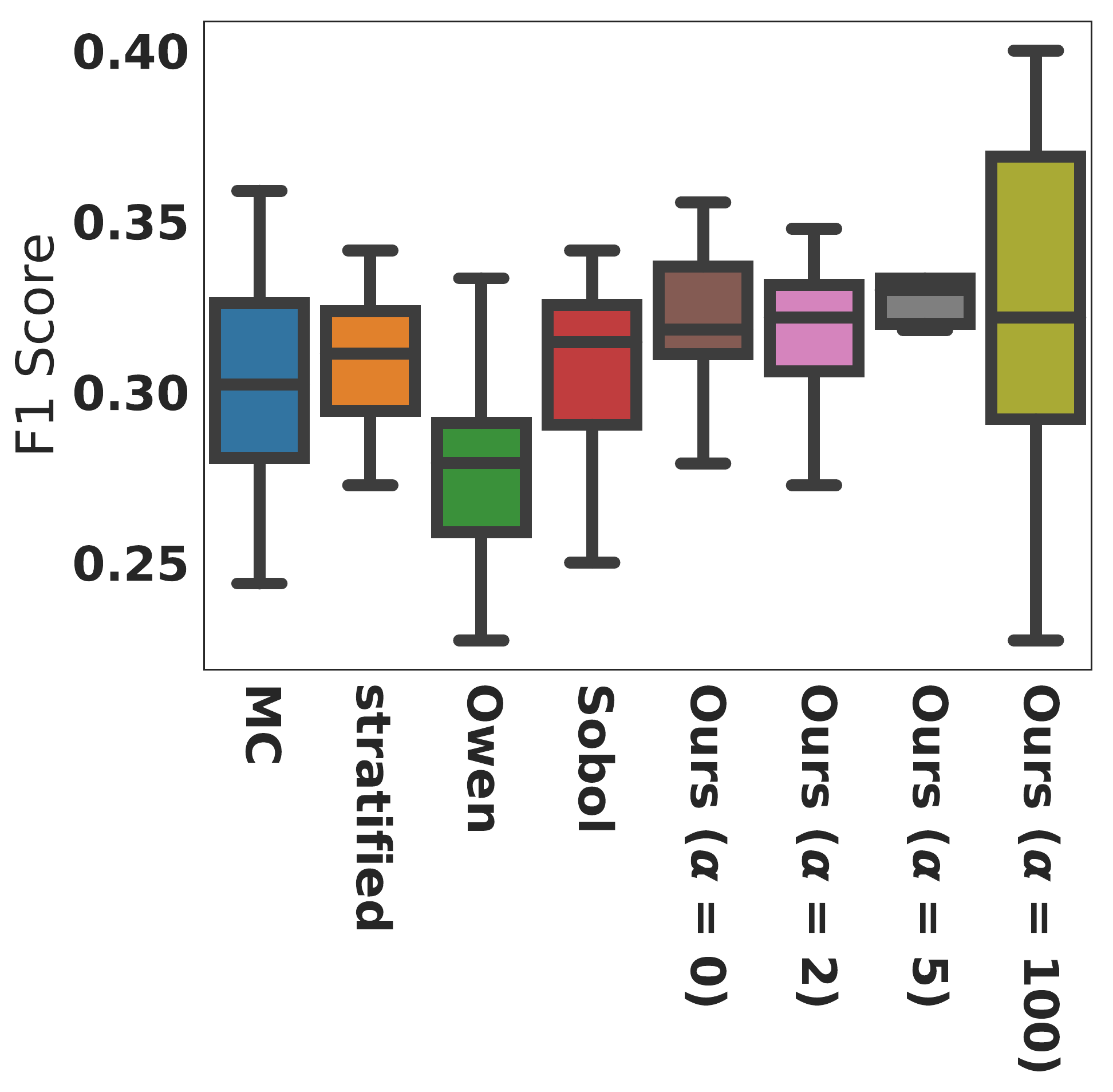}
    \caption{Beta Shapley (Gaussian SVC)}
    \end{subfigure}
    \caption{ Comparison between F1 score of noise detection task by various estimation methods on breast cancer and MNIST datasets. Caption of each plot shows the combination. KernelSHAP is omitted on MNIST dataset because its result is poor and obscures visualisation (medians are 0.194 on (a), 0.127 on (b), and 0.156 on (c) and (d)). }
    \label{fig:f1_score}
\end{figure}

\subsection{Training Example Addition and Removal}
In addition to noisy label detection, Shapley value is also applied to find out the value of a training example in training a ML model, as demonstrated by the point addition and removal experiment. In addition, we perform more experiments on both Data Shapley and Beta Shapley.
Results in \cref{fig:pt_addition} show most methods perform comparable except KernelSHAP.
The performance of KernelSHAP can be attributed to its estimates are less accurate (as shown by its higher MAPEs and MSEs in previous tables), consistent with \cref{fig:add-remove}.

\begin{figure}[ht!]
    \begin{subfigure}[b]{0.24\textwidth}\centering
    \includegraphics[width=\textwidth]{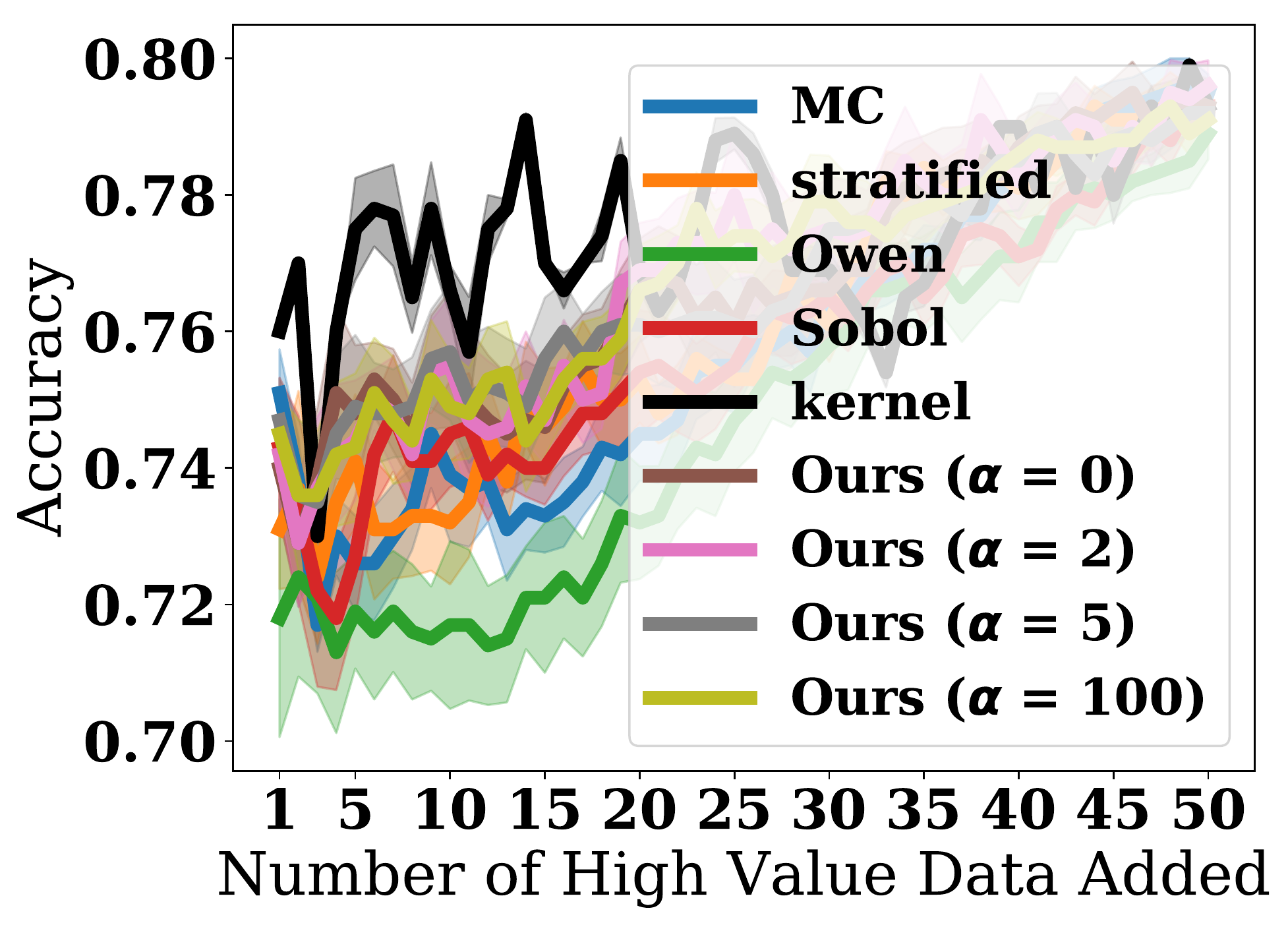}
    \caption{Data Shapley (Add High Value Data)}
    \end{subfigure}
\hfill
     \begin{subfigure}[b]{0.24\textwidth}\centering
    \includegraphics[width=\textwidth]{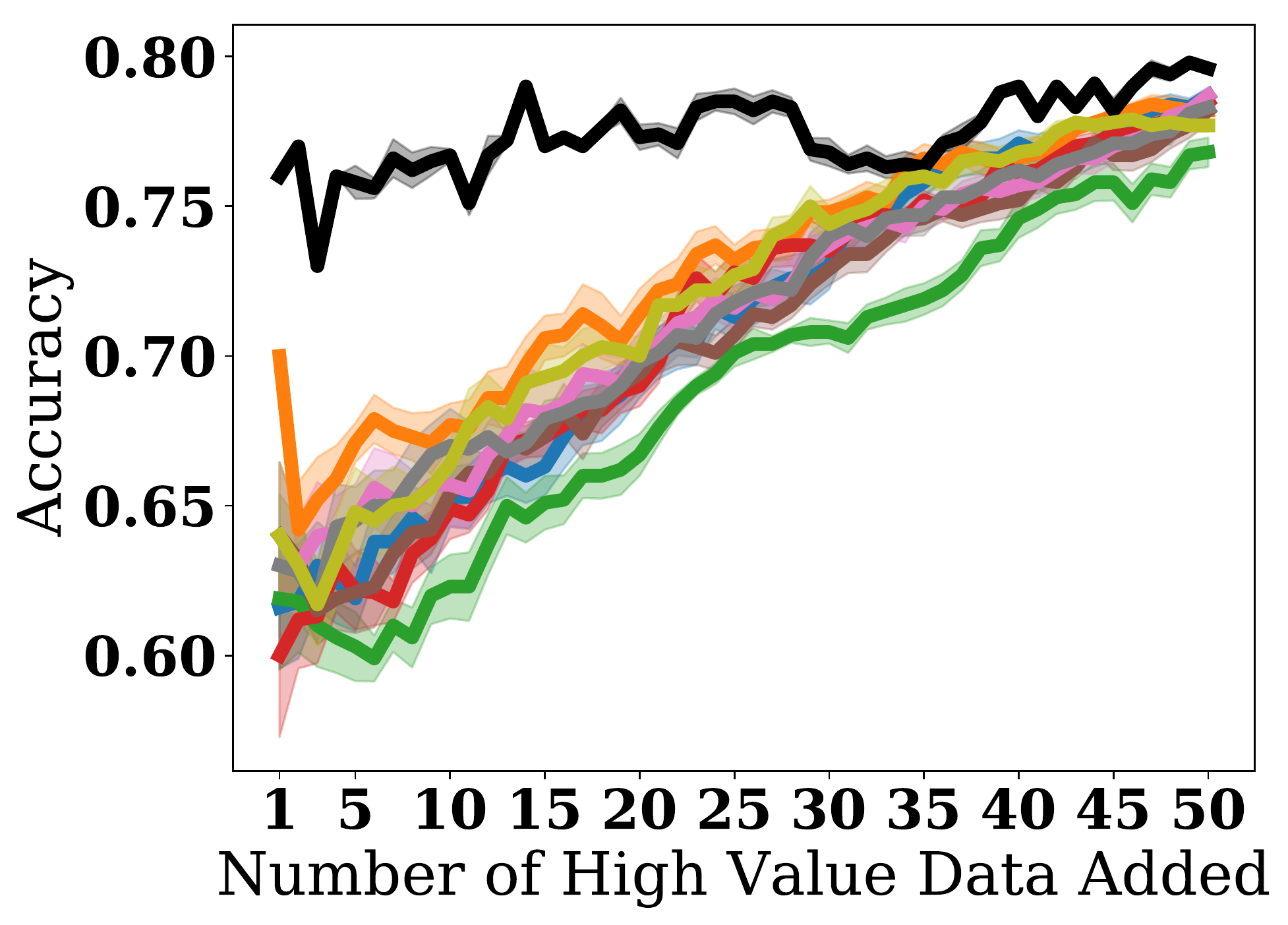}
    \caption{Beta Shapley (Add High Value Data)}
    \end{subfigure}
\hfill
    \begin{subfigure}[b]{0.24\textwidth}\centering
    \includegraphics[width=\textwidth]{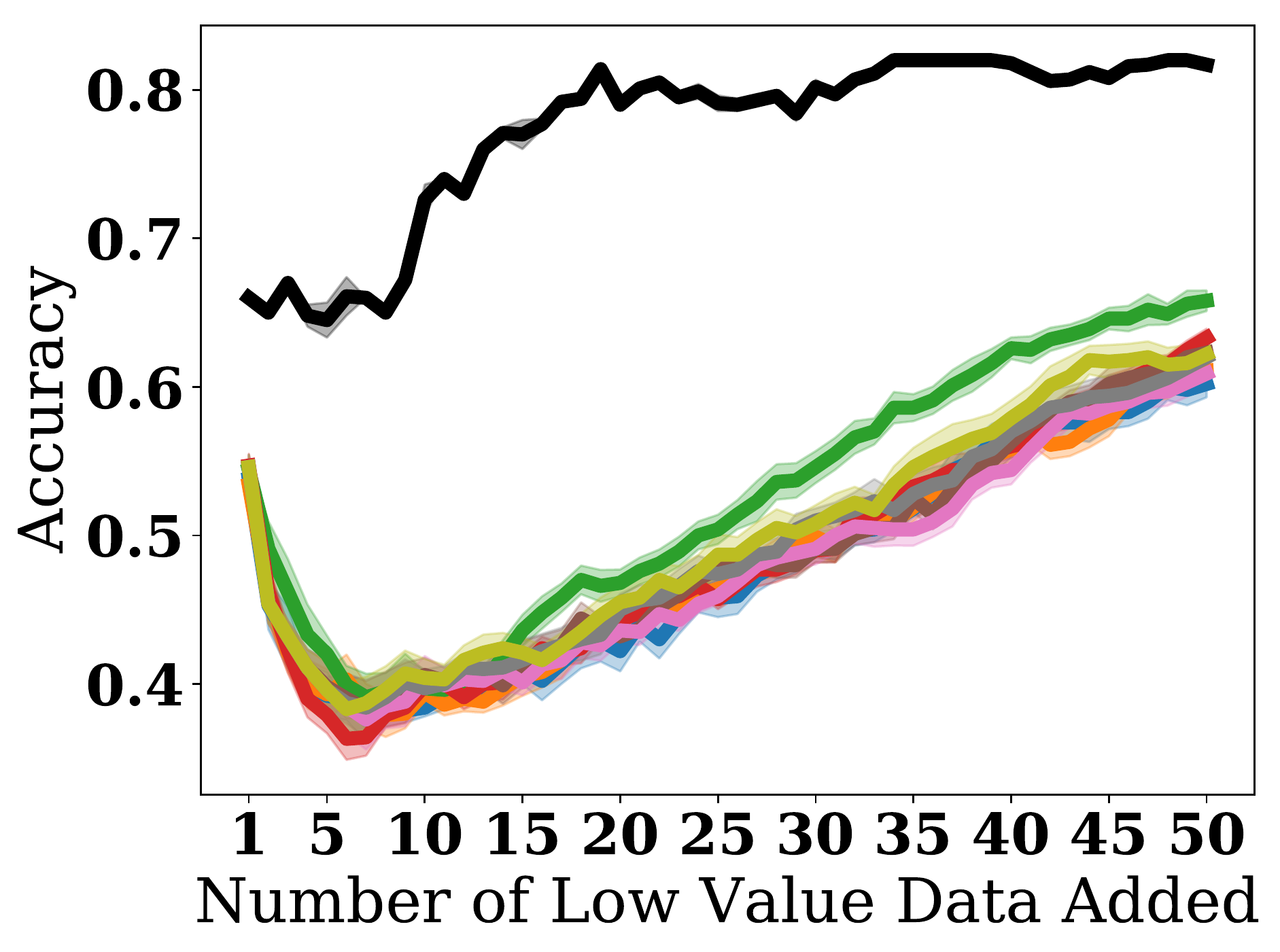}
    \caption{Data Shapley (Add Low Value Data)}
    \end{subfigure}
\hfill
    \begin{subfigure}[b]{0.24\textwidth}\centering
    \includegraphics[width=\textwidth]{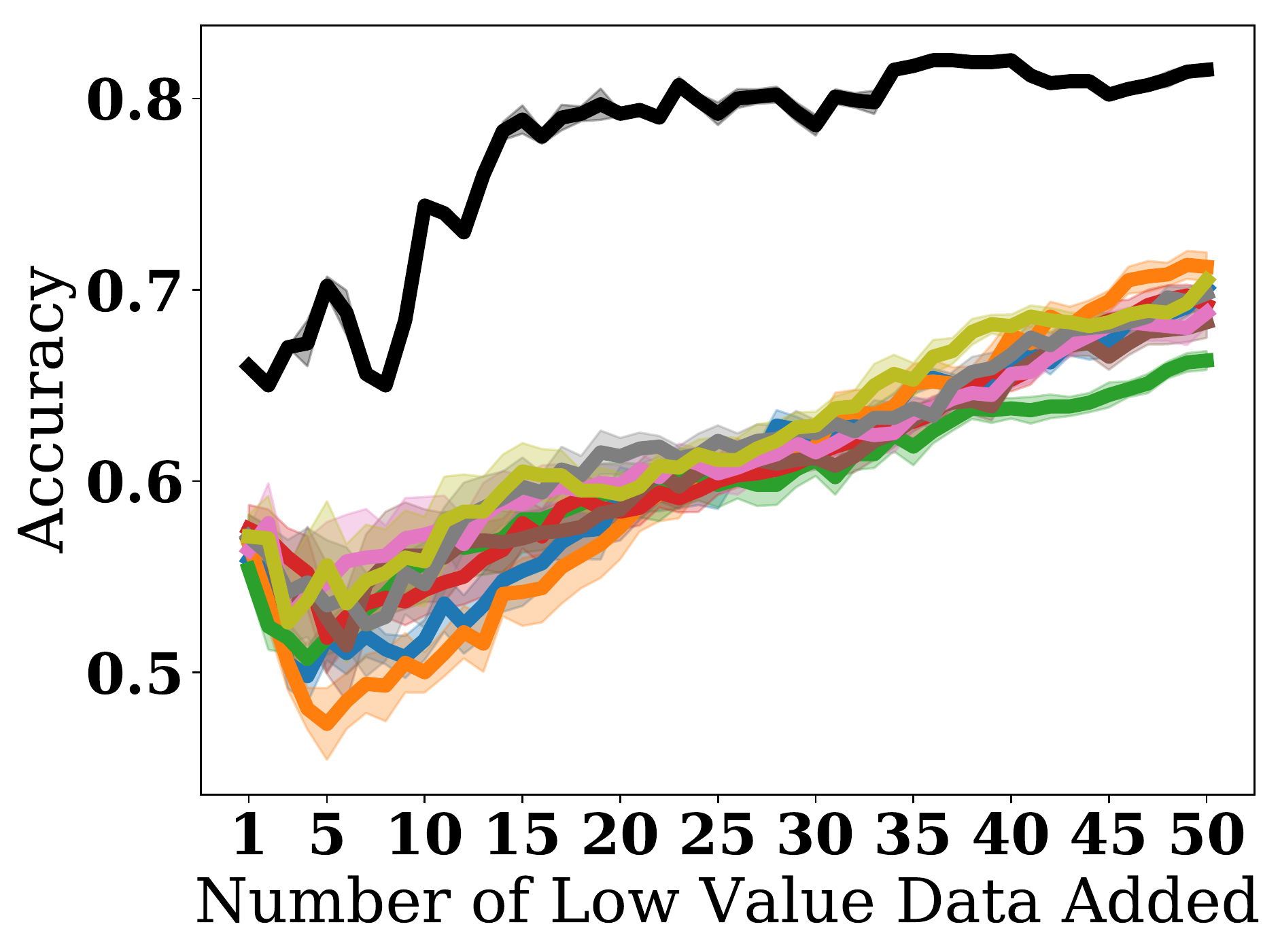}
    \caption{Beta Shapley (Add Low Value Data)}
    \end{subfigure}
\hfill
    \begin{subfigure}[b]{0.24\textwidth}\centering
    \includegraphics[width=\textwidth]{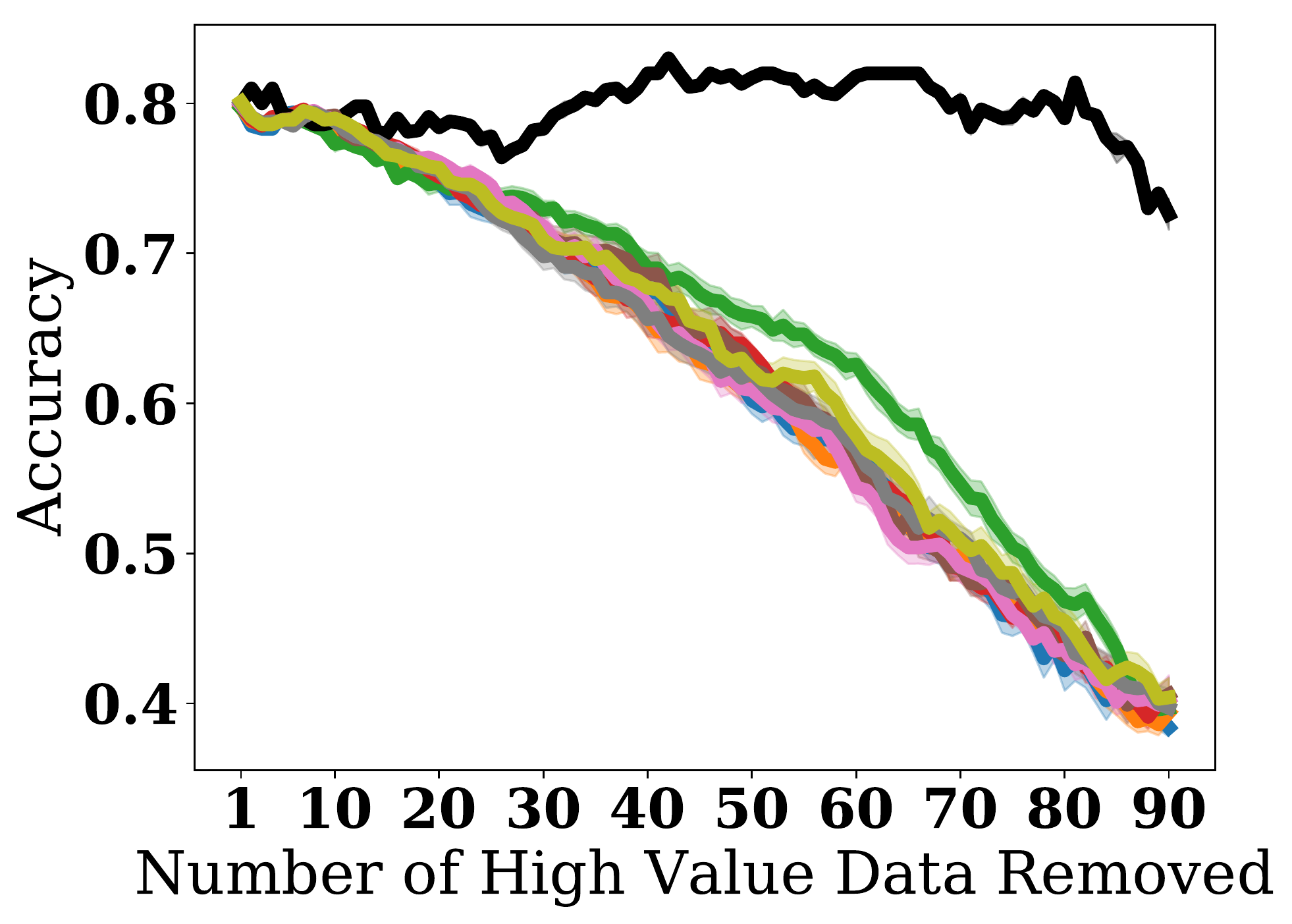}
    \caption{Data Shapley (Remove High Value Data)}
    \end{subfigure}
\hfill
     \begin{subfigure}[b]{0.24\textwidth}\centering
    \includegraphics[width=\textwidth]{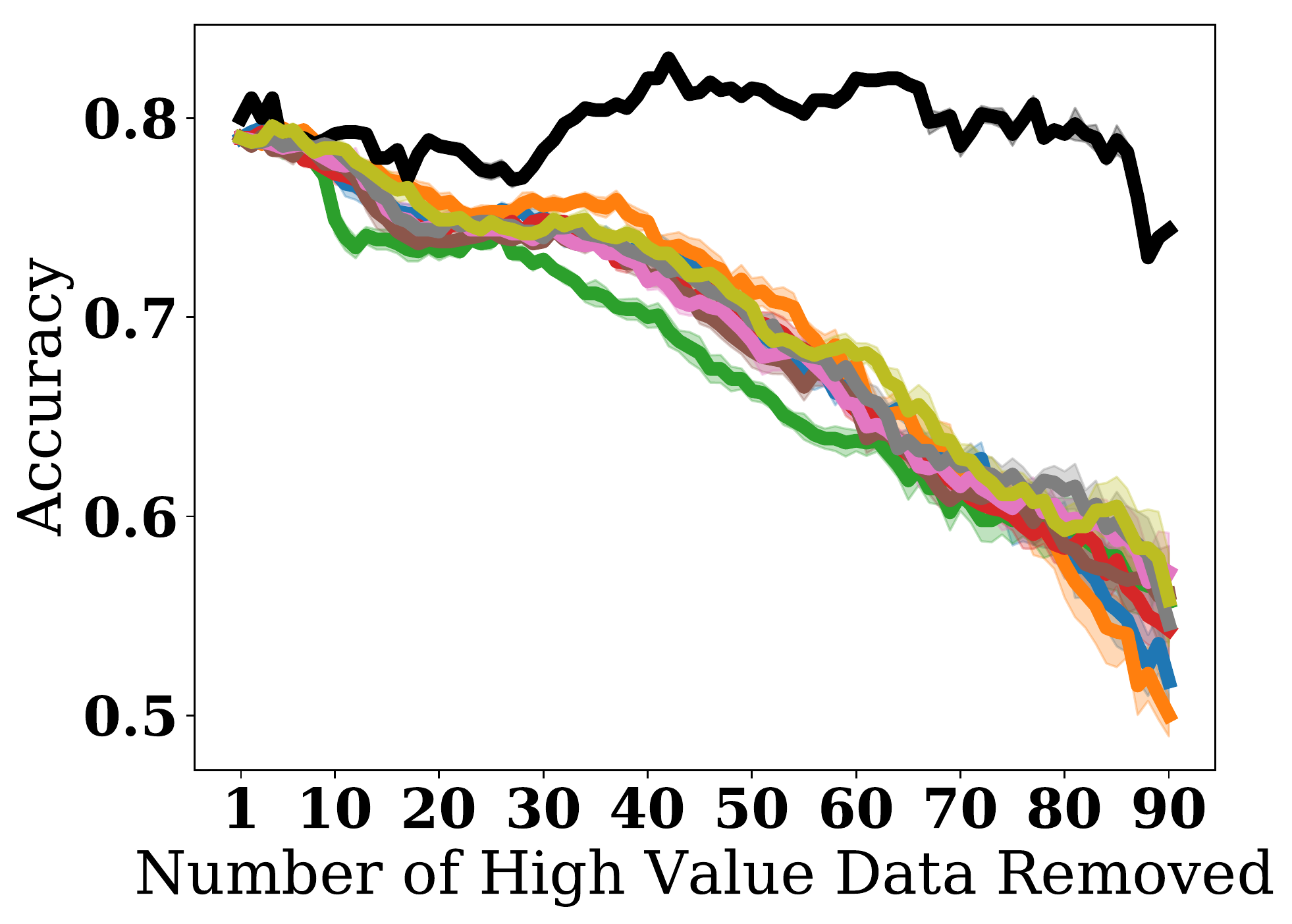}
    \caption{Beta Shapley (Remove High Value Data)}
    \end{subfigure}
\hfill
    \begin{subfigure}[b]{0.24\textwidth}\centering
    \includegraphics[width=\textwidth]{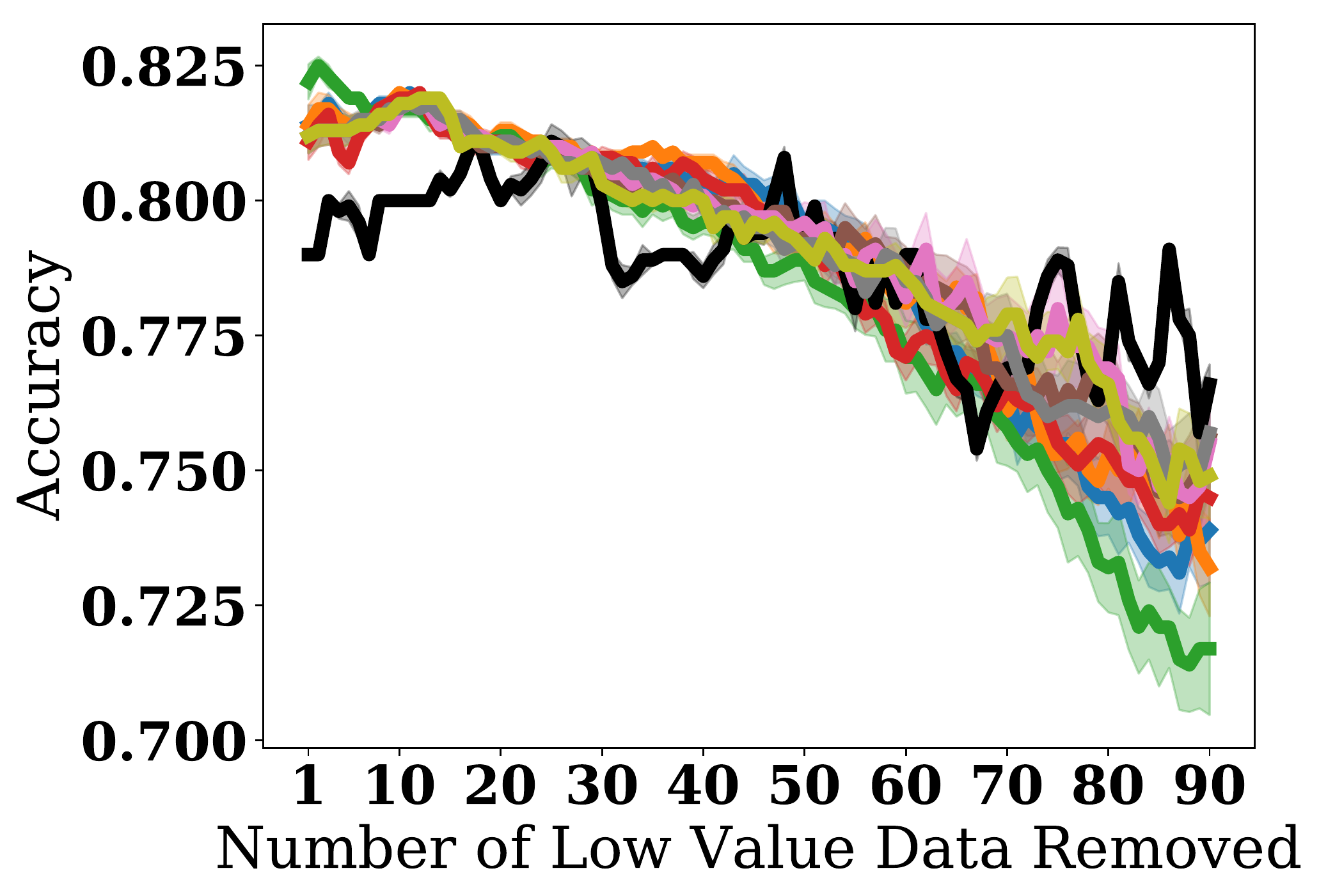}
    \caption{Data Shapley (Remove Low Value Data)}
    \end{subfigure}
\hfill
    \begin{subfigure}[b]{0.24\textwidth}\centering
    \includegraphics[width=\textwidth]{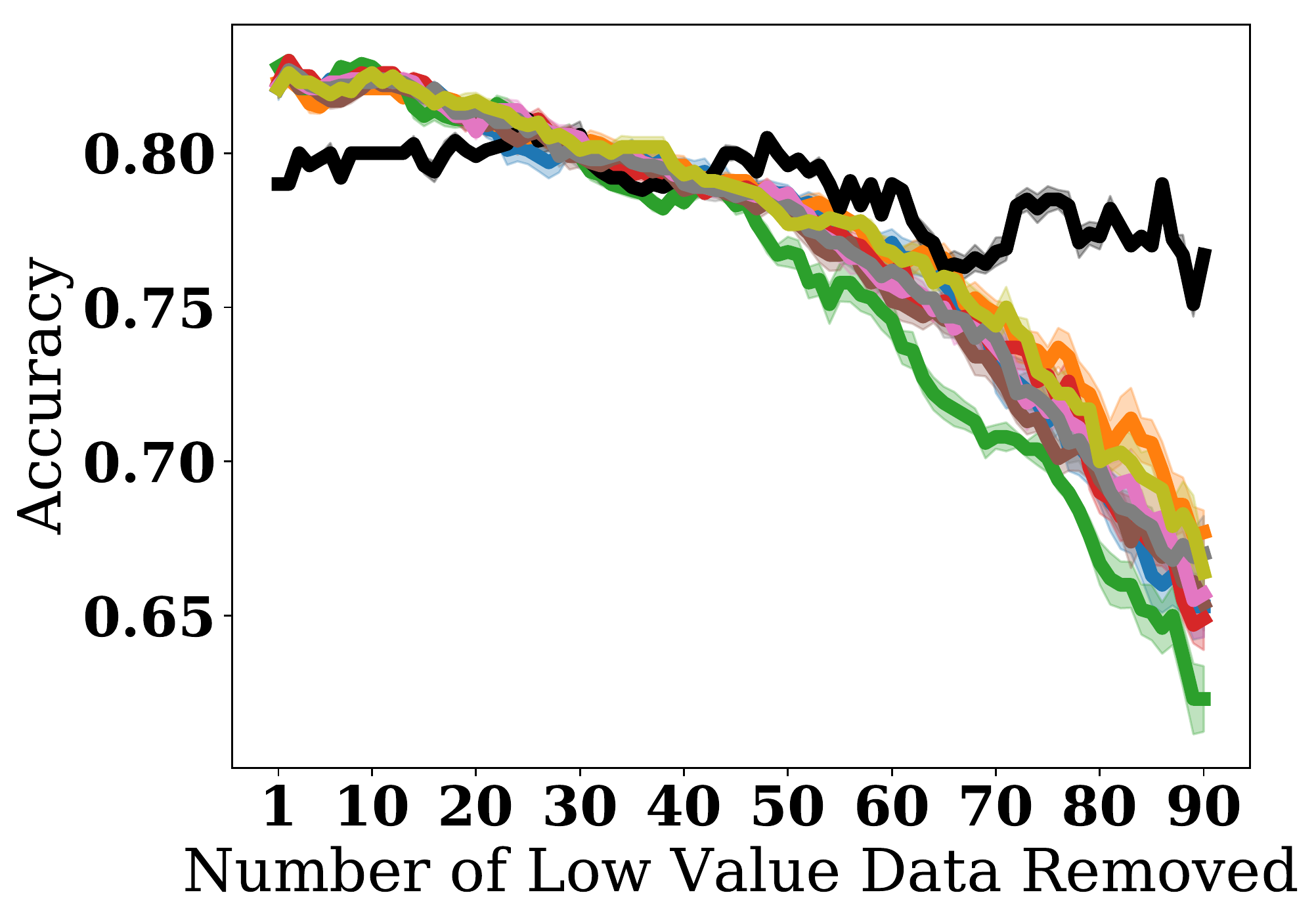}
    \caption{Beta Shapley (Remove Low Value Data)}
    \end{subfigure}
\hfill
    \caption{ Comparison between test accuracy by adding and removing data with large and small estimated Shapley value first (on synthetic Gaussian dataset trained on logistic regression classifier). Black lines represent KernelSHAP (omitted in main content). It can be observed that KernelSHAP does not produce a clear trend but stays horizontal in all plots.
     }
    \label{fig:pt_addition}
\end{figure}